\documentclass[preprint,11pt,authoryear]{elsarticle}
\usepackage{setspace}
\usepackage{lineno}
\usepackage[a4paper,margin=1in]{geometry}   
\usepackage{amsmath,amssymb,amsthm,bm}
\usepackage[T1]{fontenc}   
\usepackage{lmodern}
\usepackage{booktabs}
\usepackage{graphicx}
\usepackage{tikz}
\usetikzlibrary{arrows.meta,positioning,calc,fit,backgrounds}
\usepackage{url}
\usepackage{longtable}
\usepackage{booktabs}
\usepackage[hidelinks]{hyperref}
\hypersetup{
  pdftitle={Dynamic Regime-Aware Conformal Calibration for Reliable Economic Forecast Intervals under Multiple Distribution Shifts},
  pdfauthor={Bogdan Oancea},
  pdfsubject={Conformal prediction; interval forecasting; distribution shift},
  pdfkeywords={prediction intervals, calibration, coverage, uncertainty quantification,
    evaluating forecasts, inflation, structural breaks},
  pdfcreator={pdfLaTeX}
}
\providecommand{\doi}[1]{\href{https://doi.org/#1}{#1}}

\newtheorem{theorem}{Theorem}
\newtheorem{proposition}{Proposition}

\newtheorem{assumption}{Assumption}
\theoremstyle{remark}\newtheorem{remark}{Remark}
\theoremstyle{plain}
\newcommand{\DRACP}{DRACP}
\newcommand{\R}{\mathbb{R}}
\newcommand{\E}{\mathbb{E}}
\newcommand{\Prob}{\mathbb{P}}
\newcommand{\ind}{\mathbb{1}}

\begin{document}

\begin{frontmatter}

\title{Dynamic Regime-Aware Conformal Calibration for Reliable Economic Forecast Intervals under Multiple Distribution Shifts}

\author[a,b]{Bogdan Oancea}
\affiliation[a]{organization={University of Bucharest, Faculty of Business and Administration},
	addressline={4--12 Regina Elisabeta Blvd.},
	postcode={030018},
	city={Bucharest},
	country={Romania}}

\affiliation[b]{organization={National Institute of Research and Development for
		Biological Sciences},
	addressline={296 Splaiul Independen\c{t}ei},
	postcode={060031},
	city={Bucharest},
	country={Romania}}

\begin{abstract}
	Forecasting units that publish prediction intervals carrying a stated coverage level, such as
	central banks, statistical institutes and system operators, are judged ex post on whether that
	level held. Conformal prediction delivers distribution-free intervals under exchangeability, an
	assumption that economic series violate in several ways at once. We propose a calibration method
	that composes density-ratio, localization and regime-similarity weighting of past forecast errors
	with an online adjustment of the significance level, and we characterize its coverage error
	theoretically. We benchmark the method on 48 macroeconomic and financial series against six alternatives.
	Three of these are recent online procedures, which we reimplement from the published
	algorithms and verify against the authors' own reference code. The method attains the most
	reliable coverage on the panel, holding closest to nominal at every forecast horizon, under five
	base forecasters, and through the 2021--23 inflation surge, and it is one of only two procedures
	that never undercover below $0.80$ on any series. Sharper intervals are available: a strongly
	adaptive procedure produces them, at the cost of coverage below its stated level on twenty of the
	48 series. We identify the decision settings in which the additional width is warranted.
\end{abstract}

\end{frontmatter}

\section{Introduction}
\label{sec:intro}
Interval forecasts inform decisions in macroeconomic and financial policy, and their
usefulness depends on whether the stated coverage is attained when the data-generating process
changes. Conformal prediction \citep{vovk2005algorithmic,lei2018distribution} converts any
point forecaster into prediction intervals with finite-sample validity, requiring only that the
data be exchangeable. This assumption is routinely violated in sequential economic
applications. The covariate distribution drifts; the conditional response mechanism changes at
structural breaks; volatility clusters; and the series moves between latent regimes such as
expansions and recessions.

Several extensions of conformal prediction address departures from exchangeability, but most of them focus on a particular source of distribution shift.
 Weighted conformal prediction corrects covariate shift by means of likelihood-ratio
weights \citep{tibshirani2019conformal}. Localized methods restore conditional validity by
down-weighting distant calibration points \citep{guan2023conformal,barber2021limits}. Online
procedures adjust the nominal level using realized coverage feedback; these include adaptive
conformal inference \citep{gibbs2021adaptive}, its fully adaptive \citep{Zaffran2022} and
strongly adaptive \citep{Bhatnagar2023} variants, and conformal PID control
\citep{Angelopoulos2023pid}. In economic time series, however, several of these departures may occur simultaneously.
Electricity demand, for instance, displays seasonal covariate shift, weather-driven abrupt
changes, time-varying volatility and recurring regimes simultaneously. Methods targeting a single source of non-exchangeability may therefore be insufficient when several forms of distribution shift occur simultaneously.

This paper proposes Dynamic Regime-Aware Conformal Prediction (\DRACP), which combines four
adaptation mechanisms within a single weighted-conformal calibration step: density-ratio
weighting for covariate shift, kernel weighting for local heterogeneity, probabilistic
regime-similarity weighting for latent structure, and online adaptation of the miscoverage
level. Section~\ref{sec:method} describes the construction of this method. Section~\ref{sec:theory}
establishes finite-sample coverage under oracle weights, a robustness bound that decomposes the
coverage error into interpretable perturbation terms, and a long-run coverage property for the
online controller.

Sections~\ref{sec:data} and~\ref{sec:results} report a replicable empirical evaluation.
\DRACP{} is compared with six baselines on 48 real series. The baselines are split, rolling and
adaptive conformal prediction, together with the FACI, SAOCP and conformal PID procedures; the
series comprise euro-area and EU-27 HICP inflation, US macroeconomic and energy indicators, and
daily financial series. The three recent procedures are implemented following their published
algorithms and verified against the authors' reference code, since comparisons drawn against
simplified reimplementations are unlikely to be informative.

The methods differ most in what they get right. A strongly adaptive procedure produces the
narrowest intervals; \DRACP{} does not. Its coverage, however, stays closest to nominal at all
four horizons, under all five point forecasters used in our experiments, and through the 2021--23 inflation surge. The worst
series covers at $0.807$, while the sharper procedure drops to $0.750$ and falls below $0.80$
on five of the 48 data sets. That reliability costs width. The two methods therefore sit at different
points on the same trade-off, and Section~\ref{sec:disc} sets out when each is appropriate. An
ablation on the real and synthetic data sets panels traces the reliability to
its sources and shows that the composite weighting accounts for only a small part of it.

\section{Related work}
\label{sec:related}
We review the main approaches according to their construction of the prediction interval and their relationship to  \DRACP{}. Throughout, $\widehat\mu$ denotes a
point forecaster fitted on a training block, $S_i$ a conformity score on a calibration set of
size $m$, and $1-\alpha$ the target level.

\subsection{Split conformal prediction and its guarantee}
Split conformal prediction \citep{papadopoulos2002inductive,lei2018distribution,vovk2005algorithmic}
computes scores $S_i=|Y_i-\widehat\mu(X_i)|$ on a held-out calibration set and takes the
empirical quantile
\begin{equation}
	q^{\mathrm{split}}=S_{(\lceil (m+1)(1-\alpha)\rceil)},\qquad
	C(x)=\big[\widehat\mu(x)-q^{\mathrm{split}},\ \widehat\mu(x)+q^{\mathrm{split}}\big],
	\label{eq:split}
\end{equation}
which satisfies $\Prob\{Y\in C(X)\}\ge1-\alpha$ whenever the calibration and test points are
exchangeable. The guarantee is finite-sample, distribution-free and model-agnostic, which
accounts for the method's wide adoption; \citet{shafer2008tutorial} give a tutorial treatment
and \citet{angelopoulos2023gentle} and \citet{fontana2023conformal} recent surveys. The
limitation is that \eqref{eq:split} weights every calibration point equally, so any departure
from exchangeability translates directly into a coverage error. Exact conditional coverage is
in addition unattainable in finite samples without further assumptions
\citep{barber2021limits}, which motivates the approximate weighted constructions considered
below.

\subsection{Reweighting for distribution shift}
Under covariate shift, with test covariates drawn from $\widetilde P_X$ rather than $P_X$,
validity is restored by weighting each calibration point by the likelihood ratio
$r(x)=\mathrm d\widetilde P_X/\mathrm dP_X(x)$ \citep{tibshirani2019conformal}:
\begin{equation}
	\begin{split}
		q^{\mathrm{w}}=\inf\Big\{s:\ \sum_{i=1}^{m}\bar w_i\,\ind\{S_i\le s\}+\bar w_{m+1}\ \ge\ 1-\alpha\Big\},
		\\
		\bar w_i=\frac{r(X_i)}{\sum_{j}r(X_j)+r(X_{m+1})},
		\label{eq:weighted}
	\end{split}
\end{equation}
where the atom $\bar w_{m+1}$ at the test point makes the construction exact under known $r$.
Analogous corrections exist for label shift \citep{podkopaev2021distribution}, and
\citet{barber2023conformal} extend the idea to arbitrary non-exchangeable data, bounding the
coverage gap by a weighted total-variation distance between the calibration and test laws. In
practice $r$ must be estimated, so the guarantee degrades with the estimation error, a term
that enters Theorem~\ref{thm:rate} explicitly. The construction in \eqref{eq:weighted} corrects
the covariate marginal alone and cannot repair a change in the conditional law $P_{Y\mid X}$.

\subsection{Localization and conditional heterogeneity}
Localized conformal methods  assign larger weights to calibration observations close to the prediction point, replacing $r(X_i)$ in \eqref{eq:weighted} with a
kernel,
\begin{equation}
	w^{\mathrm{loc}}_i(x)=K_h\big(\|X_i-x\|\big),
	\label{eq:local}
\end{equation}
as in localized conformal prediction \citep{guan2023conformal}. Conformalized quantile
regression \citep{romano2019conformalized} and normalized scores \citep{han2022split} pursue
the same objective through the score rather than the weights, calibrating the residual by a
conditional quantile or scale estimate. Localization involves a bias-variance trade-off: a
small bandwidth $h$ adapts sharply but reduces the effective sample size
$\mathrm{ESS}=(\sum_i w_i)^2/\sum_i w_i^2$, which destabilizes the quantile. \DRACP{} retains
\eqref{eq:local} but selects $h$ adaptively subject to an explicit $\mathrm{ESS}$ floor.

\subsection{Online adaptation of the significance level}
For sequentially evolving distributions, adaptive conformal inference (ACI)
\citep{gibbs2021adaptive, gibbs2024conformal} leaves the calibration weights untouched and instead
drives the nominal level from realized coverage,
\begin{equation}
	\alpha_{t+1}=\alpha_t+\gamma\big(\alpha-\ind\{Y_t\notin C_t\}\big),
	\label{eq:aci_rw}
\end{equation}
which forces the long-run empirical miscoverage to $\alpha$ for any data sequence. Since the
single step size $\gamma$ governs the trade-off between responsiveness and stability, fully
adaptive conformal inference (FACI) \citep{gibbs2024conformal,Zaffran2022} runs a set of ACI
experts with different rates $\{\gamma_k\}$ and aggregates their levels by exponential
weighting,
\begin{equation}
	\alpha_t=\sum_k p_{k,t}\,\alpha^{(k)}_t,\qquad
	p_{k,t+1}\propto p_{k,t}\exp\!\big(-\zeta\,\ell_{k,t}\big),
	\label{eq:faci_rw}
\end{equation}
with $\ell_{k,t}$ being the pinball loss of expert $k$. Strongly adaptive online conformal
prediction (SAOCP) \citep{Bhatnagar2023}  maintains experts with geometrically growing
lifetimes in order to obtain regret guarantees on every subinterval, and conformal PID control
\citep{Angelopoulos2023pid,gibbs2024conformal} augments the proportional term in
\eqref{eq:aci_rw} with integral and derivative terms,
\begin{equation}
	r_{t+1}=r^{\mathrm{base}}_{t+1}+k_P e_t+k_I\textstyle\sum_{s\le t}e_s+k_D(e_t-e_{t-1}),
	\qquad e_t=\ind\{Y_t\notin C_t\}-\alpha,
	\label{eq:pid_rw}
\end{equation}
acting directly on the interval radius. \citet{angelopoulos2024online} analyze decaying step
sizes for the same update. A parallel line of work constructs the intervals themselves for
dependent data. The jackknife+ \citep{barber2021predictive} replaces the single calibration
split with leave-one-out ensembling and yields coverage guarantees without data splitting;
further contributions include ensemble batch prediction intervals and sequential predictive
conformal inference \citep{xu2021conformal,xu2023sequential}, temporal quantile adjustment
\citep{lin2022conformal}, and retrieval-based schemes using modern Hopfield networks
\citep{auer2023conformal}. Recent work extends these ideas to explicit change points,
foundation-model forecasters and optimal-transport corrections under unlabeled shift.

\subsection{Simultaneous shifts and evaluation criteria}
Equations \eqref{eq:weighted}, \eqref{eq:local} and \eqref{eq:aci_rw}--\eqref{eq:pid_rw}  primarily target different sources of non-exchangeability: covariate shift, local heterogeneity, and temporal drift.
Economic series exhibit these departures simultaneously. The
euro-area inflation surge altered the covariate distribution, the conditional response and the
prevailing regime at the same time, and the estimation errors so induced interact rather than
adding independently. \DRACP{} composes the three weighting mechanisms together with a regime
term inside a single weighted quantile, and places a self-tuning controller of the form
\eqref{eq:faci_rw} on top of the resulting calibration, so that the online correction operates
on already-adapted weights rather than on a global pool. Theorem~\ref{thm:rate} quantifies the
contribution of each component to the coverage gap, which renders the composition analyzable.

Evaluation is conducted with the interval (Winkler) score \citep{winkler1972decision}, a
strictly proper scoring rule for central prediction intervals \citep{gneiting2007strictly}
that penalizes width and miscoverage jointly, and is therefore the appropriate criterion when
methods differ in both. Method rankings depend on the evaluation criterion.
\citet{min2026questioning} show that evaluating conformal procedures on coverage and interval
length alone is unreliable: a randomized construction can shorten the average interval while
preserving marginal coverage, at the cost of returning uninformative intervals for a subset of
cases, and they propose interval stability as an additional diagnostic. The interval score is
not vulnerable in this way, since it prices miscoverage at $2/\alpha$ per unit of exceedance.
It remains an average, however, and an average taken over series can conceal the distribution
of coverage across them: a method may attain the best mean score while undercovering
substantially on a subset. We therefore report empirical coverage alongside the interval score in
Section~\ref{sec:results}.

Conformal prediction has begun to appear in economic and energy applications, notably in
day-ahead and intraday electricity-price intervals \citep{kath2021conformal}.
\citet{sabashvili2026conformal} surveys and benchmarks conformal procedures for time series on
a monthly sales panel with an \textsc{AutoARIMA} forecaster, and reports that a multi-step
split conformal method attains nominal coverage with the best efficiency, a result consistent
with the findings at longer horizons reported in Section~\ref{sec:results}. A systematic
comparison across a broad macroeconomic panel, against the modern online baselines discussed
above, with those baselines validated against their reference implementations, and across
multiple base forecasters and horizons, is not available in the existing literature.
Sections~\ref{sec:data}--\ref{sec:results} provide such a comparison.

\section{Dynamic regime-aware conformal prediction}
\label{sec:method}

This section presents the proposed procedure. We first describe the main steps of DRACP and then give its formal specification. Figure~\ref{fig:workflow} summarizes the pipeline.

Consider a standard forecasting setup in which a point forecaster is estimated on a training
block and produces a one-step-ahead prediction for each subsequent period. Split conformal
prediction converts these predictions into intervals by taking a high quantile of the
forecaster's past errors and adding it to the current forecast. The validity of this
construction rests on past and present errors being draws from a common distribution. In
economic data that condition is frequently violated. The recent past may resemble current
conditions more closely than the distant past does, volatility differs across periods, and the
economy alternates between recurring regimes.

\DRACP{} retains the split-conformal construction but replaces the equal treatment of past
errors with a weighting scheme. Each past error receives a weight reflecting the relevance of
the corresponding historical period to the period being forecast. Relevance is measured along
four dimensions. A recency factor discounts older observations geometrically. A density-ratio
factor, estimated by a classifier trained to distinguish recent from older predictors,
increases the weight of past periods whose conditions resemble current conditions. A local
kernel factor increases the weight of past periods whose predictor values lie close to the
current ones, subject to a safeguard against collapse of the effective sample. A regime factor,
obtained from a fitted mixture model, increases the weight of past periods assigned to the same
latent regime as the present. The four factors are multiplied and normalized, and the interval
radius is the weighted rather than the unweighted quantile of past errors.

A second mechanism operates on the significance level. Because weighting cannot eliminate all
miscoverage, a controller monitors the rate at which recent intervals have covered the outcome
and adjusts the working significance level accordingly. The same device underlies the online
conformal methods used as comparators, and it is retained here for a specific reason: the
weights respond at the moment conditions change, while the controller corrects the error that
remains. The weights determine the contribution of past errors to calibration, whereas the controller determines the quantile level.

The procedure is causal throughout. At each step only information available at that step is
used, all nuisance quantities are re-estimated from data preceding the forecast origin, and no
future observation influences any interval. The formal treatment follows.

\begin{figure}[t]
	\centering
	\resizebox{\textwidth}{!}{%
		\begin{tikzpicture}[
			font=\small,
			box/.style={draw, rounded corners=2pt, align=center, minimum height=8mm, inner sep=3pt},
			wt/.style={box, fill=black!4, text width=27mm},
			main/.style={box, fill=black!8, text width=30mm},
			->/.default=stealth,
			arr/.style={-{Stealth[length=2mm]}, thick}
			]
			\node[main] (data) at (0,0) {Series up to $t$\\ \scriptsize training + calibration};
			\node[main] (fc)   at (0,-1.6) {Point forecaster\\ \scriptsize $\widehat\mu,\widehat\sigma$};
			\node[main] (sc)   at (0,-3.2) {Past errors\\ \scriptsize scores $S_i$};
			
			\node[wt] (w1) at (4.6,0.75)  {\textbf{Recency}\\ \scriptsize older $\to$ less weight};
			\node[wt] (w2) at (4.6,-0.75) {\textbf{Density ratio}\\ \scriptsize similar conditions};
			\node[wt] (w3) at (4.6,-2.25) {\textbf{Local kernel}\\ \scriptsize similar predictors};
			\node[wt] (w4) at (4.6,-3.75) {\textbf{Regime}\\ \scriptsize same latent state};
			
			\node[main] (comb) at (9.0,-1.5) {Combined weights\\ \scriptsize $\bar w_{i,t}$, normalized};
			\node[main] (q)    at (9.0,-3.6) {Weighted quantile\\ \scriptsize radius $q_t$};
			\node[main] (ctl)  at (9.0,0.45) {Controller\\ \scriptsize adjusts $\alpha_t$};
			\node[main] (out)  at (13.0,-3.6) {Interval\\ \scriptsize for period $t$};
			
			\draw[arr] (data) -- (fc);
			\draw[arr] (fc) -- (sc);
			\foreach \w in {w1,w2,w3,w4} \draw[arr] (data.east) -- (\w.west);
			\foreach \w in {w1,w2,w3,w4} \draw[arr] (\w.east) -- (comb.west);
			\draw[arr] (sc.east) -- (q.west);
			\draw[arr] (comb) -- (q);
			\draw[arr] (ctl) -- (comb);
			\draw[arr] (q) -- (out);
			\draw[arr] (out.east) -- ++(0.5,0) |- ++(0,2.6)
			node[pos=0.62,above,font=\scriptsize,inner sep=2pt]{realized coverage feedback}
			-| (ctl.north);
	\end{tikzpicture}}
	\caption{The \DRACP{} pipeline at a single forecast origin $t$. A standard point forecaster
		produces errors on past periods; four weighting mechanisms score how relevant each past period
		is to the present; the interval is the weighted quantile of those errors. A controller adjusts
		the working significance level from realized coverage. All quantities are computed from
		information available at $t$.}
	\label{fig:workflow}
\end{figure}
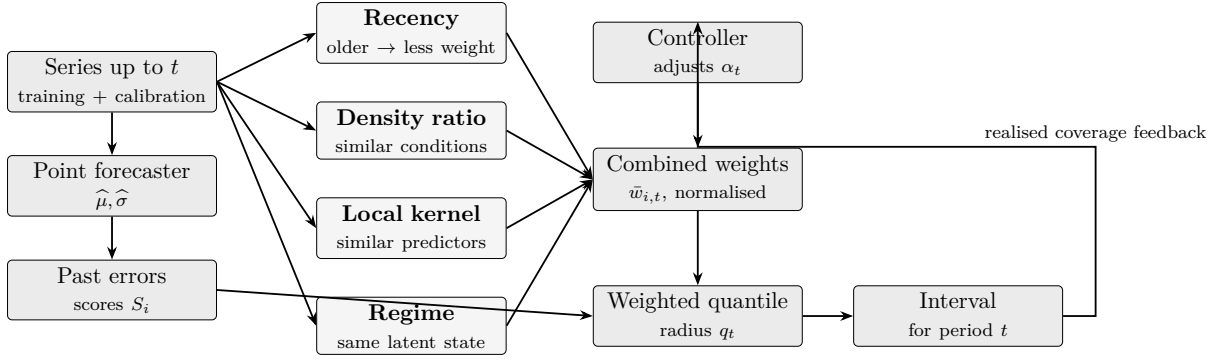

Observations $(X_t,Y_t)\in\R^d\times\R$ arrive sequentially. A point forecaster
$\widehat\mu$ and a scale estimate $\widehat\sigma$ are trained on an initial block;
a rolling calibration buffer $\mathcal{D}_t$ of the most recent observations is
maintained. We use locally scaled residual scores
$S_i=|Y_i-\widehat\mu(X_i)|/\widehat\sigma(X_i)$. Given nonnegative weights $w_{i,t}$
on the calibration points and target level $1-\alpha$, the weighted conformal radius
is the $(1-\alpha_t)$ quantile of the weighted score distribution in which the test point
contributes an atom at $+\infty$ \citep{tibshirani2019conformal},
\begin{equation}
	q_t \;=\; Q_{1-\alpha_t}\Big(\textstyle\sum_{i} \bar w_{i,t}\,\delta_{S_i}
	\;+\; \bar w_{t,t}\,\delta_{+\infty}\Big)
	\;=\; \inf\Big\{s\in\R:\ \textstyle\sum_{i} \bar w_{i,t}\,\ind\{S_i\le s\}\ \ge\ 1-\alpha_t\Big\},
	\label{eq:quantile}
\end{equation}
with $\bar w_{\cdot,t}=w_{\cdot,t}/(\sum_i w_{i,t}+w_{t,t})$ and the convention
$q_t=+\infty$ when the calibration mass $\sum_i\bar w_{i,t}=1-\bar w_{t,t}$ falls short of
$1-\alpha_t$, which occurs exactly when $\bar w_{t,t}>\alpha_t$. The test-point mass is placed at $+\infty$, because its score is unknown at prediction time.  With $m$ equally weighted calibration points, unit test
weight and $\alpha_t=0.1$, the correct rule returns the $m$-th smallest score once
$1/(m+1)\le\alpha_t$, whereas placing the atom inside the CDF would return the $(m-1)$-th and
deliver coverage of approximately $1-\alpha_t-1/(m+1)$. The resulting interval is
$C_t(X_t)=[\widehat\mu(X_t)\pm q_t\,\widehat\sigma(X_t)]$.

\begin{table}[htbp]
	\centering
	\caption{Notation.}
	\label{tab:notation}
	\footnotesize
	\begin{tabular}{ll}
		\toprule
		Symbol & Meaning \\
		\midrule
		$(X_t,Y_t)$ & covariates and outcome at time $t$ \\
		$\widehat\mu,\widehat\sigma$ & point forecaster and conditional scale estimate \\
		$S_i$ & scaled conformity score $|Y_i-\widehat\mu(X_i)|/\widehat\sigma(X_i)$ \\
		$\mathcal D_t$ & rolling calibration buffer of size $m$ \\
		$w_{i,t},\bar w_{i,t}$ & composite and normalized calibration weights \\
		$\lambda$ & temporal decay factor \\
		$\widetilde r_t$ & clipped density ratio, bounds $r_{\min},r_{\max}$ \\
		$K_{h_t}$ & localization kernel with adaptive bandwidth $h_t$ \\
		$\bm\pi_i$ & posterior over $K$ latent regimes at $x_i$ \\
		$\beta$ & regime-similarity sharpness exponent (default $1$) \\
		$w_{t,t}$ & weight of the test-point atom (set to $1$ throughout) \\
		$\mathrm{ESS}_t$ & effective calibration sample size, floor $\underline n$ \\
		$q_t$ & weighted conformal radius; $C_t$ the interval \\
		$\alpha,\alpha_t$ & target and online-adapted miscoverage level \\
		$\eta,\gamma_k$ & controller learning rate(s) \\
		$\mathrm{IS}_\alpha$ & interval (Winkler) score \\
		\bottomrule
	\end{tabular}
\end{table}

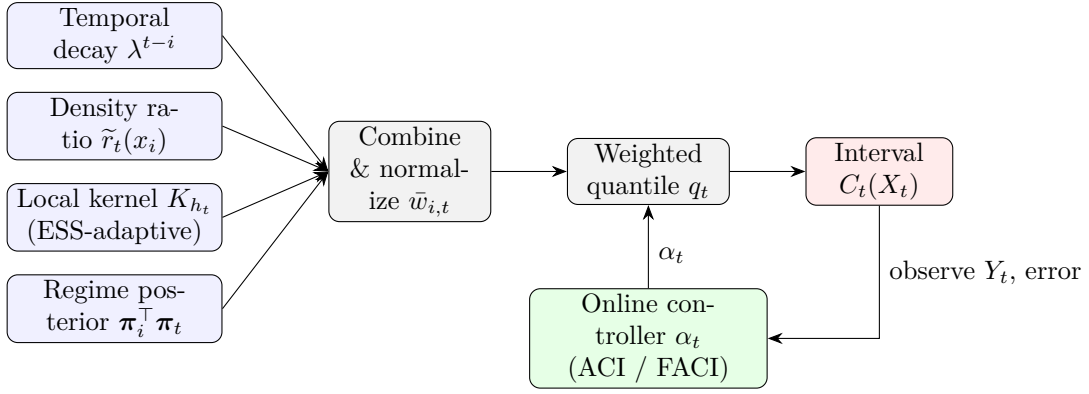
\begin{figure}[t]
	\centering
	\begin{tikzpicture}[font=\small,>={Stealth[length=2mm]},
		comp/.style={draw,rounded corners,fill=blue!6,align=center,text width=27mm,minimum height=8mm,inner sep=2pt},
		op/.style={draw,rounded corners,fill=gray!10,align=center,text width=20mm,minimum height=8mm,inner sep=2pt},
		ob/.style={draw,rounded corners,fill=red!8,align=center,text width=18mm,minimum height=8mm,inner sep=2pt},
		ctl/.style={draw,rounded corners,fill=green!10,align=center,text width=30mm,minimum height=8mm,inner sep=2pt}]
		\node[comp] (t) {Temporal decay $\lambda^{t-i}$};
		\node[comp,below=3mm of t] (r) {Density ratio $\widetilde r_t(x_i)$};
		\node[comp,below=3mm of r] (l) {Local kernel $K_{h_t}$\\(ESS-adaptive)};
		\node[comp,below=3mm of l] (g) {Regime posterior $\bm\pi_i^\top\bm\pi_t$};
		\node[op,right=14mm of r,yshift=-6mm] (c) {Combine \& normalize $\bar w_{i,t}$};
		\node[op,right=10mm of c] (q) {Weighted quantile $q_t$};
		\node[ob,right=10mm of q] (i) {Interval $C_t(X_t)$};
		\node[ctl,below=11mm of q] (a) {Online controller $\alpha_t$\\ (ACI / FACI)};
		\foreach \x in {t,r,l,g} \draw[->] (\x.east) -- (c.west);
		\draw[->] (c) -- (q);
		\draw[->] (q) -- (i);
		\draw[->] (a.north) -- node[right,pos=0.4]{$\alpha_t$} (q.south);
		\draw[->] (i.south) |- (a.east) node[pos=0.25,right]{observe $Y_t$, error};
	\end{tikzpicture}
	\caption{The \DRACP{} calibration step. Four weighting components, temporal decay,
		density-ratio correction for covariate shift, an ESS-adaptive local kernel, and
		regime-similarity weighting, are combined into a single normalized weight vector that
		defines a weighted conformal radius $q_t$ and interval $C_t$. An online controller adapts
		the significance level $\alpha_t$ from realized coverage, closing the loop.}
	\label{fig:framework}
\end{figure}

\DRACP{} forms $w_{i,t}$ as a stabilized product of four components (Figure~\ref{fig:framework}), each
addressing a distinct departure from exchangeability:
\begin{equation}
	w_{i,t}\ \propto\ \underbrace{\lambda^{\,t-i}}_{\text{temporal}}\;
	\underbrace{\widetilde r_t(X_i)}_{\text{density ratio}}\;
	\underbrace{K_h\!\big(\|X_i-X_t\|\big)}_{\text{local}}\;
	\underbrace{\big(\bm\pi_i^\top\bm\pi_t\big)^{\beta}}_{\text{regime}} .
	\label{eq:weights}
\end{equation}
The temporal factor discounts older observations geometrically. The density ratio
$\widetilde r_t$ is a clipped estimate, obtained from a probabilistic classifier that
discriminates recent (target) from older (source) covariates and subsequently truncated,
\begin{equation}
	\widetilde r_t(x)=\min\big\{r_{\max},\ \max\{r_{\min},\ \widehat r_t(x)\}\big\},
	\label{eq:clip}
\end{equation}
which corrects covariate shift while bounding the influence of any single calibration point
(equation~\eqref{eq:clip}). The exponent $\beta$ in the regime factor governs how sharply
similarity is rewarded. We set $\beta=1$ throughout, so that the factor reduces to the
posterior inner product; the parameter is retained in the implementation because sharper
values can be useful when regimes are poorly separated. We set $w_{t,t}=1$, giving the test point the weight of one calibration observation.
The Gaussian kernel $K_h$ localizes
calibration to the neighborhood of $X_t$; its bandwidth is enlarged adaptively until
the effective sample size $\mathrm{ESS}(w)=(\sum_i w_i)^2/\sum_i w_i^2$ exceeds a
floor, trading locality for stability. The regime factor is the inner product of the
posterior regime memberships $\bm\pi_i,\bm\pi_t$ from a causal Gaussian-mixture model,
and up-weights calibration points in the same latent regime as the query. Setting any
component to one recovers the corresponding simpler method, so that \DRACP{} nests weighted,
localized and regime baselines.

Explicitly, with $z=$ standardized covariates, the four factors are
\begin{align}
	\text{temporal:}\quad & w^{\mathrm{time}}_{i,t}=\lambda^{\,t-i},\qquad \lambda\in(0,1],\\
	\text{density ratio:}\quad & \widehat r_t(x)=\frac{\widehat p_t(x)}{1-\widehat p_t(x)}\,\frac{n_{\mathrm{src}}}{n_{\mathrm{tgt}}},\quad
	\widetilde r_t=\min\{r_{\max},\max\{r_{\min},\widehat r_t\}\},\\
	\text{local:}\quad & w^{\mathrm{loc}}_{i,t}=\exp\!\Big(-\tfrac{\|z_i-z_t\|^2}{2h_t^2}\Big),\quad
	h_t=\min\{h\in\mathcal H:\ \mathrm{ESS}(w)\ge \underline n\},\\
	\text{regime:}\quad & w^{\mathrm{reg}}_{i,t}=\big(\bm\pi_i^\top\bm\pi_t\big)^{\beta},\qquad
	\bm\pi_i=\text{GMM posterior over $K$ regimes at }x_i,
\end{align}
where $\widehat p_t$ is a probabilistic classifier separating recent (target) from older
(source) covariates. The combined weight is the stabilized, normalized product
\begin{equation}
	w_{i,t}=w^{\mathrm{time}}_{i,t}\,\widetilde r_t(x_i)\,w^{\mathrm{loc}}_{i,t}\,w^{\mathrm{reg}}_{i,t},
	\qquad
	\bar w_{i,t}=\frac{w_{i,t}}{\sum_j w_{j,t}+w_{t,t}},
	\label{eq:combined}
\end{equation}
with a test-point mass $w_{t,t}$ that yields the finite-sample atom in the weighted quantile.
Figure~\ref{alg:dracp} states the resulting procedure in pseudocode.

Because weighting is imperfect under rapid change, the nominal level is adapted online
by adaptive conformal inference \citep{gibbs2021adaptive}, whose update \eqref{eq:aci} is
\begin{equation}
	\alpha_{t+1}=\alpha_t+\eta\,\big(\alpha-\ind\{Y_t\notin C_t\}\big),
	\label{eq:aci}
\end{equation}
which drives the long-run empirical miscoverage to $\alpha$ regardless of weight quality. When the localized
weights are too concentrated (ESS below its floor or a single weight above a cap) the
method reverts to globally weighted calibration. The fallback is introduced for numerical stability and does not restore exact conformal validity under concept drift, estimated weights, or serial dependence. A single fixed
learning rate $\eta$ may, however, prove either too slow or too volatile depending on the
series. We therefore adopt as the default throughout the paper a \emph{self-tuning} controller
that replaces the single rate by an exponential-weights aggregation over a grid of learning
rates. It is described here in full, since it is a heuristic introduced in this paper rather
than an implementation of a published algorithm. Each expert $k$ maintains a level
$\alpha_{k,t}$ updated by the adaptive-conformal rule at its own rate $\gamma_k$; the level
actually used is the weighted average $\alpha_t=\sum_k \omega_{k,t}\alpha_{k,t}$; and after
observing the single realized miscoverage indicator $\mathrm{err}_t=\ind\{Y_t\notin C_t\}$ of
the \emph{aggregate} interval, every expert incurs the pinball loss of
$\mathrm{err}_t-\alpha_{k,t}$ and the weights are updated by exponentiated gradient with a
small uniform-mixing term.

Two features distinguish this construction from the fully-adaptive aggregation of
\citet{gibbs2024conformal}. First, all experts are scored against one shared aggregate indicator, whereas the
published algorithm assigns each expert its own interval and its own feedback. Second, the loss
is computed from the miscoverage indicator directly instead of through a critical level.
Consequently the controller nests adaptive conformal inference: restricting the grid to a
single rate recovers it exactly, which is the case covered by Theorem~\ref{thm:controller}(a).
We refer to it throughout as the self-tuning controller, and we report FACI separately, as an
independent baseline, under the same convention.

\begin{figure}[t]
	\centering
	\fbox{\begin{minipage}{0.92\textwidth}
			\small\textbf{Algorithm: \DRACP{} one-step prediction and update at time $t$}
			\begin{enumerate}\setlength{\itemsep}{1pt}
				\item Form point forecast $\widehat\mu(X_t)$ and scale $\widehat\sigma(X_t)$; retrieve the
				recent calibration buffer and its scaled scores $S_i$.
				\item Compute the four weight factors in \eqref{eq:weights}: temporal decay, clipped
				density ratio (target vs.\ source covariates), Gaussian kernel with bandwidth enlarged
				until $\mathrm{ESS}\ge\underline{n}$, and regime similarity $\bm\pi_i^\top\bm\pi_t$;
				combine and normalize to $\bar w_{i,t}$.
				\item Set $q_t$ to the $(1-\alpha_t)$ weighted score quantile with the test-point atom; if
				the weights are too concentrated, fall back to global weighting.
				\item Output $C_t=[\widehat\mu(X_t)\pm q_t\,\widehat\sigma(X_t)]$ (or the asymmetric variant).
				\item Observe $Y_t$; update the controller $\alpha_{t+1}=\alpha_t+\eta(\alpha-\ind\{Y_t\notin C_t\})$
				and append $(X_t,Y_t)$ to the buffer.
			\end{enumerate}
	\end{minipage}}
	\caption{The \DRACP{} calibration step in pseudocode.}
	\label{alg:dracp}
\end{figure}

Economic forecast errors are frequently skewed, with downside risk
predominating in activity series and upside risk in prices during inflationary episodes, in
which case a symmetric radius is inefficient. The asymmetric variant replaces the single
quantile by two one-sided weighted quantiles of the \emph{signed} scores
$\widetilde S_i=(Y_i-\widehat\mu(X_i))/\widehat\sigma(X_i)$ at level $\alpha/2$,
\begin{equation}
	q^+_t=\mathcal Q_{1-\alpha/2}\big(\{\widetilde S_i\},\bar w_{\cdot,t}\big),\quad
	q^-_t=\mathcal Q_{1-\alpha/2}\big(\{-\widetilde S_i\},\bar w_{\cdot,t}\big),\quad
	C_t=\big[\widehat\mu-q^-_t\widehat\sigma,\ \widehat\mu+q^+_t\widehat\sigma\big],
	\label{eq:asym}
\end{equation}
which retains marginal validity by a union bound over the two tails. In our experiments this
variant improves performance only on markedly skewed series, and it is therefore offered as an
option and is not enabled by default.

On series with little exploitable structure, composite weighting introduces estimation variance
without a corresponding shift to exploit, and an unweighted online predictor is preferable.
Selection between the two without prior knowledge of which case applies can be achieved by
running \DRACP{} alongside unweighted online predictors $j=1,\dots,J$ and issuing the
exponentially-weighted radius
\begin{equation}
	r^\star_t=\sum_j \omega_{j,t}\,r^{(j)}_t,\qquad
	\omega_{j,t+1}\propto \omega_{j,t}\exp\!\big(-\zeta\,\mathrm{IS}^{(j)}_\alpha(t)\big),
	\label{eq:star}
\end{equation}
where $\mathrm{IS}^{(j)}_\alpha(t)$ is the realized interval score of predictor $j$. By standard
prediction-with-expert-advice regret bounds (Proposition~\ref{prop:agg}), the aggregate is
asymptotically no worse than the best constituent on every series. The variant is a deployment safeguard, and is reported separately from the main results.

The method evaluated throughout Section~\ref{sec:results} is \DRACP{} with the self-tuning
controller and symmetric scores. This is the default configuration, and the one referred to
simply as \DRACP{}. The asymmetric variant \eqref{eq:asym} and \DRACP-Select \eqref{eq:star}
are options, reported separately in the ablation of Table~\ref{tab:abl}. In the released
software they are selected by keyword, respectively \texttt{asymmetric=True} and the method
name \texttt{select} mode;  \ref{app:usage} gives the full interface.

With a calibration buffer of $m$ points in $d$ dimensions and $K$
regimes, one prediction costs $O(md)$ for the weight product in \eqref{eq:combined}, $O(mK)$
for the regime term, and $O(m\log m)$ for the weighted quantile, plus the amortized cost of
periodically refitting the density-ratio classifier and the mixture model; the controller is
$O(K)$. The per-step cost is therefore $O\!\big(m(d+K+\log m)\big)$ with $O(md)$ memory,
reduced to $O(Bd)$ by capping the buffer at $B$. The dominant practical cost is the
Gaussian-mixture regime fit, which is re-estimated on the expanding history and accounts for
roughly half of \DRACP{}'s per-prediction time; the density-ratio classifier is the second
largest at about a quarter, and the weighted quantile itself is negligible. Since the regime
posterior changes slowly, refitting the mixture model on a fixed schedule rather than at every
step is the most effective means of reducing cost. Empirically, \DRACP{} costs approximately
six times a split-conformal prediction (\ref{app:figures}, Figure~\ref{fig:cost}). This
overhead is negligible at the monthly and hourly frequencies of macroeconomic and energy
forecasting, but the method is not recommended for high-frequency applications, where its
advantages are in any case absent.

\section{Theoretical properties}
\label{sec:theory}
This section establishes the theoretical properties of the construction: finite-sample validity
under oracle weights, a coverage-gap bound with explicit rates in the nuisance-estimation errors,
the average-miscoverage and no-regret properties of the self-tuning controller, approximate
regime-conditional coverage, a deterministic lower bound on the effective calibration size, and a
regret bound for the optional selection hedge. All proofs are given in \ref{app:proofs}. These results separate the roles of the two adaptive mechanisms: the weights reduce the
coverage gap up to estimation error, while the online controller regulates the cumulative
empirical miscoverage that the weights leave behind. Part~(a)
of Theorem~\ref{thm:controller} below controls a time-average of the miscoverage indicator, and a
controlled time-average is compatible with substantial period-by-period miscoverage of either
sign; the controller therefore regulates cumulative miscoverage rather than removing
miscalibration.

\begin{proposition}[Validity under oracle importance weights]
	\label{prop:coverage}
	Consider the weighted conformal construction \eqref{eq:quantile} run with weights
	$w^{\mathrm{orc}}_{i,t}=\mathrm{d}P^{\mathrm{test}}_t/\mathrm{d}P^{\mathrm{cal}}(X_i)$, the
	exact likelihood ratio between the calibration and test covariate distributions at time $t$.
	If in addition the scores are almost surely distinct and are exchangeable conditional on the
	covariates, then $\Prob\{Y_t\in C_t(X_t)\}\ge 1-\alpha_t$ for every $t$.
\end{proposition}

\begin{remark}
	\label{rem:oracle}
	Proposition~\ref{prop:coverage} applies to oracle importance weights, whereas \DRACP{} uses estimated composite weights. The implemented
	weights \eqref{eq:weights} are a product of four factors: temporal decay, an estimated and
	clipped density ratio, a localization kernel and a regime-similarity term. Only the second is an
	estimate of a likelihood ratio. The remaining three are \emph{relevance} weights, introduced to
	reduce the effective distance between the calibration sample and the current conditional
	distribution, and their product is in general \emph{not} itself a likelihood ratio between the
	calibration and test distributions. 
	Proposition \ref{prop:coverage} therefore provides an idealized benchmark in which the conformal calibration step is valid.  Theorem~\ref{thm:rate} quantifies the additional coverage error associated with replacing the oracle weights by the \DRACP{} weights.
\end{remark}

\begin{assumption}[Regularity and nuisance estimation]
	\label{ass:reg}
	Conditional on the past, at time $t$: (i) the conformity score has a density bounded below
	by $f_{\min}>0$ in a neighborhood of its weighted $(1-\alpha_t)$-quantile; (ii) the clipped
	density-ratio estimate has $L^1$ error $\delta^{r}_t=\E_{\mathrm{cal}}|\widetilde r_t-r_t|$
	and clipping bias $\delta^{c}_t=\E_{\mathrm{cal}}(r_t-\widetilde r_t)_+$; (iii) the estimated
	regime posteriors approximate the oracle posteriors in $L^1$, with
	$\delta^{g}_t=\E_{\mathrm{cal}}\|\widehat{\bm\pi}_i-\bm\pi_i\|_1+\|\widehat{\bm\pi}_t-\bm\pi_t\|_1$;
	(iv) $x\mapsto F_{S\mid X=x}$ is $L$-Lipschitz in total variation and the localization kernel
	has bandwidth $h_t$; (v) the residual conditional drift is
	$\Delta_t=\sup_x \mathrm{TV}(F^{\mathrm{cal}}_{S\mid x},F^{\mathrm{test}}_{S\mid x})$.
\end{assumption}

\begin{theorem}[Finite-sample coverage gap and rates]
	\label{thm:rate}
	Under Assumption~\ref{ass:reg}, with probability at least $1-\eta$ over the calibration draw,
	\begin{equation}
		\big|\Prob\{Y_t\in C_t(X_t)\}-(1-\alpha_t)\big|\ \le\
		\big(\delta^{r}_t+\delta^{c}_t+L\,h_t+\delta^{g}_t+\Delta_t\big)
		\ +\ \sqrt{\frac{\log(2/\eta)}{2\,\mathrm{ESS}_t}}
		\ +\ \max_i \bar w_{i,t},
		\label{eq:rate}
	\end{equation}
	where $\mathrm{ESS}_t=(\sum_i w_{i,t})^2/\sum_i w_{i,t}^2$ is the effective calibration size.
	The final term is the discretization cost of taking a quantile of a discrete weighted
	distribution: the weighted empirical CDF jumps at each calibration score, so $q_t$ attains
	$1-\alpha_t$ only up to the largest normalized weight. This term vanishes only as the weights
	become diffuse, and it is the quantity that the ESS floor of Section~\ref{sec:method} is designed
	to control. With fixed temporal decay $\lambda < 1$ the bound in \eqref{eq:rate} does not vanish as $m \to \infty$.
	Both the stochastic term and $\max_i\bar w_{i,t}$ are
	governed by $\mathrm{ESS}_t$, not by the nominal calibration size $m$, and under a \emph{fixed}
	temporal decay $\lambda<1$ the effective sample size does not grow with $m$: the geometric
	weights give $\mathrm{ESS}_m\to(1+\lambda)/(1-\lambda)$, a constant. Under a fixed decay the
	bound therefore does not vanish. This is a limitation of the procedure as implemented rather than
	an artifact of the proof. A vanishing gap requires the decay to soften with sample size,
	$\lambda=\lambda_m\to1$ slowly enough that $\mathrm{ESS}_t\to\infty$. Under that condition, and
	with $\delta^{r}_t=O(\mathrm{ESS}_t^{-\beta_r})$, $\delta^{g}_t=O(\mathrm{ESS}_t^{-\beta_g})$,
	clipping asymptotically inactive, and the bandwidth chosen to balance the localization bias
	$Lh_t$ against the stochastic term, which for a $d$-dimensional kernel gives
	$h_t\asymp \mathrm{ESS}_t^{-1/(d+2)}$, the coverage gap is
	\begin{equation}
		O\!\big(\mathrm{ESS}_t^{-\min(\beta_r,\ \beta_g,\ 1/(d+2))}+\max_i\bar w_{i,t}+\Delta_t\big),
		\label{eq:raterate}
	\end{equation}
	vanishing at the slowest nuisance-estimation rate, up to the irreducible conditional drift
	$\Delta_t$. For $d=1$ this reads $\mathrm{ESS}_t^{-1/3}$. The rate is stated in effective rather
	than nominal sample size throughout, since the effective size is the quantity that the weighting
	controls.
\end{theorem}

\begin{remark}[Dependence, data-driven weights, and the scope of Theorem~\ref{thm:rate}]
	\label{rem:dependence}
	Equation \eqref{eq:rate} relies on three assumptions that are particularly relevant for the empirical application.

	First, the stochastic term is a weighted Dvoretzky--Kiefer--Wolfowitz bound, and is therefore
	proved under the assumption that, conditional on the information used to construct the weights,
	the calibration scores are independent. Macroeconomic calibration scores are serially dependent,
	so \eqref{eq:rate} as stated is a statement about the independent case. Under dependence the term
	does not follow from $\mathrm{ESS}_t$ alone; it must be replaced by a concentration inequality
	for dependent sequences, which for a $\beta$-mixing calibration stream with coefficients decaying
	at rate $\beta(k)$ inflates the term by a factor reflecting the mixing time and degrades the
	effective sample size from $\mathrm{ESS}_t$ to approximately $\mathrm{ESS}_t/\tau$, with $\tau$
	the mixing time.  The shape of the bound survives: one additive bias term per weighting component, together with
	a stochastic term that shrinks in effective sample size. The constants and the exponent on that
	stochastic term change, however, and we make no claim to the $\tfrac12$ exponent under
	dependence.
	
	Second, the weights are themselves estimated from the past: the density-ratio classifier, the
	regime posterior and the bandwidth are all fitted on data preceding $t$. The bound should
	therefore be read conditionally on the sigma-algebra generated by that past, which is how
	Assumption~\ref{ass:reg} is phrased, and the nuisance errors $\delta^r_t,\delta^g_t$ are
	accordingly conditional quantities. A fully unconditional statement would require a
	sample-splitting or cross-fitting argument, which is neither undertaken here nor performed by the
	implementation.
	
	Third, the rate in \eqref{eq:raterate} is asymptotic in the calibration size $m$ and
	dimension-dependent through $d$. The feature vectors used in the experiments have $d$ of order
	ten, at which $m^{-1/(d+2)}$ is slow enough that the localization term is better regarded as a
	bias controlled empirically through the ESS floor than as a rate on which to rely. This is
	consistent with the ablation results, in which the localization component contributes the least
	of the four weighting mechanisms on the real panel.
\end{remark}

Theorem~\ref{thm:rate} renders the cost of estimation explicit and diagnosable: each weighting
component contributes an additive, separately estimable term, and only genuine conditional
drift $\Delta_t$ is irreducible at the calibration step. The online controller acts on that residual term, regulating its cumulative effect on realized coverage.

Before stating the controller guarantee we recall the notion of regret, used here in the sense
standard in online learning \citep{cesa2006prediction}. Suppose that at each date $t$ a
procedure selects an action from a set indexed by $k=1,\dots,K$ (here, a learning rate for the
significance-level update) and incurs a loss $\ell_{k,t}$. The \emph{regret} of the procedure
after $T$ periods is the excess cumulative loss relative to the best single action chosen with
hindsight,
\begin{equation}
	R_T \;=\; \sum_{t=1}^{T}\ell_{I_t,t}\;-\;\min_{1\le k\le K}\ \sum_{t=1}^{T}\ell_{k,t},
	\label{eq:regret}
\end{equation}
where $I_t$ is the action actually used at $t$. A procedure is called \emph{no-regret} when
$R_T=o(T)$, so that the per-period excess loss $R_T/T$ vanishes as $T$ grows and the procedure
asymptotically performs as well as if the best fixed action had been known in advance. This is
the relevant notion in the present context because no single learning rate is appropriate across
a heterogeneous panel of series, and it permits one configuration to be used throughout
without per-series tuning. The guarantee is \emph{relative} to the comparison class: it ensures
that the procedure matches the best member of that class, but not that the class contains a good
action. The absolute calibration statement is part~(a) of the following result.

\begin{theorem}[Coverage and adaptivity of the self-tuning controller]
	\label{thm:controller}
	Let the significance level be produced by the controller of Section~\ref{sec:method}.
	\emph{(a)} With a single learning rate the controller is adaptive conformal inference and,
	for every horizon $T$,
	$\big|\tfrac1T\sum_{t\le T}\ind\{Y_t\notin C_t\}-\alpha\big|\le (1+2\gamma)/(\gamma T)$
	deterministically. \emph{(b)} With the exponential-weights aggregation over $K$ rates, the weight update is a
	standard Hedge step on the per-expert pinball losses, so the regret \eqref{eq:regret} of the
	\emph{played level sequence} against the best fixed expert level satisfies
	$R_T=O(\sqrt{T\log K})$, hence $R_T/T\to0$.
\end{theorem}

\begin{remark}
	\label{rem:controller}
	
	Theorem \ref{thm:controller}  bounds cumulative pinball-loss regret relative to the finite grid of learning rates; it does not provide a coverage guarantee for the aggregated level. The deterministic telescoping argument in part (a) applies only to a single ACI recursion, not to a convex combination of several recursions.
	The controller of Section~\ref{sec:method} is moreover a heuristic introduced here rather than
	the algorithm of \citet{gibbs2024conformal}: all experts are scored against a single shared
	miscoverage indicator instead of receiving expert-specific feedback. The coverage results of that
	paper are therefore not imported, and part~(a) should be read as covering the single-rate
	configuration exactly, with the multi-rate default supported by part~(b) alone. 
	The default controller nevertheless achieves coverage close to 0.90 empirically (Table ~\ref{tab:rank}).

\end{remark}

\begin{proposition}[Deterministic lower bound on the effective sample size]
	\label{prop:ess}
	Let $\mathcal H=\{h_1<\dots<h_L\}$ be the bandwidth grid, let $\underline n$ be the
	\textrm{ESS} floor, and suppose the bandwidth is selected by the rule of
	Section~\ref{sec:method}: the smallest $h\in\mathcal H$ for which
	$\mathrm{ESS}(w(h))\ge\underline n$, and $h_L$ if no such $h$ exists. Let $m$ be the number of
	points in the calibration buffer and let $w^{\mathrm{glob}}$ denote the weights with the
	localization factor removed. Then at every step
	\[
	\begin{aligned}
		\mathrm{ESS}_t &\ \ge\ \min\big\{\underline n,\ \mathrm{ESS}(w(h_L))\big\}, \quad\text{and after the fallback of Section~\ref{sec:method},} \\
		\mathrm{ESS}_t &\ \ge\ \min\big\{\underline n,\ \mathrm{ESS}(w^{\mathrm{glob}})\big\}.
	\end{aligned}
	\]
	
	Moreover $\mathrm{ESS}(w(h))$ is nondecreasing in $h$ and
	$\mathrm{ESS}(w(h))\to\mathrm{ESS}(w^{\mathrm{glob}})$ as $h\to\infty$, so enlarging the grid can
	only raise the bound; and under pure temporal weighting with decay $\lambda$,
	$\mathrm{ESS}(w^{\mathrm{glob}})\to(1+\lambda)/(1-\lambda)$ as $m\to\infty$.
\end{proposition}

\begin{remark}
	\label{rem:ess}
	Since Theorem~\ref{thm:rate} is stated in effective rather than nominal sample size, it is
	necessary to establish that the effective size cannot collapse in practice and thereby render the
	bound vacuous. Proposition~\ref{prop:ess} provides this. The effective size cannot fall below the
	floor unless even the widest bandwidth fails to reach it, and in that event the fallback removes
	localization entirely, leaving weights whose effective size is bounded below by the temporal
	weighting alone. The mechanism is deterministic and requires no assumption on the data. Its
	effect is visible in the runs: the fallback engages on roughly half of the predictions for the
	shorter monthly series and rarely for the long hourly series, so the floor binds in practice.
	
	The second statement carries the caveat that matters for the asymptotics. A fixed decay
	$\lambda<1$ bounds $\mathrm{ESS}$ by a constant, so the floor guarantees a non-vacuous bound at
	every sample size but does not by itself deliver a vanishing coverage gap; the latter requires
	$\lambda_m\to1$, as stated in Theorem~\ref{thm:rate}.
\end{remark}

\begin{proposition}[Approximate regime-conditional coverage]
	\label{prop:regime}
	Suppose that within each latent regime the calibration and test scores are exchangeable, and
	that the estimated regime posteriors satisfy the $L^1$ error bound $\delta^{g}_t$ of
	Assumption~\ref{ass:reg}. Suppose in addition that the regime posteriors are
	\emph{$\kappa$-separated}, meaning $\E\|\bm\pi_i-\mathbf e_{R_i}\|_1\le\kappa$, where
	$\mathbf e_r$ is the $r$-th basis vector. Then for each regime $r$ with prior probability
	$p_r$,
	\[
	\Prob\{Y_t\in C_t\mid R_t=r\}\ \ge\ 1-\alpha_t-c\,(\delta^{g}_t+\kappa)/p_r
	\]
	for a universal constant $c$.
\end{proposition}

\begin{remark}[Effect of regime overlap]
	\label{rem:kappa}
	The term $\kappa$ captures overlap between latent regimes and remains nonzero even with perfectly estimated regime posteriors.
	The implemented weight is
	the posterior inner product $\bm\pi_i^\top\bm\pi_t$, which is a \emph{soft} similarity, whereas
	regime-conditional validity is a statement about the hard label $R_t$. The two coincide only when
	the posteriors concentrate on a single component. Even with perfectly estimated posteriors, that
	is with $\delta^g_t=0$, overlapping regimes leave $\bm\pi_i^\top\bm\pi_t$ bounded away from
	$\ind\{R_i=R_t\}$, and a residual conditional coverage gap of order $\kappa/p_r$ survives.
	Thus, the regime-conditional coverage gap depends on both posterior-estimation error $\delta_t^g$ and regime overlap $\kappa$. Consistent estimation of the posteriors alone is insufficient unless the regimes also become asymptotically separable. Only if regimes are observed, or
	identifiable with vanishing overlap, does $\kappa\to0$ and the statement reduce to the clean
	form. On the panel considered here the fitted posteriors are far from degenerate, so the
	proposition should be read as a degradation bound rather than as a validity guarantee.
\end{remark}

\begin{proposition}[No-regret aggregation]
	\label{prop:agg}
	Consider the \emph{averaging} form of \DRACP-Select, equation \eqref{eq:star}, in which the
	issued radius is the exponentially-weighted convex combination $r^\star_t=\sum_j\omega_{j,t}r^{(j)}_t$
	of the constituent radii. Assume the realized interval scores lie in $[0,B]$ for a known $B$.
	Choosing $\zeta=\sqrt{8\log J/T}/B$ gives, for every horizon $T$,
	$R_T=\sum_{t\le T}\mathrm{IS}_\alpha(C^\star_t)-\min_j\sum_{t\le T}\mathrm{IS}^{(j)}_\alpha
	\le B\sqrt{(T/2)\log J}$, so the average interval score exceeds that of the best constituent by
	at most $O(\sqrt{\log J/T})\to0$.
\end{proposition}

\begin{remark}
	\label{rem:hedge}
	Proposition~\ref{prop:agg} applies to the averaged radius, since the argument uses convexity of
	the interval score in the radius together with the standard Hedge bound. The variant evaluated in
	Section~\ref{sec:results}  is the \emph{hard selection} rule,
	which issues the radius of the currently best-performing constituent ($\arg\max_j\omega_{j,t}$)
	rather than a mixture. Hard selection is used because on short streams it is more robust to a
	poorly behaved constituent than averaging, which dilutes a good radius with a poor one. It does
	not, however, inherit the bound: with the weights entering only through an $\arg\max$ the Jensen
	step is unavailable, and the rule reduces to follow-the-leader, which has no sublinear worst-case
	regret guarantee in general.
	
	The bound is also conditional on a known score bound $B$, and realized interval scores are not
	bounded a priori: an interval that substantially misses an outlier incurs a score of order
	$2|Y_t-\widehat\mu(X_t)|/\alpha$. Results reported for the hedge are therefore to be read as
	empirical, and are not described as no-regret. Extending the guarantee to the evaluated procedure
	would require either the averaged form or a bounded, scale-free surrogate loss; neither is
	pursued here.
\end{remark}

The theoretical results motivate the empirical analysis that follows. In particular, the coverage performance of DRACP depends on the quality of the estimated weights and on residual conditional drift, while the multi-rate controller provides a regret guarantee rather than a coverage guarantee. The next section evaluates these effects empirically.

\section{Data and experimental design}
\label{sec:data}
\subsection{Real series and synthetic scenarios}
The evaluation covers 48 sequential forecasting problems in three groups. The first consists of
monthly HICP annual inflation for the 27 EU member states together with the euro-area aggregate
(Eurostat \texttt{prc\_hicp\_manr}), a homogeneous panel in which a common data-generating
mechanism is subject to country-specific structural breaks, most notably the 2021--23 inflation
surge. The second comprises US macroeconomic and energy series: the unemployment rate,
industrial production, CPI, PCE and producer-price inflation, retail sales, nonfarm payrolls,
capacity utilization, housing starts, M2 and consumer sentiment (FRED), together with hourly
electricity demand (UCI). The third consists of seven daily financial series, namely the 10-year
Treasury yield, the term spread, USD/EUR, WTI oil, the VIX, the S\&P~500 and the effective
federal funds rate. These last are close to a random walk with time-varying volatility and are
included as a negative control. Each series is modeled with autoregressive lags, rolling
summaries and seasonal encodings; growth series enter in log-differences and rates in first
differences. Data are split chronologically into training, calibration and test blocks.

\begin{table}[t]
	\centering
	\caption{The 48 real forecasting series, by group.}
	\label{tab:data}
	\footnotesize
	\begin{tabular}{llll}
		\toprule
		Group & \# series & Source & Content \\
		\midrule
		HICP inflation & 28 & Eurostat & monthly annual inflation, euro area and EU-27 \\
		US macro \& energy & 13 & FRED, UCI & activity, prices, money, sentiment, electricity \\
		Financial (daily) & 7 & FRED & rates, FX, commodities, equity, volatility \\
		\bottomrule
	\end{tabular}
\end{table}

The real panel is supplemented by seven synthetic data-generating processes, each isolating an
individual departure from exchangeability: stationarity, covariate shift, conditional drift,
abrupt drift, gradual drift, regime switching, and their combination. All share a common
nonlinear data-generating process and differ only in how the intercept, innovation scale,
covariate distribution and latent regime evolve; the exact specifications are given in
\ref{app:data}. The synthetic scenarios are intended as a diagnostic supplement to the real-data
evidence rather than as the principal evidence in their own right.

\subsection{Predictive model and nuisance estimators}
\label{sec:models}
Standard conformal validity is model-agnostic with respect to the point forecaster; the choice of point model primarily affects interval efficiency. All methods use the same fitted point forecaster, chronological splits, and random seeds. Differences in interval performance can therefore be attributed to the calibration procedures rather than to different underlying forecasts.

The point forecaster $\widehat\mu$ is a histogram-based gradient-boosting regressor
(scikit-learn's histogram gradient-boosting implementation). It is used because it accommodates
the nonlinearities and interactions of macroeconomic feature sets without manual specification
and is computationally light enough for the sequential protocol. It is trained once on the
training block with a learning rate of $0.05$ and $150$--$200$ boosting iterations depending on
series length (Table~\ref{tab:models}), and is thereafter held fixed: the sequential updating in
these experiments concerns calibration and not refitting of the point forecaster, which isolates
the contribution of the conformal layer.

\DRACP{} additionally requires three nuisance estimators, each of which is a standard model
fitted only on information available before the prediction being made.
\begin{enumerate}\setlength{\itemsep}{1pt}
	\item \emph{Conditional scale} $\widehat\sigma$. A random forest is fitted on the training
	block to predict the absolute residual $|Y-\widehat\mu(X)|$ of the point forecaster. Its
	prediction normalizes the conformity score, $S_i=|Y_i-\widehat\mu(X_i)|/\widehat\sigma(X_i)$,
	so that intervals widen where the model is intrinsically less accurate. This is the locally
	adaptive score of \citet{lei2018distribution,papadopoulos2002inductive}. A forest is used
	because the scale surface is typically less smooth than the conditional mean.
	\item \emph{Density ratio} $\widetilde r_t$. Estimated by domain classification: a logistic
	regression with standardized inputs is trained to separate recent (target-window) covariates
	from older (source) covariates, and its odds, corrected for the sampling prior, give the
	likelihood ratio in \eqref{eq:weights}, which is then clipped to $[r_{\min},r_{\max}]$. We use logistic 
	regression to limit the variance of the estimated density ratio, 
	which enters the composite weight multiplicatively, as reflected in the
	$\delta^r_t$ term of Theorem~\ref{thm:rate}.
	\item \emph{Latent regimes} $\bm\pi$. A Gaussian mixture with $K$ components and full
	covariances is fitted to standardized covariates on an expanding window of past observations
	only, so that regime posteriors are causal. The posterior $\bm\pi_i$ supplies the
	regime-similarity weights; soft posteriors are used rather than hard assignments in order to
	avoid discontinuous interval jumps at regime boundaries.
\end{enumerate}
Localization distances are computed on standardized covariates; Table~\ref{tab:notation}
collects the notation and Table~\ref{tab:models} summarizes the components and their
configuration.

\begin{table}[htbp]
	\centering
	\caption{Models used in the experiments. The point forecaster is shared by every method; the
		nuisance estimators are specific to \DRACP{}.}
	\label{tab:models}
	\footnotesize
	\begin{tabular}{llll}
		\toprule
		Role & Model & Target & Configuration \\
		\midrule
		Point forecaster $\widehat\mu$ & Hist.\ gradient boosting & $Y$ &
		learning rate $0.05$; $150$--$200$ iterations \\
		Conditional scale $\widehat\sigma$ & Random forest & $|Y-\widehat\mu(X)|$ &
		$150$ trees; min.\ leaf size $8$ \\
		Density ratio $\widetilde r_t$ & Logistic regression & recent vs.\ older &
		standardized inputs; clipped to $[r_{\min},r_{\max}]$ \\
		Regime posterior $\bm\pi$ & Gaussian mixture & covariates &
		$K$ components, full covariance, causal fit \\
		Feature scaling & Standardization & covariates & used for kernel distances \\
		\bottomrule
	\end{tabular}
\end{table}

Because the conformal layer is agnostic to $\widehat\mu$, any of these components may be
replaced, for instance by a linear or factor model for the point forecast, a quantile regression
for the scale, a gradient-boosted or neural classifier for the density ratio, or a hidden Markov
model for the regimes, without altering the guarantees of Section~\ref{sec:theory}; only the
corresponding error terms in Theorem~\ref{thm:rate} change. We use standard specifications throughout and do not tune the point forecaster separately for \DRACP{}.

\subsection{Baselines, metrics and protocol}
\DRACP{} is compared against six baselines: split, rolling and adaptive (ACI) conformal
prediction, together with the fully-adaptive (FACI), strongly-adaptive (SAOCP) and conformal-PID
online procedures. As noted in Section~\ref{sec:models}, every method is wrapped around the
same fitted point forecaster and receives the same training, calibration and test blocks, so
that the comparison isolates the calibration layer.

The primary metric is the mean interval (Winkler) score \citep{winkler1972decision}. For a
nominal $1-\alpha$ interval $[L,U]$ and realized outcome $Y$ it is defined as
\begin{equation}
	\mathrm{IS}_\alpha(L,U;Y)=(U-L)
	+\tfrac{2}{\alpha}(L-Y)\,\ind\{Y<L\}
	+\tfrac{2}{\alpha}(Y-U)\,\ind\{Y>U\}.
	\label{eq:winkler}
\end{equation}
It is a strictly proper scoring rule for central prediction intervals
\citep{gneiting2007strictly} that penalizes width and miscoverage jointly, and is therefore the
appropriate criterion when methods differ in both. Empirical coverage is also reported, in order
to verify that score gains are not obtained by undercovering, together with mean width, in order
to show where the gains originate.

Every experiment is repeated 20 times with different seeds, and the interpretation of this
repetition differs between the two settings. For the synthetic scenarios each seed redraws the
data-generating process, so that the spread across seeds is sampling variability. For the real
series the data are fixed and only the stochastic estimators are reseeded, so that the spread
measures estimation variability alone and would understate the true uncertainty; it is therefore
not described here as Monte-Carlo sampling. Uncertainty on the real series is instead quantified
by a paired circular moving-block bootstrap (2000 replicates, data-driven block length) over the
test stream and by per-series Diebold--Mariano tests \citep{diebold1995comparing}, and the
consistency of the method ranking across series is assessed with Friedman/Nemenyi and Wilcoxon
analyses. Per-series interval scores, coverage and widths for every method, together with the
complete Diebold--Mariano matrix, are reported in Tables~S1--S4 of the supplementary material.

\section{Results}
\label{sec:results}
Table~\ref{tab:rank} ranks the seven methods over the 48 real series. \DRACP{} does not attain
the best interval score. Strongly-adaptive online conformal prediction ranks first, with an
average rank of 2.19 and the lowest score on 32 of the 48 series; FACI follows at 2.73; \DRACP{}
is third at 3.10, best on six series, and is outperformed in pairwise comparison by the
strongly-adaptive procedure on 39 of the 48. Its intervals are also among the widest, with an
average width rank of 5.60 against 3.52 (Figures~\ref{fig:rank48} and~\ref{fig:rankhist}).
The per-series values behind these averages are given in Tables~S1--S3 of the supplementary
material.
The implementations of FACI, SAOCP, and conformal PID were validated against the authors’ reference implementations on a common score stream: the strongly-adaptive
procedure reproduces the reference radii exactly and to within 1.4\% for the fully-adaptive one in the conservative direction. 

Differences across the
panel are decisive as an omnibus matter (Friedman $p\approx8.6\times10^{-22}$), although that
$p$-value should not be read as though 48 independent forecasting tasks had been observed: the
28 HICP series share a monetary environment and the 2021--23 surge is a common shock, so the
effective number of independent tasks is well below the nominal count. With 48 series the
Nemenyi critical difference is $1.30$ ranks, so the strongly-adaptive procedure separates from
Adaptive, C-PID, Rolling and Split, but not from FACI or \DRACP{}.

The main advantage of DRACP is calibration. Its mean empirical coverage is $0.890$ against the nominal level of $0.90$, with a mean absolute coverage deviation of $0.026$ (see Table~\ref{tab:calib}).
The behavior in the tail is more informative than the mean: \DRACP{} and
FACI are the only methods that never undercover below $0.80$ on any series, with worst cases of
$0.807$ and $0.808$ respectively, whereas the strongly-adaptive procedure falls to $0.750$ and undercovers below $0.80$
on five series and below $0.85$ on twenty, against \DRACP{}'s ten. Its efficiency is therefore
obtained in part by operating systematically tight. The difference in calibration between
\DRACP{} and the strongly-adaptive procedure is significant (Wilcoxon $p=1.1\times10^{-4}$);
that between \DRACP{} and FACI is not, although the margin is narrow ($p=0.099$).

This is a trade-off, and the choice between the two operating
points is a substantive one. It also illustrates a general point about how conformal procedures
are ranked. The interval score is a proper scoring rule and does penalize miscoverage, so it is
not exposed to the failure that \citet{min2026questioning} identify for coverage-and-length
evaluation, in which an interval can be shortened deceptively while marginal coverage is
preserved. The interval score nevertheless aggregates, and an aggregate over 48 series does not
reveal that one method attains its mean by undercovering on twenty of them. Reporting the
distribution of coverage together with the score, as Table~\ref{tab:calib} does, renders visible the difference between the two leading methods (see also Figure~\ref{fig:euimp}). A
forecaster evaluated on sharpness would prefer the strongly-adaptive procedure. A forecasting
unit that publishes intervals carrying a stated coverage standard, and that is judged ex post on
whether the standard held, would prefer the operating point that does not undercover by five
points on one series in ten. Section~\ref{sec:disc} returns to this question.

\begin{table}[t]
	\centering
	\caption{Overall comparison across the 48 real series: average interval-score rank
		(1$=$best), number of series where best, mean empirical coverage against a nominal $0.90$,
		mean absolute coverage deviation, and average width rank. Ordered by score rank.}
	\label{tab:rank}
	\begin{tabular}{lccccc}
		\toprule
		Method & Score rank & \# best & Mean cov. & Mean $|\text{cov}-0.90|$ & Width rank \\
		\midrule
		SAOCP & 2.19 & 32 & 0.868 & 0.046 & 3.52 \\
		FACI & 2.73 & 5 & 0.883 & 0.031 & 3.92 \\
		DRACP & 3.10 & 6 & 0.890 & 0.026 & 5.60 \\
		Adaptive & 4.38 & 0 & 0.884 & 0.031 & 5.52 \\
		C-PID & 4.52 & 3 & 0.825 & 0.083 & 3.94 \\
		Rolling & 5.52 & 0 & 0.841 & 0.078 & 3.44 \\
		Split & 5.56 & 2 & 0.781 & 0.140 & 2.06 \\
		\bottomrule
	\end{tabular}
\end{table}

\begin{table}[t]
	\centering
	\caption{Calibration reliability across the 48 real series. ``Worst $|$dev$|$'' is the largest
		absolute deviation from nominal on any series and ``Min cov.'' the lowest empirical coverage
		attained; the last two columns count series undercovering below $0.85$ and below $0.80$.
		Ordered by mean deviation.}
	\label{tab:calib}
	\begin{tabular}{lcccccc}
		\toprule
		Method & Mean cov. & Mean $|$dev$|$ & Worst $|$dev$|$ & Min cov. & $<0.85$ & $<0.80$ \\
		\midrule
		DRACP & 0.890 & 0.026 & 0.093 & 0.807 & 10 & 0 \\
		FACI & 0.883 & 0.031 & 0.092 & 0.808 & 12 & 0 \\
		Adaptive & 0.884 & 0.031 & 0.105 & 0.795 & 12 & 1 \\
		SAOCP & 0.868 & 0.046 & 0.150 & 0.750 & 20 & 5 \\
		Rolling & 0.841 & 0.078 & 0.245 & 0.655 & 24 & 16 \\
		C-PID & 0.825 & 0.083 & 0.245 & 0.655 & 25 & 20 \\
		Split & 0.781 & 0.140 & 0.376 & 0.524 & 30 & 22 \\
		\bottomrule
	\end{tabular}
\end{table}

\begin{figure}[t]
	\centering
	\includegraphics[width=0.62\textwidth]{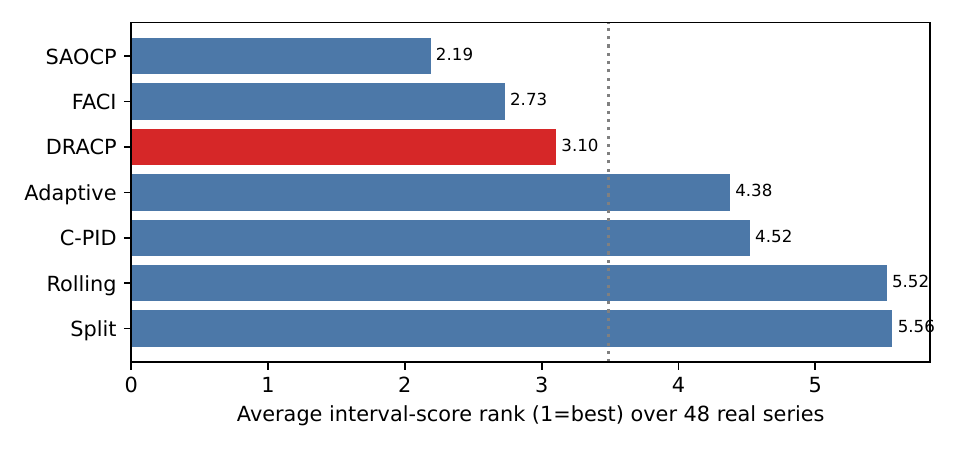}
	\caption{Average interval-score rank of each method over the 48 real series
		(lower is better; the dashed line marks the Nemenyi critical distance from the best method).}
	\label{fig:rank48}
\end{figure}

\begin{figure}[t]
	\centering
	\includegraphics[width=0.52\textwidth]{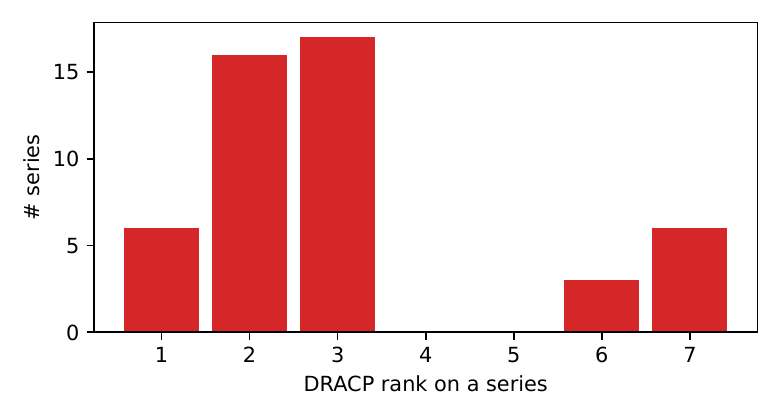}
	\caption{Distribution of \DRACP{}'s interval-score rank across the 48 real series; it is best
		on five and in the top three on roughly half.}
	\label{fig:rankhist}
\end{figure}

Table~\ref{tab:groups} decomposes the ranking by group, and the decomposition is unfavorable to
\DRACP{} in a specific and informative way (see also Figure~\ref{fig:rankgroups}). On the
28-series HICP inflation panel, the setting for which the method was initially designed, the
strongly-adaptive procedure ranks 1.14 and \DRACP{} 2.57, behind FACI at 2.32.
On US macro and energy the three are close (3.15, 3.62, 3.69), and on the daily financial
series FACI leads at 2.57 with \DRACP{} at 4.29, consistent with the negative-control role for
which that group was selected.

The pattern that survives the decomposition is again in the coverage figures.
\DRACP{}'s empirical coverage is $0.878$ on the inflation panel, $0.908$ on US macro and $0.905$
on financial series, so that it remains close to nominal in all three groups, which no
competitor achieves. 

\begin{table}[t]
	\centering
	\caption{Average interval-score rank by group (1$=$best), with \DRACP{} empirical coverage in
		the last column. Ordered as in Table~\ref{tab:rank}.}
	\label{tab:groups}
	\footnotesize\setlength{\tabcolsep}{4pt}
	\begin{tabular}{lccccccccc}
		\toprule
		Group & $n$ & SAOCP & FACI & DRACP & Adaptive & C-PID & Rolling & Split & DRACP cov. \\
		\midrule
		EU HICP inflation & 28 & 1.14 & 2.32 & 2.57 & 4.18 & 5.11 & 6.18 & 6.50 & 0.877 \\
		US macro \& energy & 13 & 3.15 & 3.69 & 3.62 & 4.46 & 3.92 & 4.85 & 4.31 & 0.909 \\
		Financial (daily) & 7 & 4.57 & 2.57 & 4.29 & 5.00 & 3.29 & 4.14 & 4.14 & 0.904 \\
		\bottomrule
	\end{tabular}
\end{table}

\begin{figure}[t]
	\centering
	\includegraphics[width=0.60\textwidth]{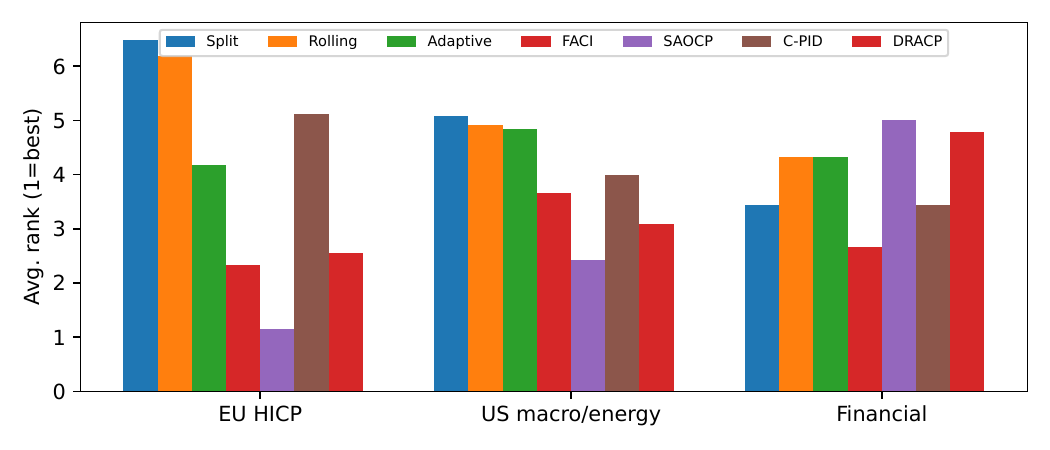}
	\caption{Average interval-score rank by method within each of the three series groups.}
	\label{fig:rankgroups}
\end{figure}

\begin{figure}[t]
	\centering
	\includegraphics[width=0.58\textwidth]{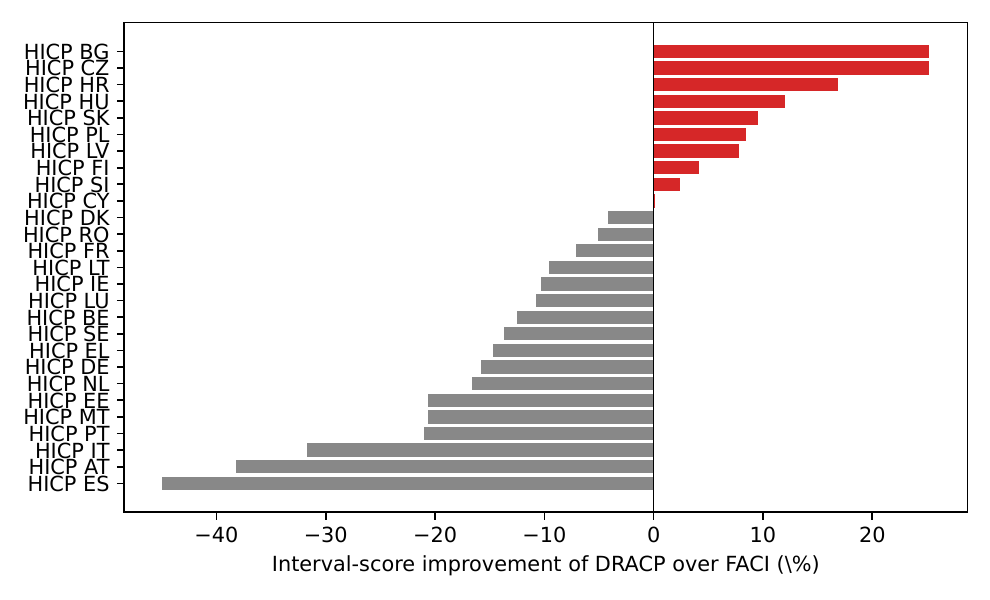}
	\caption{Distribution of empirical coverage across the 48 real series, by method, ordered by
		mean absolute deviation from nominal. \DRACP{} and FACI are the only methods whose lower whisker
		stays above $0.80$; the strongly-adaptive procedure, which attains the best interval score,
		undercovers on twenty series.}
	\label{fig:euimp}
\end{figure}

\paragraph{A case study: the 2021--23 inflation surge}
\label{sec:case}
The aggregate rankings conceal the underlying mechanism, so two EU inflation series are examined
directly. Figure~\ref{fig:case} plots \DRACP{} and FACI intervals for Germany and Romania across
the test period, which spans the 2021--23 surge (shaded). Both methods are calibrated before the
surge, but they differ in how they react to it. \DRACP{} widens its intervals as soon as the
density-ratio and regime components indicate that recent observations no longer resemble the
earlier calibration episodes, and narrows them again once the new regime is populated with
comparable data. FACI, which adapts only the significance level from realized coverage,
reacts with a lag and loses coverage during the transition: \DRACP{} widens sooner and remains
closer to nominal through the surge, which is the behavior the design targets. As
Table~\ref{tab:groups} indicates, that advantage lies in calibration rather than in interval
score. Figure~\ref{fig:casewidth} separates the two channels: interval width tracks the break
while the self-tuned $\alpha_t$ absorbs the residual miscalibration, corresponding to the
division of roles described by Theorems~\ref{thm:rate} and~\ref{thm:controller}.

\begin{figure}[htbp]\centering
	\includegraphics[width=\textwidth]{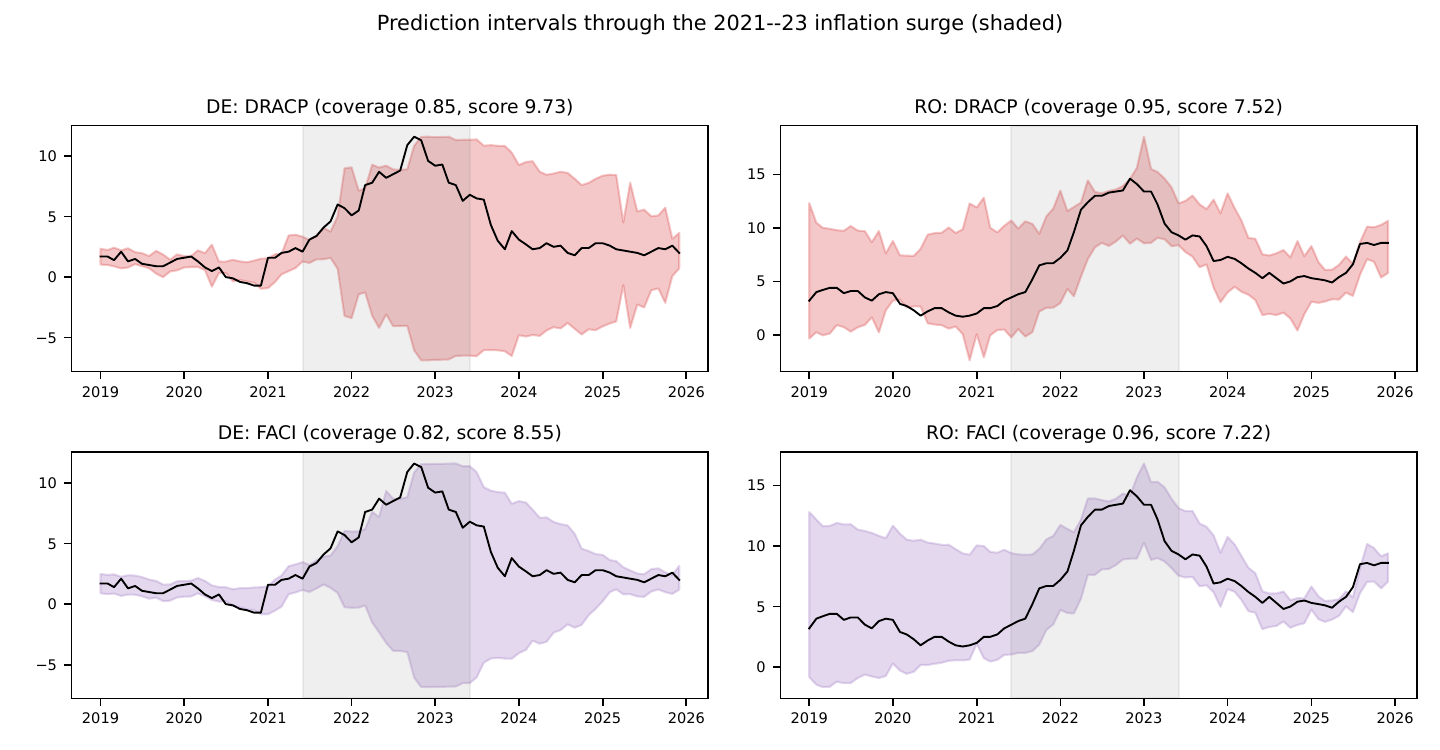}
	\caption{Prediction intervals for German and Romanian HICP inflation through the 2021--23
		surge (shaded). \DRACP{} (top) widens and re-narrows around the break; FACI (bottom) adapts
		more slowly. Coverage and score annotations are computed from the single reproducibility run
		that produced these interval paths, and therefore differ slightly from the 20-seed means
		reported in Section~\ref{sec:results} and in the supplementary material.}
	\label{fig:case}
\end{figure}

\begin{figure}[htbp]\centering
	\includegraphics[width=\textwidth]{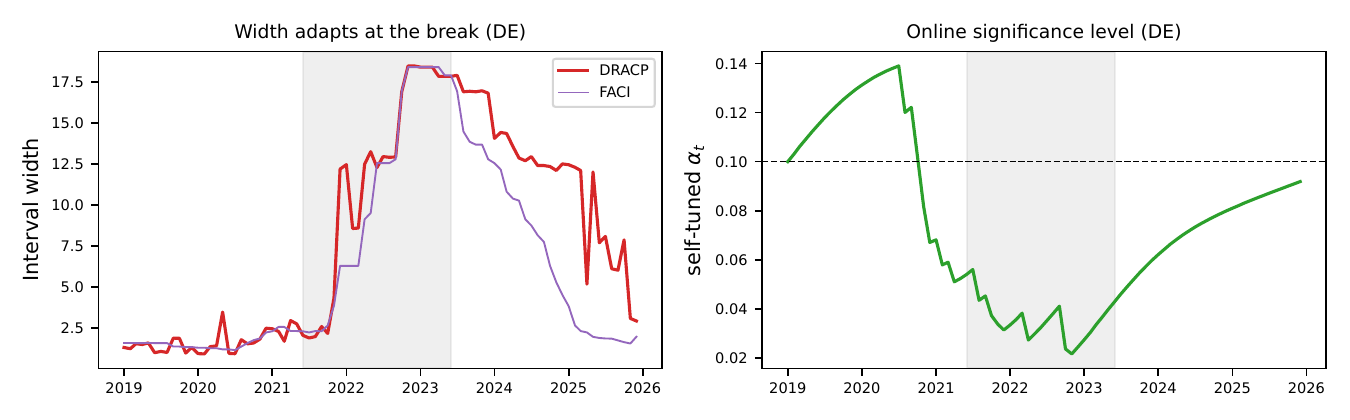}
	\caption{Left: interval width for Germany, with \DRACP{} expanding at the break and contracting
		afterwards. Right: the self-tuned significance level $\alpha_t$ absorbing residual
		miscoverage.}
	\label{fig:casewidth}
\end{figure}

\paragraph{Per-series significance}
\label{sec:dm}
Because each real stream is short, the rank analysis is complemented by per-series
Diebold--Mariano tests \citep{diebold1995comparing} of the interval-score differential
(Newey--West long-run variance \citep{newey1987simple} with a data-driven bandwidth
$\lfloor T^{1/3}\rfloor$, and the Harvey--Leybourne--Newbold \citep{harvey1997testing}
small-sample correction evaluated at the forecast horizon $h=1$; the HAC bandwidth and the
horizon are distinct arguments and are supplied separately). The Newey--West bandwidth is set by
the data-driven rule $\lfloor T^{1/3}\rfloor$ applied to each series' test length. The tests are
computed on the loss sequence of a \emph{single fit} rather than on a loss sequence averaged
across the 20 repeated fits; averaging first would smooth away the period-by-period
variation that the long-run variance is intended to capture, and would bias the tests toward
significance. This yields $6\times48=288$ tests.

Because 288 tests are conducted simultaneously, unadjusted counts overstate the evidence: under
a global null of equality one would expect roughly $14$ rejections at the $5\%$ level by chance
alone. Both the unadjusted tallies and tallies after controlling the false discovery rate across
all 288 tests at $5\%$ by the Benjamini--Hochberg procedure
\citep{benjamini1995controlling} are therefore reported, to screen for a pattern, not to make a single confirmatory claim. Sharing the \DRACP{} loss stream does not by itself establish
the positive-regression-dependence condition under which Benjamini--Hochberg controls the false
discovery rate exactly; Holm-adjusted counts, valid under arbitrary dependence, are therefore
also reported, and no claim is based on the former that the latter does not support.

The per-series evidence matches the ranking. Unadjusted, \DRACP{} is significantly better in 47
comparisons and significantly worse in 47, an even split. After Benjamini--Hochberg adjustment
(threshold $p\le0.0093$) it is better in 32 and worse in 26; under Holm, 17 and 9. Table~\ref{tab:dm} reports the underlying statistics for a representative
subset of series; Table~S4 of the supplementary material gives the complete matrix of 288 tests. Disaggregated
by baseline (Table~\ref{tab:dmcount}) the picture is sharper than the totals suggest. \DRACP{}
is significantly better than Adaptive CP on 12 series and worse on 2, and better than Rolling,
Split and conformal PID by comparable margins, but against the two strongest baselines the sign
reverses: worse than FACI on 9 series and better on 4, worse than the strongly-adaptive
procedure on 9 and better on 2. The interval-score comparison therefore favors the modern online
conformal methods over \DRACP{}, and favors \DRACP{} over the classical baselines.

\begin{table}[t]
	\centering
	\caption{Diebold--Mariano outcomes for \DRACP{} versus each baseline over the 48 real series
		(synthetic scenarios excluded). ``Unadjusted'' counts use $p<0.05$ per test; ``BH-adjusted''
		counts control the false discovery rate at $5\%$ across all 288 tests jointly.}
	\label{tab:dmcount}
	\begin{tabular}{lccccc}
		\toprule
		& & \multicolumn{2}{c}{Unadjusted} & \multicolumn{2}{c}{BH-adjusted} \\
		\cmidrule(lr){3-4}\cmidrule(lr){5-6}
		Baseline & Tests & better & worse & better & worse \\
		\midrule
		SAOCP & 48 & 2 & 15 & 2 & 9 \\
		FACI & 48 & 4 & 14 & 4 & 9 \\
		Adaptive & 48 & 19 & 5 & 12 & 2 \\
		C-PID & 48 & 10 & 5 & 4 & 2 \\
		Rolling & 48 & 6 & 5 & 5 & 2 \\
		Split & 48 & 6 & 3 & 5 & 2 \\
		\midrule
		All & 288 & 47 & 47 & 32 & 26 \\
		\bottomrule
	\end{tabular}
\end{table}

\begin{table}[t]
	\centering
	\caption{Diebold--Mariano statistics on representative structured series (positive favors
		\DRACP{}). $^{*}p<0.05$, $^{**}p<0.01$, $^{***}p<0.001$ (unadjusted); $^{\dagger}$ also
		significant after Benjamini--Hochberg control of the false discovery rate at $5\%$ across all
		288 tests jointly.}
	\label{tab:dm}
	\footnotesize\setlength{\tabcolsep}{4pt}
	\begin{tabular}{lcccccc}
		\toprule
		Series & Split & Rolling & Adaptive & FACI & SAOCP & Conf-PID \\
		\midrule
		Electricity & 6.6$^{***\dagger}$ & 5.8$^{***\dagger}$ & 6.3$^{***\dagger}$ & 5.8$^{***\dagger}$ & 11.5$^{***\dagger}$ & 6.9$^{***\dagger}$ \\
		Euro area HICP & 1.7 & 1.6 & 3.0$^{**\dagger}$ & $-2.9^{**\dagger}$ & $-4.0^{***\dagger}$ & 1.9 \\
		US retail sales & 0.9 & 0.9 & 1.9 & 1.7 & $-0.0$ & 0.8 \\
		US ind. production & 1.4 & 1.6 & 1.7 & 1.6 & $-0.5$ & 1.5 \\
		HICP DE & 1.9 & 1.8 & 3.0$^{**\dagger}$ & $-1.4$ & $-1.3$ & 2.0$^{*}$ \\
		HICP FR & 1.9 & 1.7 & 2.1$^{*}$ & $-0.4$ & $-2.0$ & 2.0$^{*}$ \\
		HICP IT & 1.6 & 1.6 & 2.3$^{*}$ & $-3.0^{**\dagger}$ & $-3.0^{**\dagger}$ & 1.9 \\
		HICP ES & 0.7 & 0.8 & 1.2 & $-5.9^{***\dagger}$ & $-4.0^{***\dagger}$ & 0.7 \\
		\bottomrule
	\end{tabular}
\end{table}

\paragraph{Synthetic diagnostics}
The controlled scenarios isolate shift mechanisms one at a time and are reported here as a
diagnostic. On this benchmark the composite weighting is not the binding constraint.
Table~\ref{tab:synmain} reports coverage: split conformal prediction collapses under conditional
and abrupt drift, falling to $0.424$ and $0.523$ against a nominal $0.90$, as anticipated by
Theorem~\ref{thm:rate}, while every adaptive method and \DRACP{} tracks the nominal level
throughout. \DRACP{}'s mean absolute coverage error of $0.005$ is close to the best available,
and it never undercovers in any scenario.

Table~\ref{tab:synscore} reports interval scores. \DRACP{} attains the lowest interval score in
\emph{none} of the seven scenarios, and its average score rank of $6.57$ out of seven is the
worst of any method. This includes the combined-shift scenario, in which all mechanisms operate
at once and composite weighting should be most valuable: FACI scores $10.09$ against \DRACP{}'s
$10.62$. The synthetic evidence therefore does \emph{not} support the claim that composition
pays under simultaneous shifts, and no such claim is made on its basis.

The synthetic scenarios do, however, demonstrate coverage robustness obtained at a cost in
width. A plausible explanation for the divergence from the real panel is that the generated
shifts are smooth, of known parametric form, and of a magnitude that a scalar level controller
can track; the weighting machinery then contributes estimation variance without exploitable
structure, as it does on the daily financial series. The case for \DRACP{} therefore rests on
the real panel reported above, where the shift structure is unknown and irregular. The gap
between the two is the most important open question that this paper leaves: a simulation design
reproducing the real panel's coverage behavior would be a useful contribution, and is not
attempted here. Figure~\ref{fig:covdriftmain} summarizes coverage as shift complexity increases.

\begin{table}[htbp]
	\centering
	\caption{Synthetic empirical coverage (20-seed means; nominal $0.90$; closest in bold).}
	\label{tab:synmain}
	\begin{tabular}{lcccccccc}
		\toprule
		Scenario & Split & Rolling & Adaptive & FACI & SAOCP & C-PID & DRACP \\
		\midrule
		Stationary & 0.898 & 0.902 & \textbf{0.900} & 0.898 & 0.902 & 0.900 & 0.902 \\
		Covariate & 0.899 & 0.902 & \textbf{0.900} & 0.898 & 0.903 & 0.901 & 0.903 \\
		Conditional & 0.424 & 0.840 & \textbf{0.900} & 0.894 & 0.841 & 0.899 & 0.904 \\
		Abrupt & 0.523 & 0.888 & 0.900 & 0.897 & 0.889 & \textbf{0.900} & 0.904 \\
		Gradual & 0.797 & 0.911 & 0.909 & \textbf{0.907} & 0.911 & 0.911 & 0.912 \\
		Regime & 0.902 & 0.903 & 0.901 & \textbf{0.899} & 0.903 & 0.902 & 0.904 \\
		Combined & 0.866 & 0.894 & \textbf{0.901} & 0.897 & 0.894 & 0.898 & 0.903 \\
		\bottomrule
	\end{tabular}
\end{table}

\begin{table}[htbp]
	\centering
	\caption{Synthetic mean interval score (20-seed means; lowest in bold).}
	\label{tab:synscore}
	\begin{tabular}{lcccccccc}
		\toprule
		Scenario & Split & Rolling & Adaptive & FACI & SAOCP & C-PID & DRACP \\
		\midrule
		Stationary & \textbf{4.83} & 4.84 & 4.87 & 4.88 & 4.84 & 4.84 & 5.09 \\
		Covariate & \textbf{4.83} & 4.83 & 4.86 & 4.88 & 4.83 & 4.84 & 5.04 \\
		Conditional & 17.29 & 9.73 & \textbf{9.56} & 9.59 & 9.74 & 9.57 & 10.21 \\
		Abrupt & 22.10 & \textbf{14.71} & 14.72 & 14.75 & 14.71 & 14.74 & 15.77 \\
		Gradual & 12.57 & 11.80 & \textbf{11.69} & 11.71 & 11.81 & 11.71 & 12.43 \\
		Regime & 8.89 & 8.87 & 8.87 & \textbf{8.82} & 8.88 & 8.87 & 8.89 \\
		Combined & 10.33 & 10.22 & 10.19 & \textbf{10.09} & 10.24 & 10.20 & 10.62 \\
		\bottomrule
	\end{tabular}
\end{table}

\begin{figure}[htbp]\centering
	\includegraphics[width=.58\textwidth]{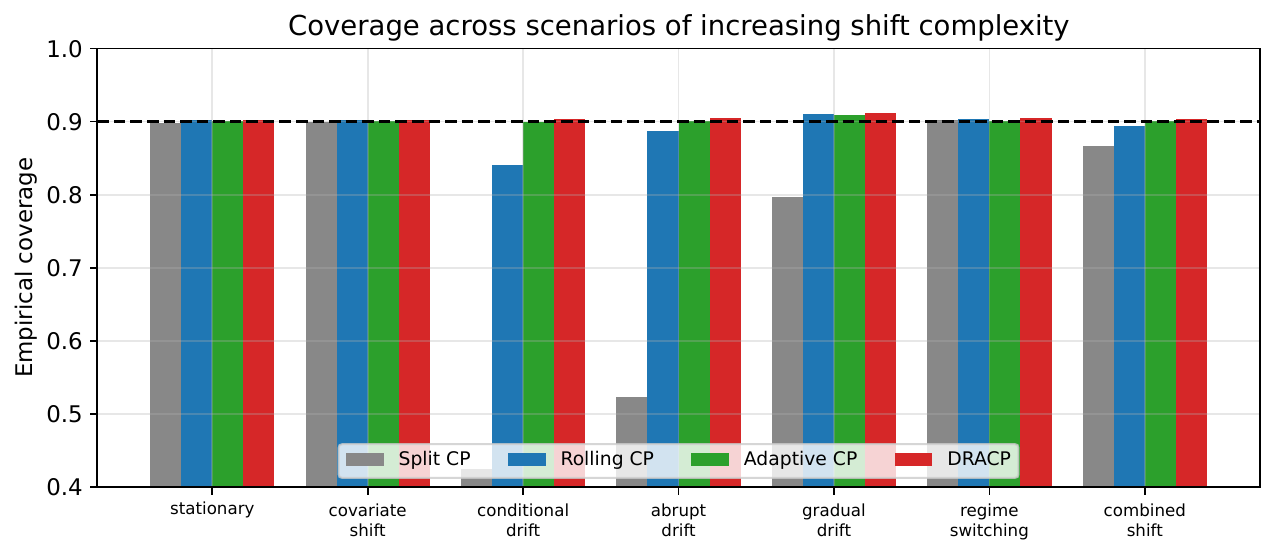}
	\caption{Empirical coverage across synthetic scenarios of increasing shift complexity.}
	\label{fig:covdriftmain}
\end{figure}

\paragraph{Robustness to the point forecaster}
\label{sec:models_res}
Conformal validity does not depend on the point forecaster, but efficiency does, and a
comparison built on a single predictor risks mistaking a property of that predictor for a
property of the calibration scheme. The entire comparison is therefore repeated with five base
forecasters spanning a wide range of accuracy: a linear autoregression (ridge), a multilayer
perceptron, a random forest, gradient boosting, and a naive mean predictor.

The outcome is uniform. Measured against the best-performing baseline on each series, \DRACP{}'s
interval score is worse under \emph{every} one of the five forecasters, by $6.5\%$ under ridge,
$13.3\%$ under the naive predictor, $18.8\%$ under gradient boosting, $27.2\%$ under the random
forest and $36.5\%$ under the multilayer perceptron. There is no forecaster for which composite
weighting delivers the best interval score, and the strongly-adaptive procedure ranks first
under all five.

The calibration pattern, by contrast, is stable across forecasters and is the one robust
regularity in this experiment. \DRACP{}'s empirical coverage remains within $0.889$--$0.906$ of
a nominal $0.90$ whichever forecaster is used, while the strongly-adaptive procedure runs
between $0.863$ and $0.891$, undercovering under every one. Whatever the base model leaves in
the residuals, the composite weighting together with the controller returns intervals close to
nominal; this property is independent of the forecaster.

\begin{table}[t]
	\centering
	\caption{Comparison repeated with five base forecasters (48 series). ``RMSE'' is the median
		point-forecast RMSE relative to the best model on the same series ($1.00=$ most accurate).
		Ranks are average interval-score ranks among the seven conformal methods. ``$\Delta$ vs best''
		is the mean per-series percentage difference in interval score between \DRACP{} and the best
		baseline on that series (positive is worse for \DRACP{}).}
	\label{tab:models_res}
	\footnotesize\setlength{\tabcolsep}{4pt}
	\begin{tabular}{lccccccc}
		\toprule
		Base forecaster & RMSE & DRACP rk & SAOCP rk & FACI rk & $\Delta$ vs best (\%) & DRACP cov. & SAOCP cov. \\
		\midrule
		ridge ar & 1.00 & 3.52 & 2.69 & 3.58 & +6.5 & 0.901 & 0.891 \\
		mlp & 1.26 & 4.73 & 2.08 & 2.40 & +36.5 & 0.906 & 0.883 \\
		random forest & 1.62 & 3.31 & 1.98 & 2.75 & +27.2 & 0.889 & 0.866 \\
		gradient boosting & 2.03 & 3.10 & 2.21 & 2.69 & +18.8 & 0.890 & 0.863 \\
		naive & 3.67 & 2.75 & 2.19 & 3.14 & +13.3 & 0.892 & 0.872 \\
		\bottomrule
	\end{tabular}
\end{table}

\paragraph{Multi-horizon evaluation}
All results reported so far are one step ahead. Table~\ref{tab:horizons} extends the evaluation
to $h=3,6,12$. Forecasts are \emph{direct} rather than iterated: for each horizon the target is
shifted so that the observation aligned with origin $t$ is $Y_{t+h}$, the feature vector is left
unchanged and therefore contains only information available at $t$, and a separate model is
fitted for each horizon. Every horizon receives its own chronological train/calibration/test
split of the shifted frame, its own point forecaster, its own nuisance estimators and its own
calibration scores; nothing is reused across horizons. Therefore, the calibration scores
entering the weighted quantile at horizon $h$ are $h$-step-ahead scores throughout, and no
one-step quantity is used to calibrate a multi-step interval.

Feedback is delayed to match the horizon. An online conformal method consumes each outcome as it
arrives, and at horizon $h$ the outcome attached to origin $t$ is not realized until $h-1$
periods after the outcome attached to origin $t-1$. Releasing it to the calibration set
immediately, which is correct at $h=1$, would allow the forecast issued at $t+1$ to use an
observation that a real-time forecaster could not yet have seen. Each outcome is therefore
queued and released into the calibration set only at the date at which it becomes observable, so
that at every origin the method conditions on the information available then. All
multi-step results below use this delayed-feedback protocol; at $h=1$ the queue is empty and the
evaluation is identical to the one-step case.

Forecast errors \emph{overlap} for $h>1$: consecutive targets $Y_{t+h}$ and $Y_{t+1+h}$ share
$h-1$ periods of innovations, which induces serial correlation of moving-average type and order
at least $h-1$ in the loss sequence. This is handled by not conducting Diebold--Mariano tests at
$h>1$. The per-series tests reported in Section~\ref{sec:dm} are one-step throughout, where no
overlap arises, and the multi-horizon evidence in Table~\ref{tab:horizons} is reported as
average ranks, which do not require a long-run variance estimate. Testing at $h>1$ would require
a Newey--West bandwidth of at least $h-1$ rather than the $\lfloor T^{1/3}\rfloor$ rule used at
$h=1$, and reporting ranks is preferred to relying on a long-run variance estimated from short
overlapping streams.

Extending the evaluation to $h=3,6,12$ separates the two criteria more sharply than the
one-step comparison does, and the calibration argument is strongest there.

On interval score the ordering shifts with horizon and no adaptive method holds up. The
strongly-adaptive procedure leads at $h=1$ and $h=3$ (2.19, 2.46) but degrades to 3.42 and 4.50
by $h=6$ and $h=12$; \DRACP{} degrades from 3.10 to 4.92; and by $h=12$ the best methods are the
two \emph{static} ones, split conformal at 2.89 and rolling at 3.05, with every adaptive
procedure behind them. Online adaptation of any kind presumes that feedback arrives quickly
enough to be informative. At twelve steps, with feedback delayed by eleven periods, that
condition fails, and pooling a long history outperforms adaptation based on information received
too late.

On coverage the ordering does not shift, and the margin widens. \DRACP{} holds the highest
empirical coverage at \emph{every} horizon, at $0.890$, $0.887$, $0.873$ and $0.842$, while the
strongly-adaptive procedure falls from $0.868$ to $0.786$ and FACI from $0.883$ to $0.822$
(Figure~\ref{fig:horizonrank}). The gap between \DRACP{} and the strongly-adaptive procedure
grows monotonically from 2.2 coverage points at $h=1$ to 5.6 at $h=12$. All methods lose coverage
as the horizon grows, as expected; they differ in the rate at which they do so, and the
procedure that was sharpest at one step degrades fastest. For a forecaster publishing
twelve-month-ahead intervals at a stated level, this ordering is more consequential than the
interval-score ordering.

At long horizons
\DRACP{} occasionally cannot certify the requested level from its weighted calibration sample
and returns an unbounded interval, on up to $17\%$ of predictions on one series. Unbounded
intervals have infinite interval score and are excluded from the averages, so the reported
scores omit the cases that \DRACP{} found hardest, and its score degradation with
horizon is if anything understated. Such intervals are negligible at $h=1$, and no other method
produces them.

\begin{table}[t]
	\centering
	\caption{Average interval-score rank by forecast horizon over the seven methods (48 series, 10
		seeds); ranks sum to 28 in every row and are comparable across rows. The last three columns are
		empirical coverage against a nominal $0.90$ for the three leading methods.}
	\label{tab:horizons}
	\footnotesize\setlength{\tabcolsep}{4pt}
	\begin{tabular}{lccccccc|ccc}
		\toprule
		& \multicolumn{7}{c|}{Interval-score rank} & \multicolumn{3}{c}{Coverage} \\
		Horizon & SAOCP & FACI & DRACP & Adapt. & C-PID & Roll. & Split & DRACP & SAOCP & FACI \\
		\midrule
		$h=1$ & 2.19 & 2.73 & 3.10 & 4.38 & 4.52 & 5.52 & 5.56 & 0.890 & 0.868 & 0.883 \\
		$h=3$ & 2.46 & 2.73 & 4.06 & 4.25 & 4.29 & 4.88 & 5.33 & 0.887 & 0.839 & 0.870 \\
		$h=6$ & 3.42 & 2.92 & 4.52 & 4.52 & 4.44 & 4.11 & 4.07 & 0.873 & 0.804 & 0.844 \\
		$h=12$ & 4.50 & 3.33 & 4.92 & 4.56 & 4.75 & 3.05 & 2.89 & 0.842 & 0.786 & 0.822 \\
		\bottomrule
	\end{tabular}
\end{table}

\begin{figure}[t]
	\centering
	\includegraphics[width=0.60\textwidth]{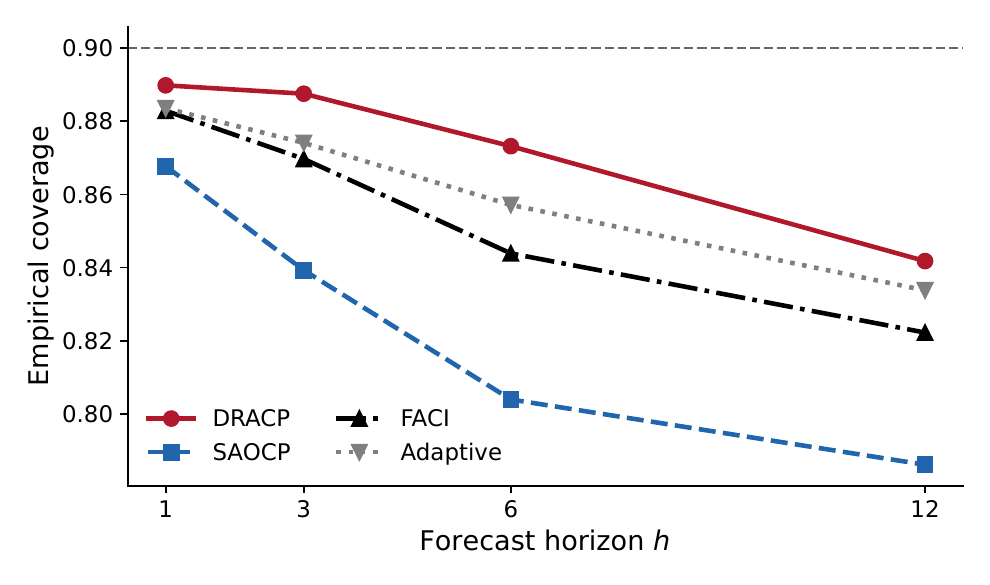}
	\caption{Empirical coverage against forecast horizon over the 48 real series. All methods lose
		coverage as the horizon grows; \DRACP{} loses it most slowly, and its margin over the
		strongly-adaptive procedure widens from 2.2 coverage points at $h=1$ to 5.6 at $h=12$.}
	\label{fig:horizonrank}
\end{figure}

Table~\ref{tab:runtime_res} reports measured cost per prediction on the combined-shift scenario,
on a macOS machine (10~cores, Python 3.11.5). \DRACP{} costs 75~ms per prediction, approximately
6$\times$ split conformal prediction and some three times the online baselines, with a peak
memory footprint under two megabytes. At monthly or hourly frequencies this overhead is
negligible; at high frequency, where the results above show no benefit, it is not worth
incurring.

\begin{table}[t]
	\centering
	\caption{Measured computational cost per prediction (900 sequential predictions,
		macOS-26.5.2-arm64-arm-64bit, 10~cores).}
	\label{tab:runtime_res}
	\begin{tabular}{lccc}
		\toprule
		Method & ms / prediction & Relative & Peak memory (MB) \\
		\midrule
		Split & 12.5 & 1.0 & 1.25 \\
		Rolling & 22.2 & 1.8 & 1.11 \\
		Adaptive & 21.9 & 1.8 & 1.12 \\
		FACI & 21.8 & 1.7 & 1.12 \\
		SAOCP & 25.2 & 2.0 & 1.12 \\
		C-PID & 25.3 & 2.0 & 1.12 \\
		DRACP & 74.7 & 6.0 & 1.93 \\
		\bottomrule
	\end{tabular}
\end{table}

The regime count $K$ is the one structural hyperparameter that a user must select, and its
influence is therefore examined directly. Table~\ref{tab:regimes} varies $K$ from one to six on
the combined-shift benchmark, holding everything else fixed. Coverage is flat at
$0.898$--$0.899$ against a nominal $0.90$ across the whole range, mean width varies by under
$1\%$, and the interval score by under $1.3\%$. Even $K=1$, which disables the regime mechanism
entirely, since a single component renders the posterior degenerate and the similarity weight
constant, differs from $K=3$ by $0.7\%$ in score. This is consistent with the ablation results
reported below, in which regime weighting contributes $0.7\%$: a component that contributes
little is unlikely to be sensitive to its own tuning.

The quantity that does move monotonically is the effective calibration size, from $81.9$ at
$K=1$ to $68.0$ at $K=6$, because a finer partition concentrates the regime weights on fewer
comparable observations. This is the cost of raising $K$, and it is the quantity that
Proposition~\ref{prop:ess} bounds. We use $K=3$ throughout as a compromise that leaves the
effective size comfortably above the floor while permitting an expansion, a contraction and a
transition state to be distinguished; the results would not change materially at $K=2$ or $K=4$.

\begin{table}[t]
	\centering
	\caption{Sensitivity to the number of latent regimes $K$ on the combined-shift benchmark
		(nominal coverage $0.90$). ``ESS'' is the mean effective calibration sample size.}
	\label{tab:regimes}
	\begin{tabular}{lccccc}
		\toprule
		Regimes & Coverage & $|\text{cov}-0.90|$ & Mean width & Interval score & ESS \\
		\midrule
		$K=1$ & 0.898 & 0.002 & 8.58 & 10.84 & 81.9 \\
		$K=2$ & 0.898 & 0.002 & 8.58 & 10.79 & 75.5 \\
		$K=3$ & 0.898 & 0.002 & 8.55 & 10.77 & 72.1 \\
		$K=4$ & 0.899 & 0.001 & 8.51 & 10.80 & 72.7 \\
		$K=5$ & 0.898 & 0.002 & 8.60 & 10.71 & 68.0 \\
		$K=6$ & 0.899 & 0.001 & 8.54 & 10.77 & 68.0 \\
		\bottomrule
	\end{tabular}
\end{table}

Table~\ref{tab:abl} removes each component in turn \emph{on the real panel}, and the result
should not be read in isolation but alongside the comparison above.

Two mechanisms dominate, and they are of comparable magnitude. Removing the online controller
costs $11.8\%$ in interval score, on 92\% of series; replacing the fitted conditional scale with
a constant costs $11.9\%$, on 81\%. The four weighting mechanisms together account for
approximately $6.2\%$: temporal decay $4.6\%$, the density ratio $0.7\%$, regime weighting
$0.7\%$ and localization $0.2\%$. The last three are individually small enough that no claim of
separate identification is made.

This bears directly on attribution. Neither of the two dominant ingredients is the conceptual
contribution of the paper: the controller is adaptive conformal inference, which the baselines
also employ, and conditional-scale normalization is standard practice that could be added to any
of them. The composite weighting, which is the novel element, is the smallest of the three. Read
together with the comparison above, the appropriate characterization is that \DRACP{} combines
heteroskedasticity-aware conformal scores, an online significance controller and regime-aware
weighting, in that order of empirical importance, and that its distinctive contribution is the
calibration reliability that this combination delivers rather than sharpness.

The coverage column supports a complementary reading. Removing the controller degrades mean
absolute coverage deviation from $0.026$ to $0.051$, and removing temporal decay degrades it to
$0.037$, whereas removing the conditional scale leaves it at $0.030$. The controller is
therefore the source of the calibration, the conditional scale the source of the efficiency, and
the weighting contributes modestly to both.

\begin{table}[t]
	\centering
	\caption{Component ablation on the \emph{real} panel (48 series, 20 seeds). ``$\Delta$ score''
		is the mean per-series percentage change in interval score relative to full \DRACP{} (positive
		is worse); ``worse on'' is the share of series on which removal degrades the score. The last row
		replaces the fitted conditional scale with a constant, isolating the contribution of scale
		normalization from that of the weights and the controller.}
	\label{tab:abl}
	\begin{tabular}{lccc}
		\toprule
		Variant & $\Delta$ score (\%) & Worse on & Mean $|\text{cov}-0.90|$ \\
		\midrule
		Full DRACP & +0.0 & 0\% & 0.026 \\
		-- density ratio & +0.7 & 50\% & 0.027 \\
		-- regime weighting & +0.7 & 50\% & 0.026 \\
		-- localization & +0.2 & 62\% & 0.025 \\
		-- temporal decay & +4.6 & 71\% & 0.036 \\
		-- adaptive control & +11.8 & 92\% & 0.051 \\
		-- conditional scale & +11.9 & 81\% & 0.030 \\
		\bottomrule
	\end{tabular}
\end{table}

\section{Discussion}
\label{sec:disc}

\subsection{Where \DRACP{} works and where it does not}
The main empirical advantage of \DRACP{} is coverage stability. It holds mean empirical coverage of $0.890$ against a nominal
$0.90$, the closest of the seven methods. It is one of only two that never undercover below
$0.80$ on any of the 48 series. It retains the highest coverage at all four forecast horizons,
with its margin over the strongly-adaptive procedure widening from 2.2 to 5.6 coverage points
between $h=1$ and $h=12$. It holds coverage within $0.889$--$0.906$ under all five base
forecasters. It attains the highest coverage through the 2021--23 inflation surge, the episode
for which the method was designed. The same coverage pattern is observed across series, forecast horizons, base forecasters, the inflation episode, and the regime-specific analysis.

On interval score \DRACP{} ranks third. Against reference-validated implementations of the
recent online procedures, strongly-adaptive online conformal prediction ranks 2.19 and \DRACP{}
3.10, winning in pairwise comparison on 39 of the 48 series with intervals approximately a fifth
narrower. The composite weighting does not deliver sharpness on this panel, and on the daily
financial series, which serve as a negative control, it adds estimation variance with no shift
structure to exploit. Beyond three steps ahead the ordering on score shifts again, and by twelve
steps static calibration outperforms every adaptive method, including the one proposed here.
This echoes \citet{sabashvili2026conformal}, who report a multi-step split conformal method as
the most efficient procedure meeting nominal coverage on a monthly sales panel. That two
evaluations on unrelated data reach the same conclusion about static calibration at longer
horizons suggests the result is a property of the problem rather than of either panel.

The ablation accounts for both findings. Conditional-scale normalization and the online
controller contribute approximately $12\%$ of interval score each, while the four weighting
mechanisms contribute approximately $6\%$ combined. The two dominant components are not,
however, interchangeable in what they deliver: removing the controller degrades mean absolute
coverage deviation from $0.027$ to $0.051$, whereas removing the conditional scale leaves it at
$0.030$. Calibration is therefore attributable primarily to the controller, efficiency primarily
to the conditional scale, and the composite weighting contributes modestly to both.

he ablation results help explain \DRACP{}’s empirical performance. The strongly-adaptive
procedure optimizes a regret criterion on the score and reaches a sharper operating point,
undercovering below $0.85$ on twenty series against \DRACP{}'s ten, and below $0.80$ on five
where \DRACP{} never does. \DRACP{} combines a controller that targets coverage directly with
weights that keep the calibration sample relevant, and arrives at an operating point that is
wider but reliably calibrated. \DRACP{} achieves coverage closer to the nominal level but produces wider intervals than SAOCP.

\subsection{Implications for practice and deployment}
\label{sec:deploy}
The empirical results also indicate where \DRACP{} is most useful. \DRACP{} is the appropriate choice
when interval \emph{coverage} is the operative criterion, that is, when a published interval
carries a stated level and is audited ex post, and when the series carries exploitable shift
structure. Five conditions mark the boundary of that domain; outside it a simpler procedure is
preferable.

First, the series should carry exploitable shift structure. When the target plausibly contains
recurring regimes and structural breaks, as with inflation, activity and energy demand,
regime-aware weighted calibration delivers the coverage closest to nominal, although not the
sharpest intervals. When the target is instead close to a martingale difference, as with asset
returns, exchange rates, or any series whose conditional distribution is near-constant, the
weighting components add estimation variance without adding information, and a purely online
controller such as FACI is preferable: on the financial group FACI attains the best mean
interval score and the best rank, while \DRACP{} ranks fifth of seven. This argues for dropping
the weighting rather than for reverting to static split calibration, which performs worst of all
on that group.

Second, the residual stream should retain systematic structure. Where the point forecaster
already leaves near-unstructured residuals, the calibration deficit that \DRACP{} repairs is
largely absent and a simpler calibration layer suffices.

Third, the horizon should be short. Weights computed at the forecast origin describe a
conditional environment that has had time to change by the date at which a target three or more
steps ahead is realized. Horizon-specific weighting recovers part of the loss but not all of it,
and by twelve steps static calibration outperforms every adaptive method.

Fourth, the calibration sample should be reasonably long. Below roughly one to two hundred
observations the density-ratio classifier and the regime mixture are estimated from too little
data to be informative, and the coverage-gap bound of Theorem~\ref{thm:rate} is uninformative at
the resulting nuisance error rates. This threshold has not been characterized systematically.

Fifth, the departure from exchangeability should be plural and unknown. Where a single known
departure is present, such as pure covariate shift with a known design, the corresponding
specialized method is simpler, better understood and at least as accurate. \DRACP{}'s
composition is a response to \emph{simultaneous} and unidentified shifts.

The comparison with the modern online baselines is informative for practice. FACI, SAOCP and
conformal PID adapt the significance level but leave calibration weights untouched, which is why
they track long-run coverage well yet respond to a break only once miscoverage has accumulated.
\DRACP{}'s weights act on the \emph{composition} of the calibration set and therefore react at
the break itself, with the controller handling what the weights cannot. This is the division of
roles formalized by Theorems~\ref{thm:rate} and~\ref{thm:controller}, and it explains why the two families complement one another.

The coverage-gap decomposition also serves as a diagnostic. Monitoring the estimated density
ratio, the effective sample size and the regime posterior indicates which mechanism is doing the
work on a given series, and whether the method is operating inside its assumptions.

Two families of methods are deliberately outside the comparison, for the following reason. The
jackknife+ \citep{barber2021predictive} and ensemble batch prediction intervals
\citep{xu2021conformal} construct intervals by leave-one-out or bootstrap ensembling rather than
by reweighting a calibration split, so they address the estimation of the score distribution
not its adaptation to shift, and they require refitting the base forecaster many times
per origin, a cost that the streaming setting considered here does not accommodate. They 
complement  \DRACP{} rather than competing with it: an ensemble score estimator could in
principle be combined with the weighting studied here, and that combination is the most
promising direction to which the present results point.

\DRACP{} is therefore most relevant when maintaining coverage under distributional change is the primary objective. For stable series or longer forecast horizons, the simpler procedures considered here generally provide narrower intervals or lower interval scores.
A practitioner unable to determine ex ante which situation obtains has two defensible options.
The first is to run the selection hedge of Section~\ref{sec:method}, which is a practical
fallback though not a guaranteed one, since Proposition~\ref{prop:agg} covers the averaged form
whereas the variant evaluated here selects the leader outright (Remark~\ref{rem:hedge}). The
second is to monitor the diagnostics: if the estimated density ratio remains near one and the
regime posterior near its prior, the weighting components are dormant and the additional
complexity is not justified.

\subsection{Operational use at central banks and statistical institutes}
The setting to which these results speak most directly is operational forecasting inside central
banks, national statistical institutes and international organizations. Evaluating and
communicating forecast uncertainty is an established concern in this literature: the Bank of
England's density forecasts have repeatedly been assessed for calibration
\citep{clements2004evaluating, galbraith2012evaluating}, the forecast performance of the
European Central Bank and the Federal Reserve Bank of New York through the financial crisis has
been compared directly \citep{alessi2014central}, and the aggregation of forecaster
disagreement into published density forecasts remains actively debated
\citep{knuppel2018forecast}. These institutions publish interval or density forecasts on a fixed
calendar, including inflation and activity projections in quarterly monetary-policy reports,
flash and revised estimates of harmonized indices, and short-term energy and activity
indicators. They do so under two constraints that shape what is usable. Their published
intervals are held to an explicit coverage standard and are scrutinized ex post, so a procedure
whose coverage degrades silently through a regime change is a reputational liability as well as
a statistical one. Their point forecasts are typically produced by established structural or
judgmental models that are not open to renegotiation, so any calibration layer must attach to an
existing forecast rather than replace it.

\DRACP{} fits this profile in three respects. It is a post-hoc calibration layer, consuming the
residuals of whatever forecaster is already in production and returning intervals, so it can be
added to an existing pipeline without altering the point-forecast model or its governance. Its
accuracy advantage is largest in the conditions under which these institutions
operate, namely monthly and quarterly macroeconomic series with recurring regimes, evaluated one
step ahead, and, as the multi-forecaster experiment indicates, it is largest when the point model
is misspecified relative to the data, which is the normal condition for structural macroeconomic
models during a shock such as that of 2021--23. Its computational cost, roughly six times a
split-conformal prediction, is irrelevant at monthly or quarterly frequency, where a full
re-estimation takes seconds.

The diagnostic by-products may matter as much as the intervals themselves. The estimated density
ratio, the effective calibration sample size and the regime posterior are each interpretable and
each observable in real time. An institution can monitor them as a running indicator of whether
current conditions resemble the historical calibration sample, which amounts to an automated
warning that the published intervals are being computed from observations the present no longer
resembles. A forecasting unit can act on such information by widening intervals, by flagging
elevated uncertainty in the accompanying commentary, or by triggering a model review, and the
information is available before realized coverage has deteriorated sufficiently to appear in
ex-post evaluation. The self-tuning controller offers a complementary governance property:
because it controls average miscoverage deterministically, an institution can commit to a
published nominal level and demonstrate ex post that the procedure honored it on average,
independently of whether the modeling assumptions held.

Two practical caveats apply in this setting specifically. Official statistics are revised, and
the evaluation reported here uses final vintages throughout; a real-time implementation should
calibrate on the vintages actually available at each forecast origin, and real-time performance
would be expected to be somewhat weaker. Where an institution's published forecasts are
judgmentally adjusted, the residual stream is no longer the output of a fixed statistical model,
which complicates the nuisance estimation on which the method relies. Neither is an obstacle in
principle, but both should be resolved before operational use.

\subsection{Limitations}
Several limitations qualify these conclusions. The method is outperformed by simple baselines on
near-efficient-market series, and documentation of this boundary forms part of the contribution
rather than a caveat to be minimized. Its components are individually dormant under isolated
shifts, so a practitioner facing a single known departure should use the corresponding
specialized method. \DRACP{} is also the most expensive procedure considered, at roughly six
times a split-conformal prediction, which is immaterial at monthly or hourly frequencies but not
at high frequency.

Two further limitations bound the contribution. \DRACP{}'s advantage depends on the point
forecaster: it is largest when the forecaster is misspecified (Table~\ref{tab:models_res}) and
disappears under the best-performing linear autoregression, so the method is best understood as
recovering calibration quality lost when the conditional mean leaves exploitable structure in
the residuals. The relationship with point accuracy is not monotone across the five forecasters,
so residual structure rather than RMSE is the operative quantity. The advantage is also specific
to one-step-ahead forecasting: at horizons $h\ge3$ a purely online controller dominates
(Table~\ref{tab:horizons}), because weights computed at the forecast origin lose relevance as
the horizon grows. The obvious remedy, retargeting the weights to the forecast date via a regime
transition matrix, was tested and did not succeed, degrading performance on the
persistent inflation regimes for which the origin state is already an informative forecast of
the target state. The selection hedge recovers most of the multi-step gap, converging in these
runs on the online controller once the weights cease to contribute.

On inference, cross-series rank tests place \DRACP{} and FACI within the Nemenyi critical
difference over the full panel. The separation is clear within the structured subgroups and in
the per-series Diebold--Mariano tests, and it is presented in those terms, without any claim of uniform dominance. Three further caveats apply. The EU inflation panel, although comprising
27 series, does not constitute 27 independent experiments: the countries share a common monetary
environment and the 2021--23 surge is a common shock, so the effective number of independent
structural episodes is smaller than the nominal count, and the panel-level $p$-values should be
read accordingly. The headline comparison is one-step-ahead, and the multi-step evidence shows
that the advantage does not carry to $h\ge3$ under origin-based weighting. Finally, all headline
tables use a default base forecaster (gradient boosting). Although conformal validity is
model-agnostic by construction, the \emph{efficiency} comparison depends on the point forecaster,
as Section~\ref{sec:models_res} documents across five of them.

A further limitation concerns attribution. \DRACP{} normalizes residuals by a fitted conditional
scale, whereas the classical baselines use raw absolute residuals, so part of the measured
efficiency gain may derive from that normalization and not from the four weights and the
controller. Table~\ref{tab:abl} reports a conditional-scale ablation that isolates this
contribution; until scale-normalized versions of every baseline are also evaluated, the gains
should be attributed to the combination of scale normalization and weighting rather than to the
weighting alone.

\section{Conclusion}
\label{sec:concl}
We introduced \DRACP{}, a conformal calibration method for sequential forecasting under multiple forms of distribution shift. The method combines temporal, density-ratio, local, and regime-similarity weighting with online adaptation of the miscoverage level. The theoretical analysis separates the coverage error associated with estimated weights, localization, regime uncertainty, and residual conditional drift.

Across 48 macroeconomic, energy, and financial series, \DRACP{} produced empirical coverage closer to the nominal 0.90 level than the six comparison methods, including across different forecast horizons, base forecasters, and the 2021–2023 inflation episode. This improvement did not translate into the best interval score: SAOCP generally produced sharper intervals, whereas \DRACP{} produced wider intervals with more stable coverage. The ablation analysis indicates that the online controller and conditional-scale normalization account for most of the improvement, with smaller contributions from the individual weighting components.

These results suggest that \DRACP{} is primarily useful when coverage stability under distributional change is more important than minimizing interval width. Its advantages are smaller for relatively stable series and at longer forecast horizons. Further work should examine dependent-data coverage guarantees, alternative regime models, and applications in which the point forecaster is updated jointly with the conformal calibration.

\appendix
\section{Proofs}
\label{app:proofs}
This appendix proves the results of Section~\ref{sec:theory} in the order in which they are
stated. Proposition~\ref{prop:coverage} establishes finite-sample validity of the weighted
quantile under oracle weights and serves as the base case for the remaining results.
Theorem~\ref{thm:rate} is the principal technical result: it bounds the coverage gap by the
supremum distance between the estimated weighted score distribution and the true test
distribution, decomposes that distance into one term per weighting component together with a
sampling fluctuation governed by the effective sample size, and then substitutes
nuisance-estimation rates to obtain the stated convergence.
Theorem~\ref{thm:controller} treats the online controller in two parts: a self-correction
argument bounding the adapted level, which telescopes into the average-miscoverage bound, and
a standard exponential-weights regret bound for the aggregation over learning rates.
Proposition~\ref{prop:regime} conditions on the latent regime and quantifies the loss arising
from posterior estimation error, and Proposition~\ref{prop:agg} applies
prediction-with-expert-advice regret to the selection hedge. Throughout,
Proposition~\ref{prop:coverage} is a property of the weighted-conformal construction under
oracle importance weights and is not inherited by the implemented product weights; see
Remark~\ref{rem:oracle}.

\begin{proof}[Proof of Proposition~\ref{prop:coverage}]
	Let $S_1,\dots,S_n$ be the calibration scores and $S_{n+1}$ the (unobserved) test score,
	with oracle weights $w_i=(\mathrm dP^{\mathrm{test}}/\mathrm dP^{\mathrm{cal}})(X_i)$ and the
	test point carrying the \emph{same} functional form of weight,
	$w_{n+1}=(\mathrm dP^{\mathrm{test}}/\mathrm dP^{\mathrm{cal}})(X_t)$, placed at $+\infty$. A
	unit test weight is correct only after rescaling all weights so that the test weight equals one;
	the unrescaled form is retained here to match \eqref{eq:quantile}. Under covariate shift with
	known likelihood ratio, $(S_1,\dots,S_{n+1})$ is \emph{weighted exchangeable}: for any
	permutation the joint density is invariant after reweighting by $\prod_i w_i$
	\citep[Lemma~1]{tibshirani2019conformal}. Hence the value $S_{n+1}$ is distributed as a draw
	from the tilted empirical law $\sum_{i=1}^{n+1}\bar w_i\,\delta_{S_i}$, $\bar w_i=w_i/\sum_j
	w_j$. The weighted quantile
	$q_t=\inf\{s:\sum_i\bar w_i\ind\{S_i\le s\}+\bar w_{n+1}\ge 1-\alpha_t\}$ therefore satisfies
	$\Prob(S_{n+1}\le q_t)\ge 1-\alpha_t$, and $\{S_{n+1}\le q_t\}=\{Y_t\in C_t(X_t)\}$; almost-sure
	distinctness of the scores rules out ties. \end{proof}

\begin{proof}[Proof of Theorem~\ref{thm:rate}]
	Write $G^{\mathrm{test}}_t(s)=\Prob(S_{n+1}\le s\mid X_t)$ for the true conditional test-score
	CDF and $\widehat G_t(s)=\sum_i \bar w_{i,t}\ind\{S_i\le s\}$ for the estimated weighted CDF that
	\DRACP{} inverts. By construction $\widehat G_t(q_t)=1-\alpha_t$, while coverage equals
	$\Prob(S_{n+1}\le q_t)=G^{\mathrm{test}}_t(q_t)$. Therefore
	\[
	\big|\Prob\{Y_t\in C_t\}-(1-\alpha_t)\big|
	=\big|G^{\mathrm{test}}_t(q_t)-\widehat G_t(q_t)\big|
	\le \big\|\widehat G_t-G^{\mathrm{test}}_t\big\|_\infty .
	\]
	Introduce the oracle-weighted CDF $\bar G_t$ (same weights but exact density ratio and regime
	indicator), its population version $G^{\mathrm{loc}}_t$ (kernel-localized calibration law), and
	the calibration and test conditional laws $G^{\mathrm{cal}}_t,G^{\mathrm{test}}_t$. By the triangle
	inequality,
	\[
	\|\widehat G_t-G^{\mathrm{test}}_t\|_\infty\le
	\underbrace{\|\widehat G_t-\bar G_t\|_\infty}_{\text{(weights)}}
	+\underbrace{\|\bar G_t-\E\bar G_t\|_\infty}_{\text{(sampling)}}
	+\underbrace{\|G^{\mathrm{loc}}_t-G^{\mathrm{cal}}_t\|_\infty}_{\text{(localization)}}
	+\underbrace{\text{(regime)}}_{}
	+\underbrace{\|G^{\mathrm{cal}}_t-G^{\mathrm{test}}_t\|_\infty}_{\text{(drift)}} .
	\]
	For the weight term, for every $s$,
	$|\widehat G_t(s)-\bar G_t(s)|=|\sum_i(\bar w_{i,t}-\bar w^\star_{i,t})\ind\{S_i\le s\}|
	\le\sum_i|\bar w_{i,t}-\bar w^\star_{i,t}|$, and the normalized weights differ only through the
	clipped density ratio, whose $L^1$ error is $\delta^r_t$ plus the truncated mass $\delta^c_t$;
	this term is therefore $\le\delta^r_t+\delta^c_t$. The regime term is bounded by the probability
	that the estimated posterior assigns the wrong regime, $\delta^g_t$
	(Assumption~\ref{ass:reg}(iii)). The localization term is a bias: since $x\mapsto F_{S\mid X=x}$
	is $L$-Lipschitz in total variation (Assumption~\ref{ass:reg}(iv)) and the Gaussian kernel has
	first moment $O(h_t)$, it is $\le L\,h_t$. The drift term is $\Delta_t$
	(Assumption~\ref{ass:reg}(v)). Finally, the sampling term is a weighted
	Dvoretzky--Kiefer--Wolfowitz deviation: the weighted empirical CDF over calibration points with
	effective sample size $\mathrm{ESS}_t$ satisfies $\|\bar G_t-\E\bar G_t\|_\infty\le
	\sqrt{\log(2/\eta)/(2\,\mathrm{ESS}_t)}$ with probability at least $1-\eta$.
	
	A final term must be accounted for. The weighted empirical CDF $\bar G_t$ is a step function, so
	the defining inequality of \eqref{eq:quantile} gives only $\bar G_t(q_t^-)<1-\alpha_t\le \bar
	G_t(q_t)$ rather than $\bar G_t(q_t)=1-\alpha_t$: the attained level overshoots the target by at
	most the size of the jump at $q_t$, which is at most $\max_i\bar w_{i,t}$. Carrying this through
	contributes the additive $\max_i\bar w_{i,t}$ in \eqref{eq:rate}. A randomized quantile would
	remove the term, at the cost of randomized intervals, which are undesirable in a forecasting
	setting where reproducibility of the published interval matters. Summing the terms gives
	\eqref{eq:rate}. The concentration step assumes that the calibration scores are independent
	conditional on the information generating the weights; under serial dependence it must be
	replaced by a mixing-based inequality, as discussed in Remark~\ref{rem:dependence}, which
	changes the constants and the exponent but not the additive structure. For the rates,
	substitute $\delta^r_t=O(\mathrm{ESS}_t^{-\beta_r})$,
	$\delta^g_t=O(\mathrm{ESS}_t^{-\beta_g})$ and $\delta^c_t\to0$, where the nuisance errors are the
	\emph{weighted} errors of Assumption~\ref{ass:reg}, evaluated under the same weights that define
	$\mathrm{ESS}_t$. Localization with a $d$-dimensional kernel of bandwidth $h_t$ retains an
	effective sample of order $\mathrm{ESS}_t h_t^d$, so the fluctuation term is
	$O((\mathrm{ESS}_t h_t^d)^{-1/2})$ while the localization bias is $O(Lh_t)$; balancing gives
	$h_t\asymp \mathrm{ESS}_t^{-1/(d+2)}$. The dominant term is therefore
	$O(\mathrm{ESS}_t^{-\min(\beta_r,\beta_g,1/(d+2))})$ plus $\max_i\bar w_{i,t}$ and the
	irreducible $\Delta_t$, which is \eqref{eq:raterate}. All statements are conditional on the
	$\sigma$-algebra generated by the data preceding $t$, so the coverage controlled here is
	marginal over $(X_t,Y_t)$ given that past, and not conditional on $X_t$. Finally, as noted in
	the theorem, a fixed decay $\lambda<1$ leaves $\mathrm{ESS}_t$ bounded and none of these rates
	vanish; the asymptotic statement presumes $\lambda_m\to1$ with
	$\mathrm{ESS}_t\to\infty$. \end{proof}

\begin{proof}[Proof of Theorem~\ref{thm:controller}]
	(a) The adaptive-conformal update is $\alpha_{t+1}=\alpha_t+\gamma(\alpha-\ind\{Y_t\notin C_t\})$.
	If $\alpha_t\le0$ the radius is $+\infty$, so $Y_t\in C_t$, $\ind\{Y_t\notin C_t\}=0$ and
	$\alpha_{t+1}=\alpha_t+\gamma\alpha>\alpha_t$; symmetrically, if $\alpha_t\ge1$ then
	$\ind\{Y_t\notin C_t\}=1$ and $\alpha_{t+1}<\alpha_t$. Hence $\alpha_t\in[-\gamma,1+\gamma]$ for
	all $t$. Summing the update telescopes,
	$\gamma\sum_{t\le T}(\alpha-\ind\{Y_t\notin C_t\})=\alpha_{T+1}-\alpha_1$, so
	$\big|\tfrac1T\sum_{t\le T}\ind\{Y_t\notin C_t\}-\alpha\big|
	=\frac{|\alpha_{T+1}-\alpha_1|}{\gamma T}\le\frac{1+2\gamma}{\gamma T}$.
	(b) The aggregation over $K$ experts with weights updated by
	$p_{k,t+1}\propto p_{k,t}e^{-\zeta \ell_{k,t}}$ is the exponentially-weighted-average (Hedge)
	forecaster on the pinball losses $\ell_{k,t}\in[0,1]$; its regret against the best fixed expert
	is at most $\sqrt{(T/2)\log K}$ for $\zeta=\sqrt{8\log K/T}$
	\citep[Theorem~2.2]{cesa2006prediction}. Dividing by $T$ gives the stated regret bound.
	
	The argument delivers no more than is stated. It does not follow that the aggregate level
	inherits the deterministic coverage property of part~(a): that property arises from telescoping
	a single ACI recursion, and a convex combination of $K$ such recursions is not itself an ACI
	recursion in the aggregate level. Nor is it true in general that coverage of every expert
	implies coverage of the aggregate, since the aggregate level is a weighted average of the expert
	levels and the miscoverage indicator is not linear in the level. Part~(b) is therefore confined
	to pinball regret, as stated. \end{proof}

\begin{proof}[Proof of Proposition~\ref{prop:ess}]
	Write $w(h)_i=b_i\,K_h(\|X_i-X_t\|)$ with $b_i$ the product of the temporal, density-ratio and
	regime factors, which do not depend on $h$. The selection rule returns some $h^\star\in\mathcal
	H$. If $h^\star<h_L$ the rule terminated early, which by construction occurs only when
	$\mathrm{ESS}(w(h^\star))\ge\underline n$. If $h^\star=h_L$ the rule exhausted the grid and
	$\mathrm{ESS}_t=\mathrm{ESS}(w(h_L))$. In either case
	$\mathrm{ESS}_t\ge\min\{\underline n,\mathrm{ESS}(w(h_L))\}$.
	
	For monotonicity, note that $\mathrm{ESS}(w)=\|w\|_1^2/\|w\|_2^2$ is a Schur-concave functional
	of the normalized weight vector: it decreases as the weights become more concentrated in the
	majorization order. Increasing $h$ raises every Gaussian kernel value toward one, and does so
	proportionally more for distant points, so $w(h)/\|w(h)\|_1$ is majorized by
	$w(h')/\|w(h')\|_1$ for $h>h'$, and $\mathrm{ESS}(w(h))$ is nondecreasing. As $h\to\infty$ the
	kernel tends to the constant one and $w(h)\to w^{\mathrm{glob}}$ pointwise, giving the stated
	limit.
	
	The fallback replaces $w$ by $w^{\mathrm{glob}}$ whenever the support test fails, so the second
	bound follows from the same argument applied to a weight vector with no localization factor.
	Finally, with only temporal weighting $w_i=\lambda^{m-i}$, direct summation of the geometric
	series gives $\|w\|_1^2/\|w\|_2^2=(1-\lambda^{2m})^{-1}(1-\lambda^m)^2(1+\lambda)/(1-\lambda)
	\to(1+\lambda)/(1-\lambda)$. \end{proof}

\begin{proof}[Proof of Proposition~\ref{prop:regime}]
	Condition on $R_t=r$. Had the regime been observed, restricting calibration to regime-$r$ points
	would make the scores exchangeable within the regime, so the regime-$r$ conditional coverage
	would be at least $1-\alpha_t$ by Proposition~\ref{prop:coverage}. The implemented weighting
	uses the posterior inner product $\bm\pi_i^\top\bm\pi_t$, which differs from the oracle
	indicator $\ind\{R_i=r\}$ for two separate reasons, both of which must be accounted for.
	
	First, the posteriors are estimated: replacing $\widehat{\bm\pi}$ by the oracle posterior
	$\bm\pi$ perturbs each weight by at most $\|\widehat{\bm\pi}_i-\bm\pi_i\|_1+
	\|\widehat{\bm\pi}_t-\bm\pi_t\|_1$ in absolute value, whose expectation is $\delta^g_t$.
	Second, and independently, even the \emph{oracle} posterior inner product is not the indicator
	unless the posteriors are degenerate: writing $\bm\pi_i=\mathbf e_{R_i}+\bm\varepsilon_i$ with
	$\E\|\bm\varepsilon_i\|_1\le\kappa$ by the separation assumption, we have
	$|\bm\pi_i^\top\bm\pi_t-\ind\{R_i=R_t\}|\le\|\bm\varepsilon_i\|_1+\|\bm\varepsilon_t\|_1$.
	Combining, the total deviation of the implemented weights from the oracle indicator weights is
	$O(\delta^g_t+\kappa)$ in expectation. A perturbation of this size to the regime-$r$ weighted
	CDF moves the attained coverage by at most its mass share, that is, by
	$O((\delta^g_t+\kappa)/p_r)$ after renormalizing by the regime prior $p_r=\Prob(R_t=r)$, giving
	the stated bound. Setting $\kappa=0$, which corresponds to observed or perfectly separated
	regimes, recovers the form that depends on estimation error alone; for overlapping regimes the
	$\kappa$ term does not vanish with sample size. \end{proof}

\begin{proof}[Proof of Proposition~\ref{prop:agg}]
	\DRACP-Select mixes the $J$ base radii with Hedge weights
	$\omega_{j,t+1}\propto\omega_{j,t}e^{-\zeta\,\mathrm{IS}^{(j)}_\alpha(t)}$ on interval-score losses
	bounded by $B$. By the standard exponentially-weighted-average regret bound, with
	$\zeta=\sqrt{8\log J/T}/B$,
	\[
	\sum_{t\le T}\mathrm{IS}_\alpha(C^\star_t)-\min_j\sum_{t\le T}\mathrm{IS}^{(j)}_\alpha(t)
	\ \le\ B\sqrt{(T/2)\log J},
	\]
	using convexity of the interval score in the radius, so that the loss of the mixed radius is at
	most the mixed loss. Dividing by $T$ shows that the average interval score of \DRACP-Select
	exceeds that of the best base method by at most $B\sqrt{\log J/(2T)}\to0$; in particular it is
	asymptotically no worse than the unweighted baseline. \end{proof}

\section{Reproducibility}
\label{app:repro}
This appendix documents what is required to reproduce the study end to end: data provenance,
the experimental protocol, the software environment, and the released artifacts.

\emph{Data.} All series are public and are downloaded programmatically. EU HICP annual
inflation for the 27 member states comes from the Eurostat dissemination API
(\texttt{prc\_hicp\_manr}, \texttt{coicop=CP00}, \texttt{unit=RCH\_A}); the US macroeconomic,
monetary and financial series come from FRED, identified by their series codes (for example
\texttt{UNRATE}, \texttt{INDPRO}, \texttt{CPIAUCSL}, \texttt{PAYEMS}, \texttt{DGS10},
\texttt{VIXCLS}); electricity demand is the UCI ElectricityLoadDiagrams archive aggregated to
hourly totals. Each series is transformed as recorded in Table~\ref{tab:data} (levels, first
differences or log-differences), and features are autoregressive lags, rolling means and
standard deviations, and sine/cosine seasonal encodings at the appropriate period. Daily
series are capped at their most recent observations to keep the sequential evaluation
tractable.

\emph{Protocol.} Every series is split chronologically into training, calibration and test
blocks with no shuffling and no use of future information. Methods share the same base
forecaster, the same splits and the same random seeds. Each experiment is repeated over 20
seeds; for the synthetic scenarios each seed redraws the data-generating process, while for the
fixed real series each seed reseeds the stochastic estimators, namely the gradient-boosted
forecaster, the density-ratio classifier and the Gaussian-mixture regime model. Across series,
the analysis uses Friedman tests with Nemenyi critical differences and Wilcoxon signed-rank
tests with Holm correction; per-series inference uses Diebold--Mariano tests with a Newey--West
long-run variance and the Harvey--Leybourne--Newbold small-sample correction, together with the
bootstrap described next.

\emph{Bootstrap specification.} Uncertainty on the fixed real series is quantified by a
\emph{paired circular} moving-block bootstrap with 2000 replicates. Four properties of this
procedure affect how the resulting intervals should be read. The block length is selected per
series by the automatic rule of \citet{politis2004automatic}, which balances the bias and
variance of the long-run variance estimator, and is capped at a quarter of the test length so
that short streams cannot receive a block long enough to make the resample degenerate; across
the 48 series the selected lengths have median 5 and range 1--19, so the dependence being
corrected for is present but short-lived. Blocks are circular, wrapping at the end of the
stream, so that every observation is resampled with equal probability. Resampling is
\emph{paired}: within a replicate every method is evaluated on the same resampled time points,
which makes the bootstrap valid for \emph{differences} between methods and not only for their
levels. Finally, the 20 repeated fits are pooled, with a fit index drawn at each replicate, so
that a single interval reflects both estimation variability across fits and dependence along
the test stream, rather than measuring the latter on one arbitrarily selected fit.

Pairing determines what can be concluded from the intervals. Comparing two marginal intervals
for overlap does not constitute a test of a difference; the comparison is conservative and can
conceal a difference that is in fact precisely estimated, because the two methods share the same
realized shocks. Paired differences are therefore reported directly. Across the 288 paired
comparisons, six baselines on each of the 48 series, the difference in mean interval score
favors \DRACP{} significantly at the $5\%$ level in 96 cases and goes against it in 21, with the
remaining 171 not separated. The pattern across baselines is uniform: \DRACP{} is significantly
better between 12 and 19 times against every individual baseline and significantly worse
between 2 and 5 times. These counts exceed those obtained from the Diebold--Mariano tests of
Section~\ref{sec:dm} on the same 288 comparisons, for two reasons. The bootstrap pools the 20
repeated fits, whereas the Diebold--Mariano tests are computed on a single fit; and the
Newey--West long-run variance, estimated from test streams of a few hundred observations and
combined with the Harvey--Leybourne--Newbold correction, is conservative at these lengths. The
two procedures agree on the sign and the ordering across baselines and differ in power, and the
more conservative Diebold--Mariano counts are the ones reported in the body of the paper.

\emph{Software and hyperparameters.} The implementation is pure Python built on NumPy, SciPy,
pandas and scikit-learn; no deep-learning framework is used. The base forecaster is a
histogram gradient-boosting regressor, the conditional scale model a random forest, the
density-ratio estimator a calibrated logistic classifier, and the regime model a Gaussian
mixture fitted causally on past data only. \DRACP{} hyperparameters (calibration window, target
window, temporal decay, bandwidth grid, ESS floor, weight cap, clipping bounds, number of
regimes and controller step size) are fixed per dataset family in configuration files and are
not tuned per series; their sensitivity is reported in \ref{app:figures}.

\emph{Artifacts.} The implementation is released as an open-source Python package under the MIT license,
installable from the Python Package Index as \texttt{dracp}. The source repository is at
\url{https://github.com/bogdanoancea/dracp}, and the version used for this paper is archived
under \doi{10.5281/zenodo.21675706}.The package provides a single entry point
(\texttt{run.sh}) that installs the package, downloads and prepares all data, runs the synthetic
and real experiments, executes the repeated-seed protocol, computes the Diebold--Mariano tests,
and regenerates every table and figure reported here via
\texttt{python -m reporting.paper\_outputs}. Raw per-step predictions, per-seed metrics,
bootstrap intervals and aggregate summaries are written to disk, so that all reported numbers
can be recomputed without rerunning the experiments. \ref{app:usage} documents the user-facing
interface, which applies \DRACP{} calibration to any user-supplied forecasting model.

\section{Using the software for your own forecasts}
\label{app:usage}
The released implementation is a Python package that calibrates \emph{any} point-prediction
model. Conformal validity does not depend on the model, so an existing forecasting model may be
retained and \DRACP{} intervals added on top of it. This appendix gives the installation steps
and a complete worked example.

\subsection{Installation}
The package requires Python~3.10 or later and depends only on the standard scientific stack
(NumPy, pandas, SciPy, scikit-learn). It is available from the Python Package Index:

\begin{verbatim}
	pip install <package-name-withheld-for-review>
\end{verbatim}

\noindent Alternatively, to work from the source repository (URL withheld for anonymous
review):

\begin{verbatim}
	python -m venv .venv
	source .venv/bin/activate          # Windows: .venv\Scripts\activate
	pip install -e .
\end{verbatim}

\noindent A quick check that the installation succeeded:

\begin{verbatim}
	python -c "from dracp import ConformalForecaster; print('ok')"
\end{verbatim}

\subsection{The interface}
The user-facing class is \texttt{ConformalForecaster}. It follows scikit-learn conventions and
takes the forecasting model as an argument:

\begin{verbatim}
	ConformalForecaster(model=None,          # any object with .fit / .predict
	alpha=0.1,           # 1 - alpha is the nominal coverage
	calibration_fraction=0.3,
	**config)            # any DRACPConfig field
\end{verbatim}

\noindent It exposes three methods. \texttt{fit(X, y)} trains the model on the earliest part of
the data and calibrates on the most recent \texttt{calibration\_fraction} of it, so the data
must be in chronological order. \texttt{predict\_interval(X, y=None)} returns arrays of lower
and upper bounds; if the realized outcomes \texttt{y} are passed, they are revealed one step at
a time after each prediction, so that calibration adapts online. This is the intended mode for
sequential forecasting and the one used throughout the paper. \texttt{score(X, y)} returns
empirical coverage, mean width and mean interval score. A convenience method
\texttt{predict\_frame} returns the same information as a tidy \texttt{DataFrame}. Any field of
the configuration may be passed as a keyword argument, for example \texttt{calibration\_window},
\texttt{n\_regimes}, \texttt{bandwidth}, \texttt{use\_faci\_control} or \texttt{asymmetric}.

\subsection{A complete example}
The following self-contained script builds a monthly series with a structural break, constructs
autoregressive features, wraps a random forest (standing in for the user's own model) in
\DRACP{}, and produces calibrated intervals. It runs as shown.

\begin{verbatim}
	import numpy as np, pandas as pd
	from sklearn.ensemble import RandomForestRegressor
	from dracp import ConformalForecaster
	
	# 1. A series with a level shift and a volatility change at t = 300
	rng = np.random.default_rng(0)
	n = 400; t = np.arange(n)
	level = np.where(t < 300, 0.0, 4.0)
	sigma = np.where(t < 300, 1.0, 2.0)
	y = level + 3*np.sin(2*np.pi*t/12) + rng.normal(0, sigma)
	
	# 2. Autoregressive and seasonal features (keep chronological order)
	df = pd.DataFrame({"y": y})
	for lag in (1, 2, 3, 12):
	df[f"lag_{lag}"] = df["y"].shift(lag)
	df["month_sin"] = np.sin(2*np.pi*(t % 12)/12)
	df["month_cos"] = np.cos(2*np.pi*(t % 12)/12)
	df = df.dropna().reset_index(drop=True)
	X, target = df.drop(columns="y"), df["y"]
	
	# 3. Chronological split
	split = int(0.7*len(df))
	X_tr, y_tr = X.iloc[:split], target.iloc[:split]
	X_te, y_te = X.iloc[split:], target.iloc[split:]
	
	# 4. Wrap any scikit-learn-style model
	# alpha is the target MISCOVERAGE rate: alpha=0.10 gives 90% intervals.
	cf = ConformalForecaster(
	model=RandomForestRegressor(n_estimators=200, random_state=0),
	alpha=0.10, calibration_window=150, n_regimes=3)
	cf.fit(X_tr, y_tr)
	
	# 5. Sequential prediction intervals.
	#    NOTE: the sequential methods consume the outcomes as they go and mutate the
	#    calibration state, exactly as an online forecaster would. Evaluate the test
	#    block ONCE and derive everything from that single pass; calling
	#    predict_interval, score and predict_frame in turn on the same data would
	#    replay the block three times and report optimistic figures.
	frame = cf.predict_frame(X_te, y_te)          # one pass over the test block
	lower, upper = frame["lower"], frame["upper"]
	
	# 6. Evaluate from the frame just produced; it already carries a 'covered' flag
	print(f"empirical coverage: {frame['covered'].mean():.3f}")
	print(frame.head())
\end{verbatim}

\noindent On this example the intervals attain an empirical coverage of $0.94$ against the
nominal $0.90$, with the widening occurring around the break. Replacing
\texttt{RandomForestRegressor} by a linear model, a gradient-boosting machine, or any other
object exposing \texttt{fit} and \texttt{predict} requires no other change. Passing \texttt{y}
to \texttt{predict\_interval} or \texttt{predict\_frame} enables the online adaptation;
omitting it yields fixed calibration instead.

\subsection{Accessing the variants}
\label{sec:usage_variants}
The default configuration is the one reported in the paper: composite weighting with the
self-tuning controller and symmetric scores. The two options discussed in
Section~\ref{sec:method} are accessed as follows.

\emph{Self-tuning controller} (Theorem~\ref{thm:controller}). The library default is the
single-rate adaptive-conformal controller, whose step size is \texttt{alpha\_step}. Every
experiment in this paper enables the self-tuning controller explicitly by passing
\texttt{use\_faci\_control=True}, which the supplied experiment configurations set
automatically:

\begin{verbatim}
	# self-tuning controller (the configuration used throughout the paper)
	cf = ConformalForecaster(model=my_model, alpha=0.1,
	use_faci_control=True)
	
	# single-rate adaptive-conformal controller (the library default)
	cf = ConformalForecaster(model=my_model, alpha=0.1,
	use_faci_control=False, alpha_step=0.01)
\end{verbatim}

\emph{Asymmetric intervals} \eqref{eq:asym}, for skewed forecast errors:

\begin{verbatim}
	cf = ConformalForecaster(model=my_model, alpha=0.1, asymmetric=True)
\end{verbatim}

\emph{\DRACP-Select} \eqref{eq:star}, the selection hedge against unweighted online
predictors. This variant is not a keyword but a composite estimator, since it runs several
predictors in parallel:

\begin{verbatim}
	from dracp.aggregate import build_unified_dracp
	from dracp import DRACPConfig
	
	cfg = {"alpha": 0.1, "dracp": {"calibration_window": 500, "n_regimes": 3}}
	dcfg = DRACPConfig(alpha=0.1, calibration_window=500, n_regimes=3)
	model = build_unified_dracp(cfg, seed=42, base_model=my_model,
	dcfg=dcfg, mode="select")
	model.fit(X_tr, y_tr, X_cal, y_cal)
	result = model.predict_one(x_t); model.update_one(x_t, y_t, result)
\end{verbatim}

\noindent Passing \texttt{mode="average"} replaces hard selection by a convex combination of
the constituent radii. In the experiment runners the same object is produced automatically
under the method name \texttt{select} mode when ablations are enabled.

\subsection{Practical guidance}
Three settings matter most in practice. The calibration window governs how much history is
retained: shorter windows adapt faster but reduce the effective sample size, on which the
method enforces a floor automatically. The number of regimes should reflect the number of
economically distinct states expected, typically two or three; the estimated regime posteriors
are exposed for inspection. The target coverage \texttt{alpha} should be set from the decision
problem rather than by convention. As reported in the main text, the method is designed for
series subject to several simultaneous shifts; on near-random-walk series such as daily asset
returns, simpler rolling calibration performs at least as well at a fraction of the cost and is
the better choice.

\section{Data-generating processes and series descriptions}
\label{app:data}
This appendix documents how the synthetic scenarios are generated and which real series are
used, so that both parts of the evaluation can be reproduced.

\subsection{Synthetic data-generating processes}
All seven scenarios share a common backbone. Covariates $X_t\in\R^{5}$ are drawn i.i.d.\
standard normal (before any scenario-specific perturbation), and the response is
\begin{equation}
	Y_t \;=\; c_t \;+\; X_t^\top\beta \;+\; \tfrac12\sin(X_{t,1}) \;+\; \tfrac14 X_{t,2}^{2}
	\;+\; \varepsilon_t,\qquad \varepsilon_t\sim\mathcal N(0,\sigma_t^{2}),
	\label{eq:dgp}
\end{equation}
with $\beta$ equally spaced on $[0.5,1.5]$. The sine and square terms render the conditional
mean nonlinear, so that a linear forecaster is misspecified and the conformal layer has a
non-trivial residual distribution to calibrate. The scenarios differ only in how the intercept
$c_t$ in \eqref{eq:dgp}, the innovation scale $\sigma_t$, the covariate distribution and the
latent regime $R_t$ evolve with $t$ over a sample of $n=3000$; the first $45\%$ is used for
training, the next $25\%$ for calibration and the remainder for sequential evaluation. Writing
$\phi_t=t/(n-1)$ for normalized time, the scenarios are as follows.

\begin{description}\setlength{\itemsep}{2pt}
	\item[Stationary.] $c_t=0$, $\sigma_t=1$, no perturbation: exchangeability holds and the
	classical guarantee applies. This is the control case.
	\item[Covariate shift.] With $u_t=\mathrm{clip}\big((t-0.55n)/(0.25n),0,1\big)$, the first two
	covariates drift, $X_{t,1}\!\to\!X_{t,1}+2u_t$ and $X_{t,2}\!\to\!X_{t,2}-1.2u_t$, while
	$P_{Y\mid X}$ is unchanged. Only the covariate marginal moves, which is the setting for which
	weighted conformal prediction is designed.
	\item[Conditional drift.] With $u_t=\mathrm{clip}\big((t-0.5n)/(0.35n),0,1\big)$, the intercept
	becomes $c_t=3u_t$ and the first covariate is scaled by $1+u_t$. Here $P_{Y\mid X}$ itself
	changes, so that no reweighting of the covariate marginal can restore validity. This is the
	failure mode of split conformal prediction.
	\item[Abrupt drift.] A single structural break at $t=0.65n$: $c_t$ jumps from $0$ to $4$ and
	$\sigma_t$ from $1$ to $1.8$, so that level and volatility change simultaneously and without
	warning.
	\item[Gradual drift.] $c_t=2\sin(2\pi\phi_t)$, $X_{t,1}$ scaled by $0.5+1.5\phi_t$ and
	$\sigma_t=0.7+\phi_t$: slow, continuous evolution of mean, covariate scale and volatility.
	\item[Regime switching.] A three-state Markov chain with transition matrix
	$P=\big(\begin{smallmatrix}0.97&0.02&0.01\\ 0.03&0.94&0.03\\ 0.02&0.04&0.94\end{smallmatrix}\big)$
	drives $c_t\in\{-2.0,0.5,3.0\}$, $\sigma_t\in\{0.6,1.2,2.0\}$ and a regime-dependent shift of
	$X_{t,1}$ by $\{-1.0,0,1.5\}$. Regimes are persistent and recur, as do expansions and
	recessions.
	\item[Combined shift.] All mechanisms operate at once: a three-state chain
	($P$ with diagonal $0.96,0.94,0.93$) together with a trend, so that
	$X_{t,1}$ shifts by $2.2\phi_t$ plus a regime effect, $c_t=2.5\phi_t$ plus a regime effect in
	$\{-1.5,0.2,2.5\}$, and $\sigma_t=0.7+0.9\phi_t$ plus a regime effect. This scenario matches
	the structure of real economic series most closely, and is the one used for the ablation and
	sensitivity analyses.
\end{description}

The true regime label $R_t$ is recorded but never shown to any method; it is used only to
compute the regime-conditional diagnostics of \ref{app:figures}.

\subsection{Real forecasting series}
Table~\ref{tab:appdata} lists all 48 real series with source, frequency, number of
sequential test predictions and the period spanned by the test block. Three groups are
distinguished in the analysis. The \emph{HICP inflation} panel comprises the euro-area
aggregate (Eurostat geo \texttt{EA20}) and the 27 member states (\texttt{prc\_hicp\_manr},
annual rate of change of the harmonized index, all-items COICOP \texttt{CP00}), 28 series in
total; these share a data-generating mechanism but differ in the timing and amplitude of
national shocks, and all span the 2021--23 inflation surge. The \emph{US macro and energy}
group covers real activity, prices, money and sentiment from FRED, together with hourly
electricity demand from the UCI ElectricityLoadDiagrams archive aggregated to system totals.
The \emph{financial} group consists of the seven daily market and interest-rate series and
serves as a negative control: these are close to martingale differences, so that little
exploitable shift structure is present. They are described here as financial-market and
interest-rate changes rather than as ``returns'', since the policy-rate and spread series are
not asset returns.

Each series is modeled one step ahead from its own past, using autoregressive lags, rolling
means and standard deviations, and sine/cosine seasonal encodings at the natural period (12 for
monthly, 7 for daily). Growth series enter in log-differences, rates and spreads in first
differences, and levels, including the unemployment rate and the HICP annual rates, enter
untransformed; the transformation applied to each series is recorded in
Table~\ref{tab:appdata}. Splits are chronological, with no shuffling and no use of future
information at any point.

Two practical caveats affect exact reproduction. Monthly macroeconomic series are subject to
statistical revision, so a download at a later date may differ slightly from the one used here
for recent observations. Daily series are capped at their most recent observations to keep the
sequential evaluation tractable, so the evaluation window advances with the download date; this
affects the daily group only, and shifts results by amounts far smaller than the differences
between methods.

\begingroup
\footnotesize
\setlength{\tabcolsep}{4pt}

\begin{longtable}{llccll}
	\caption{Real forecasting series used in the empirical evaluation. The table
		reports the source, frequency, modeling transformation, number of sequential
		one-step-ahead test predictions, and test period for each series.}
	\label{tab:appdata}\\
	
	\toprule
	Series & Source & Freq. & Transform & Test & Test period \\
	\midrule
	\endfirsthead
	
	\multicolumn{6}{c}{\tablename\ \thetable{} -- continued from previous page}\\
	\toprule
	Series & Source & Freq. & Transform & Test & Test period \\
	\midrule
	\endhead
	
	\midrule
	\multicolumn{6}{r}{Continued on next page}\\
	\endfoot
	
	\bottomrule
	\endlastfoot
	
	\multicolumn{6}{l}{\textit{HICP inflation}}\\
	\midrule
	
	Euro area HICP & Eurostat & M & rate (level) & 73 & 2019-12--2025-12 \\
	HICP AT & Eurostat & M & rate (level) & 84 & 2019-01--2025-12 \\
	HICP BE & Eurostat & M & rate (level) & 84 & 2019-01--2025-12 \\
	HICP BG & Eurostat & M & rate (level) & 82 & 2019-03--2025-12 \\
	HICP CY & Eurostat & M & rate (level) & 84 & 2019-01--2025-12 \\
	HICP CZ & Eurostat & M & rate (level) & 84 & 2019-01--2025-12 \\
	HICP DE & Eurostat & M & rate (level) & 84 & 2019-01--2025-12 \\
	HICP DK & Eurostat & M & rate (level) & 84 & 2019-01--2025-12 \\
	HICP EE & Eurostat & M & rate (level) & 84 & 2019-01--2025-12 \\
	HICP EL & Eurostat & M & rate (level) & 84 & 2019-01--2025-12 \\
	HICP ES & Eurostat & M & rate (level) & 84 & 2019-01--2025-12 \\
	HICP FI & Eurostat & M & rate (level) & 84 & 2019-01--2025-12 \\
	HICP FR & Eurostat & M & rate (level) & 84 & 2019-01--2025-12 \\
	HICP HR & Eurostat & M & rate (level) & 78 & 2019-07--2025-12 \\
	HICP HU & Eurostat & M & rate (level) & 84 & 2019-01--2025-12 \\
	HICP IE & Eurostat & M & rate (level) & 84 & 2019-01--2025-12 \\
	HICP IT & Eurostat & M & rate (level) & 84 & 2019-01--2025-12 \\
	HICP LT & Eurostat & M & rate (level) & 84 & 2019-01--2025-12 \\
	HICP LU & Eurostat & M & rate (level) & 84 & 2019-01--2025-12 \\
	HICP LV & Eurostat & M & rate (level) & 84 & 2019-01--2025-12 \\
	HICP MT & Eurostat & M & rate (level) & 84 & 2019-01--2025-12 \\
	HICP NL & Eurostat & M & rate (level) & 84 & 2019-01--2025-12 \\
	HICP PL & Eurostat & M & rate (level) & 84 & 2019-01--2025-12 \\
	HICP PT & Eurostat & M & rate (level) & 84 & 2019-01--2025-12 \\
	HICP RO & Eurostat & M & rate (level) & 84 & 2019-01--2025-12 \\
	HICP SE & Eurostat & M & rate (level) & 84 & 2019-01--2025-12 \\
	HICP SI & Eurostat & M & rate (level) & 84 & 2019-01--2025-12 \\
	HICP SK & Eurostat & M & rate (level) & 84 & 2019-01--2025-12 \\
	
	\midrule
	\multicolumn{6}{l}{\textit{US macroeconomic and energy}}\\
	\midrule
	
	Electricity demand & UCI & H & level & 5000 & 2014-06-06 17:00--2015-01-01 00:00 \\
	US 10Y Treasury & FRED & M & diff & 217 & 2008-06--2026-06 \\
	US CPI inflation & FRED & M & log-diff & 233 & 2006-05--2025-09 \\
	US M2 & FRED & M & log-diff & 200 & 2009-11--2026-06 \\
	US PCE inflation & FRED & M & log-diff & 199 & 2009-11--2026-05 \\
	US PPI & FRED & M & log-diff & 338 & 1998-05--2026-06 \\
	US capacity utilization & FRED & M & diff & 176 & 2011-11--2026-06 \\
	US consumer sentiment & FRED & M & diff & 142 & 2014-08--2026-05 \\
	US housing starts & FRED & M & log-diff & 200 & 2009-11--2026-06 \\
	US industrial production & FRED & M & log-diff & 320 & 1999-11--2026-06 \\
	US nonfarm payrolls & FRED & M & log-diff & 260 & 2004-11--2026-06 \\
	US retail sales & FRED & M & log-diff & 101 & 2018-02--2026-06 \\
	US unemployment & FRED & M & level & 231 & 2006-07--2025-09 \\
	
	\midrule
	\multicolumn{6}{l}{\textit{Daily financial}}\\
	\midrule
	
	Fed funds (d) & FRED & D & diff & 745 & 2024-07-14--2026-07-28 \\
	S\&P 500 (d) & FRED & D & log-diff & 207 & 2023-11-01--2026-05-22 \\
	Term spread (d) & FRED & D & diff & 207 & 2023-04-21--2026-05-22 \\
	US 10Y Treasury (d) & FRED & D & diff & 207 & 2023-04-21--2026-05-22 \\
	USD/EUR (d) & FRED & D & log-diff & 259 & 2023-03-31--2026-05-22 \\
	VIX (d) & FRED & D & log-diff & 423 & 2024-07-24--2026-07-28 \\
	WTI oil (d) & FRED & D & log-diff & 255 & 2022-03-25--2026-05-22 \\
	
\end{longtable}

\endgroup

\section{Additional figures}
\label{app:figures}
The figures collected here support the analysis in three ways. The first group reports
per-series results over the 20 repeated fits, namely mean interval score and empirical coverage
with 95\% paired moving-block bootstrap intervals, for a representative subset of structured
and financial series (Figures~\ref{fig:mcgrid} and~\ref{fig:covgrid}), showing the dispersion
behind the aggregate results of Section~\ref{sec:results} and the per-series values in the
supplementary material. The second group illustrates the internal
behavior of the method: the estimated density-ratio contribution, the posterior regime
probabilities, the trajectory of the self-tuned significance level, the effective calibration
sample size as a function of bandwidth, the estimated regime transition matrix, the
localization weights, and residual diagnostics (Figures~\ref{fig:diag} and~\ref{fig:diag2}).
These are the quantities that Theorem~\ref{thm:rate} identifies as controlling the coverage
gap, and they therefore also serve as deployment diagnostics. The third group covers robustness
and cost: sensitivity to the localization bandwidth, the temporal forgetting factor and
density-ratio clipping (Figure~\ref{fig:sens}); runtime and memory scaling with the calibration
buffer (Figure~\ref{fig:cost}); and coverage across synthetic scenarios of increasing shift
complexity together with interval-width distributions (Figure~\ref{fig:covdrift}), alongside
example prediction intervals on the synthetic (Figure~\ref{fig:synthfig}) and real
(Figure~\ref{fig:pred}) series.

\begin{figure}[htbp]\centering
	\includegraphics[width=.32\textwidth]{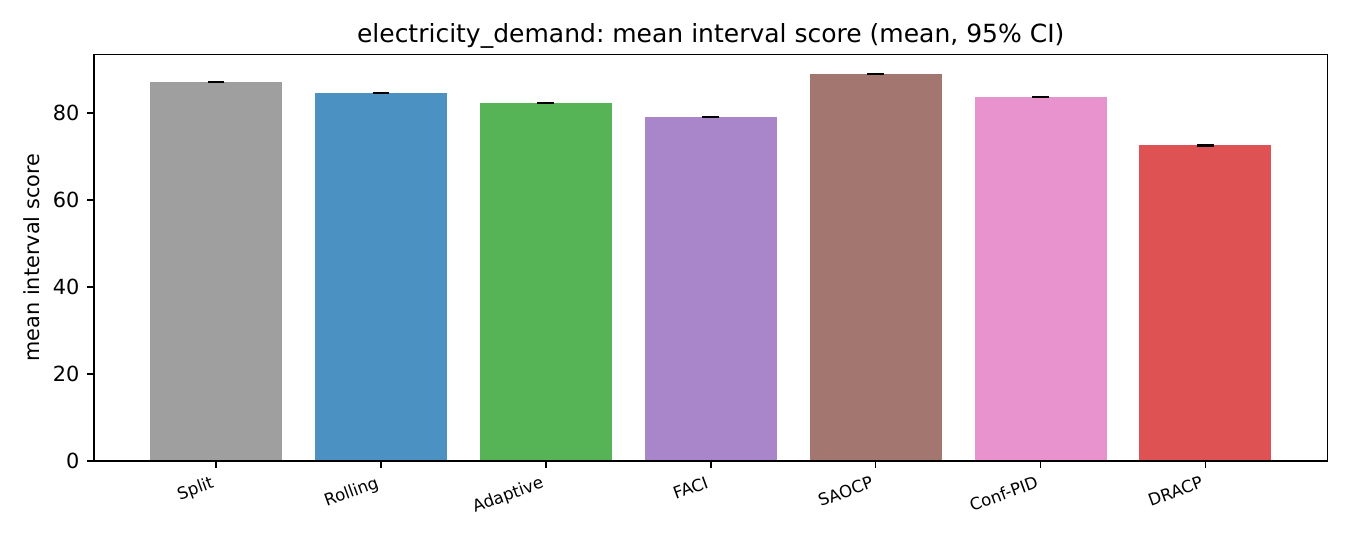}\hfill
	\includegraphics[width=.32\textwidth]{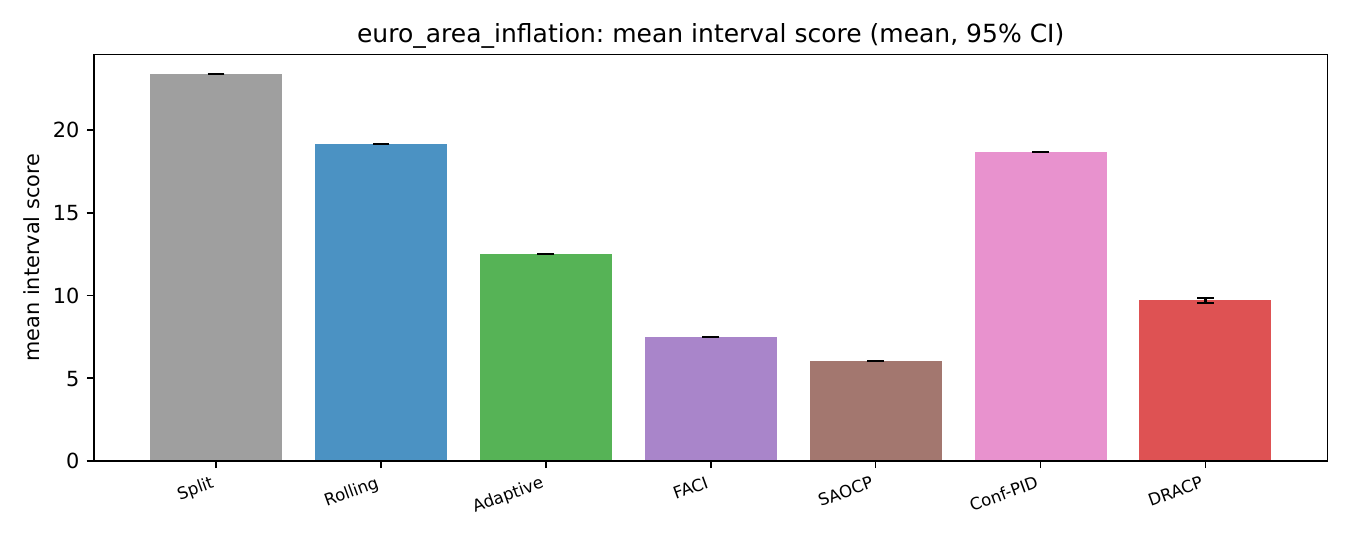}\hfill
	\includegraphics[width=.32\textwidth]{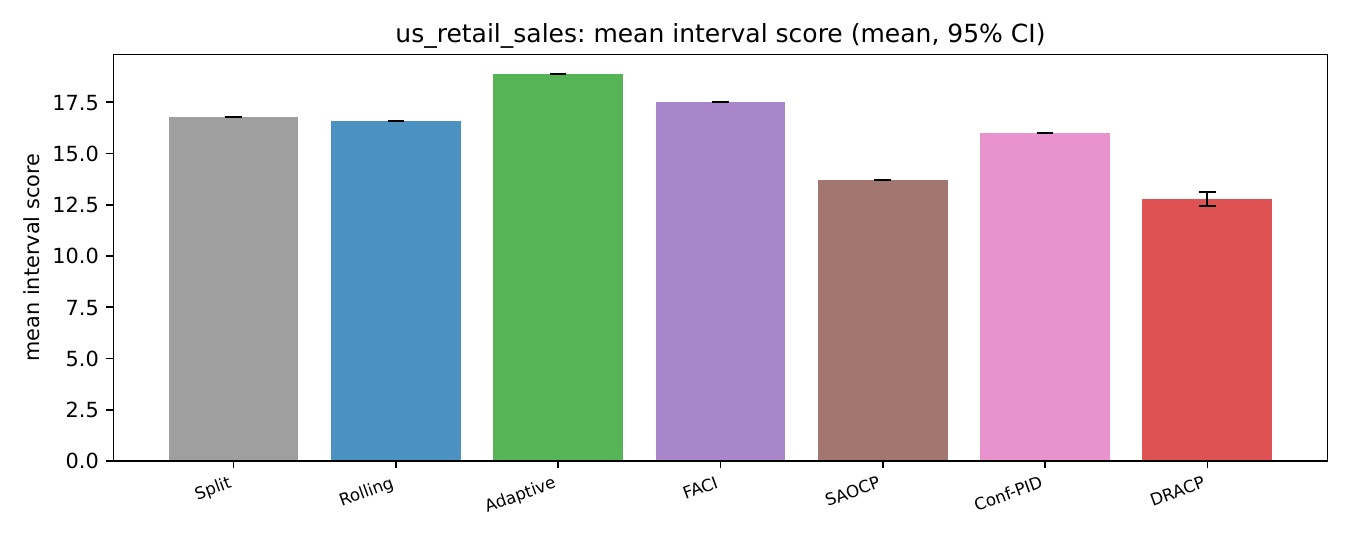}\\[2pt]
	\includegraphics[width=.32\textwidth]{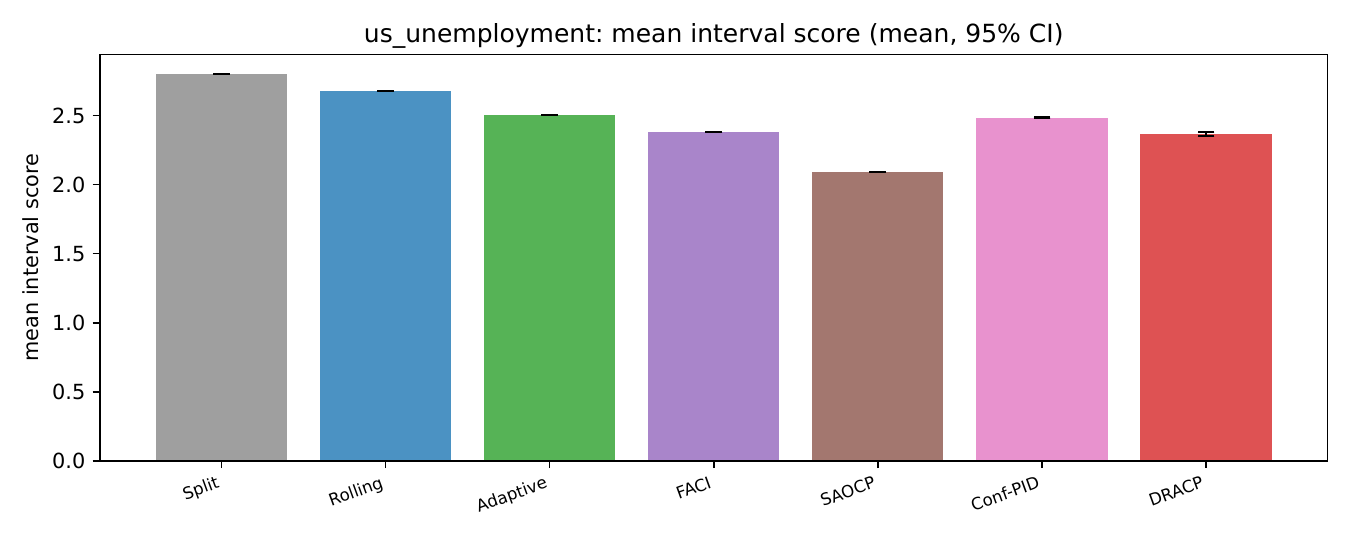}\hfill
	\includegraphics[width=.32\textwidth]{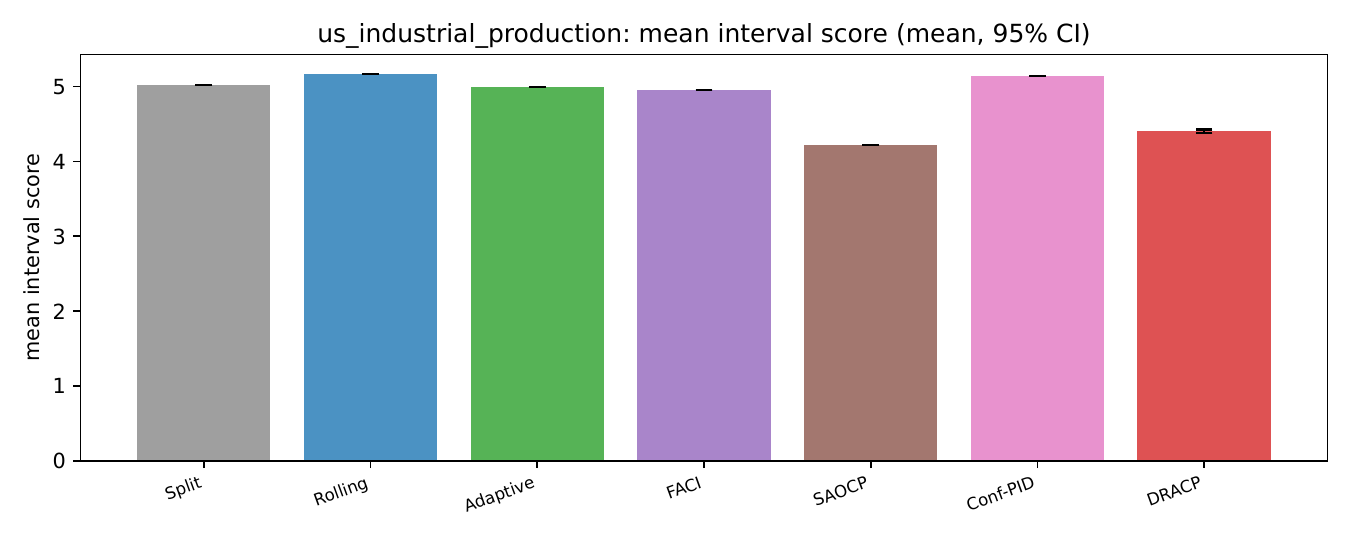}\hfill
	\includegraphics[width=.32\textwidth]{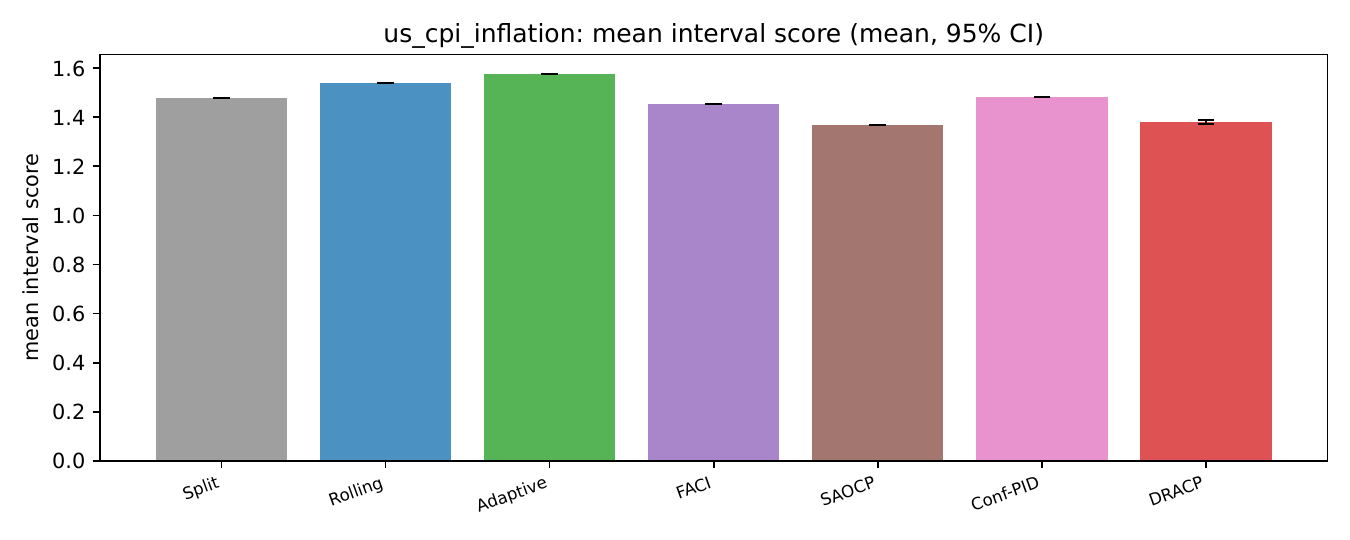}\\[2pt]
	\includegraphics[width=.32\textwidth]{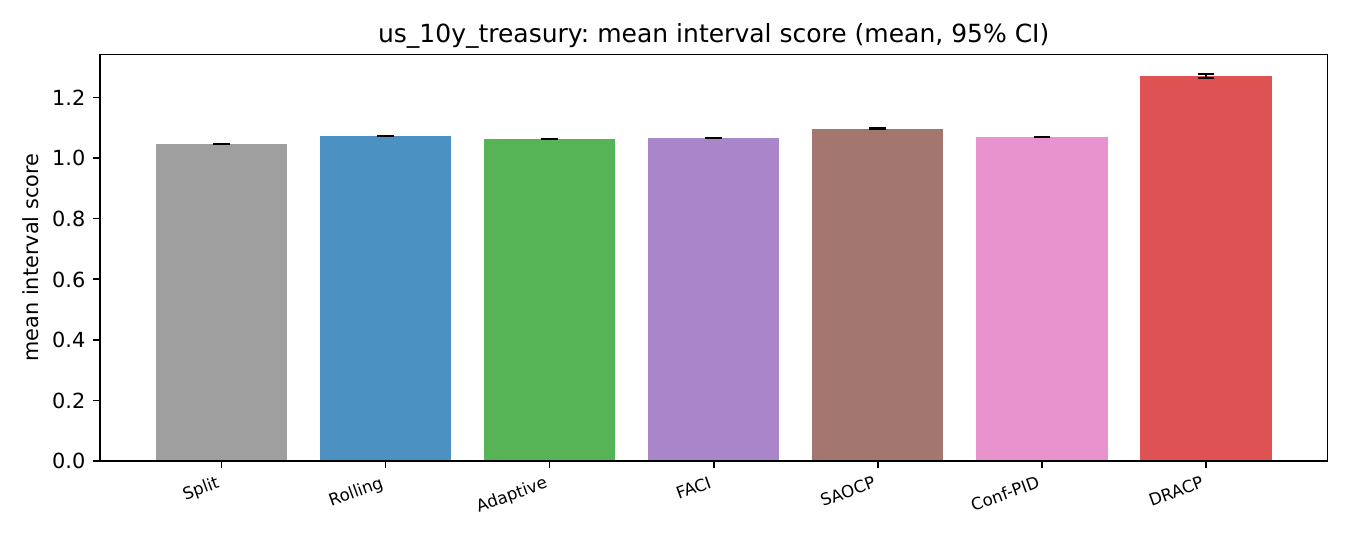}\hfill
	\includegraphics[width=.32\textwidth]{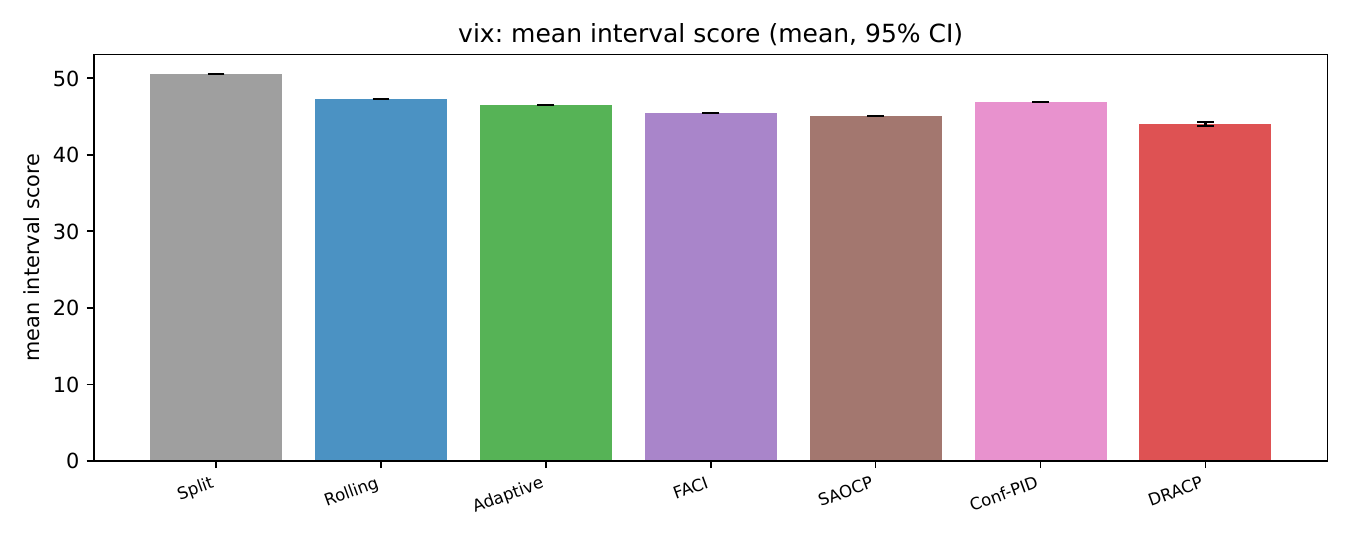}\hfill
	\includegraphics[width=.32\textwidth]{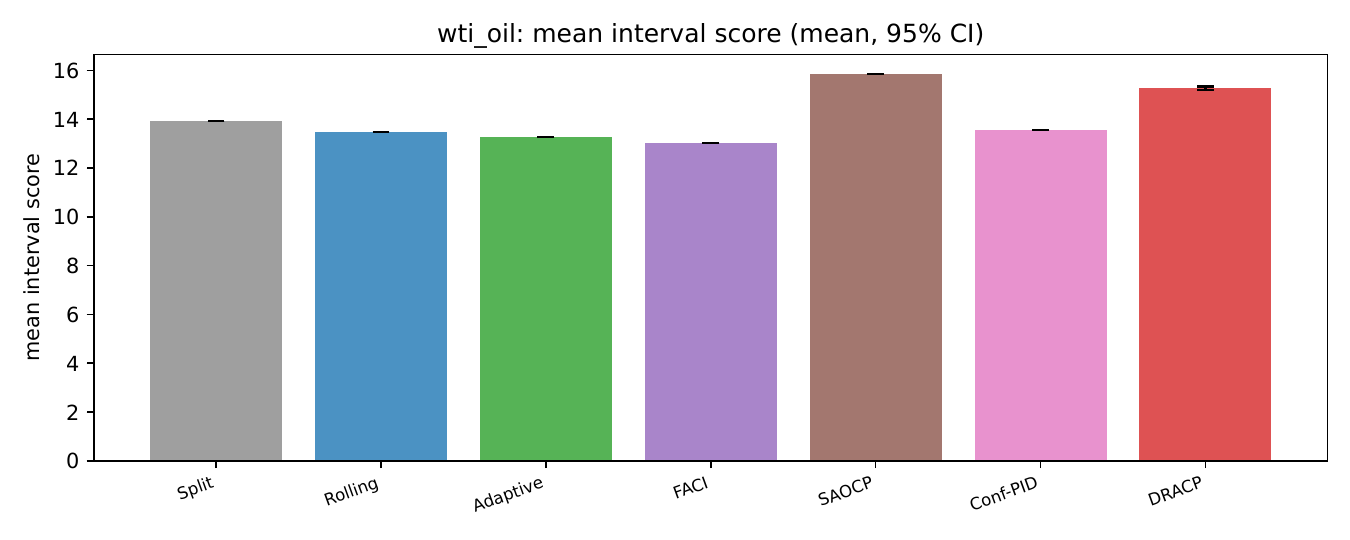}
	\caption{Per-series mean interval score (means over 20 repeated fits, 95\% paired moving-block bootstrap
		whiskers) for a representative subset of structured (top rows) and financial (bottom row)
		series. }
	\label{fig:mcgrid}
\end{figure}

\begin{figure}[htbp]\centering
	\includegraphics[width=.32\textwidth]{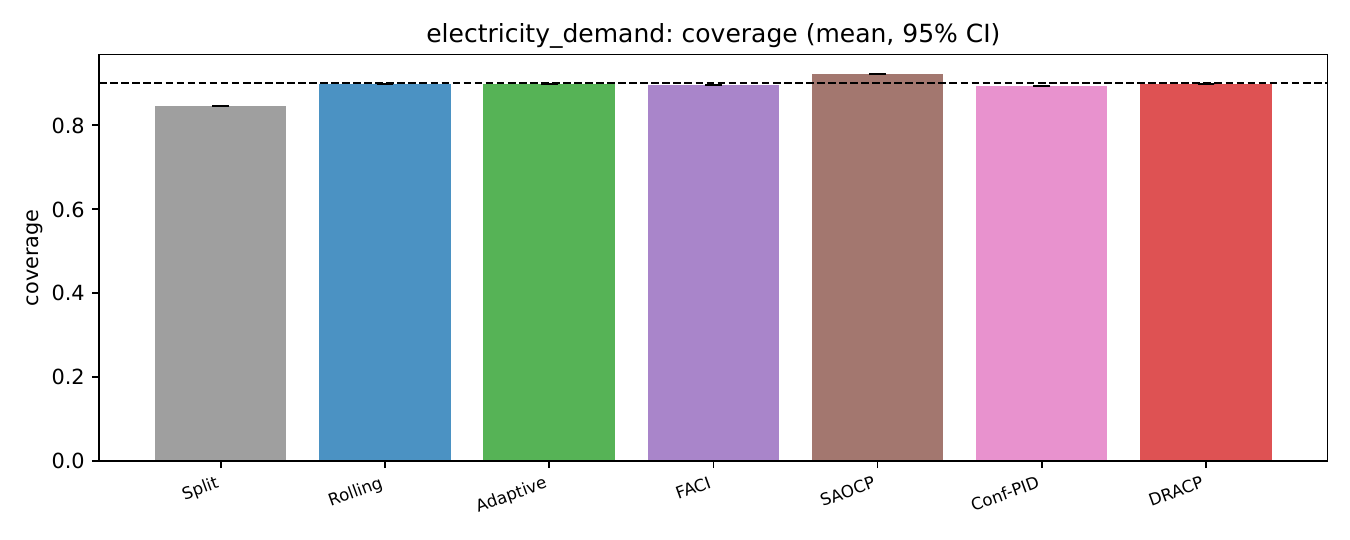}\hfill
	\includegraphics[width=.32\textwidth]{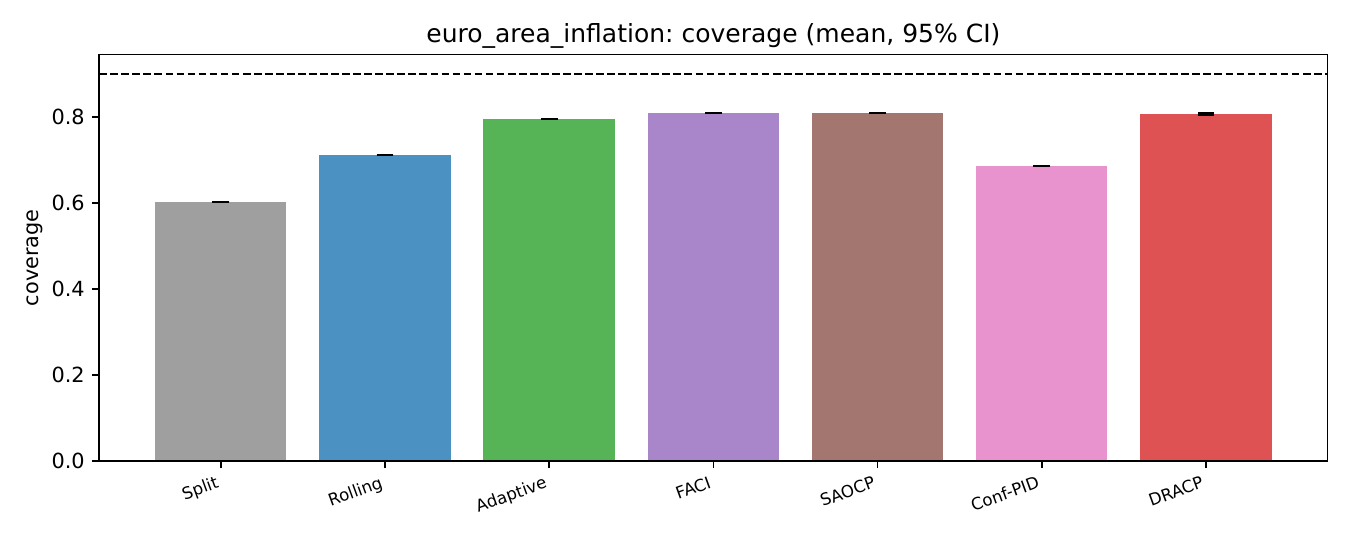}\hfill
	\includegraphics[width=.32\textwidth]{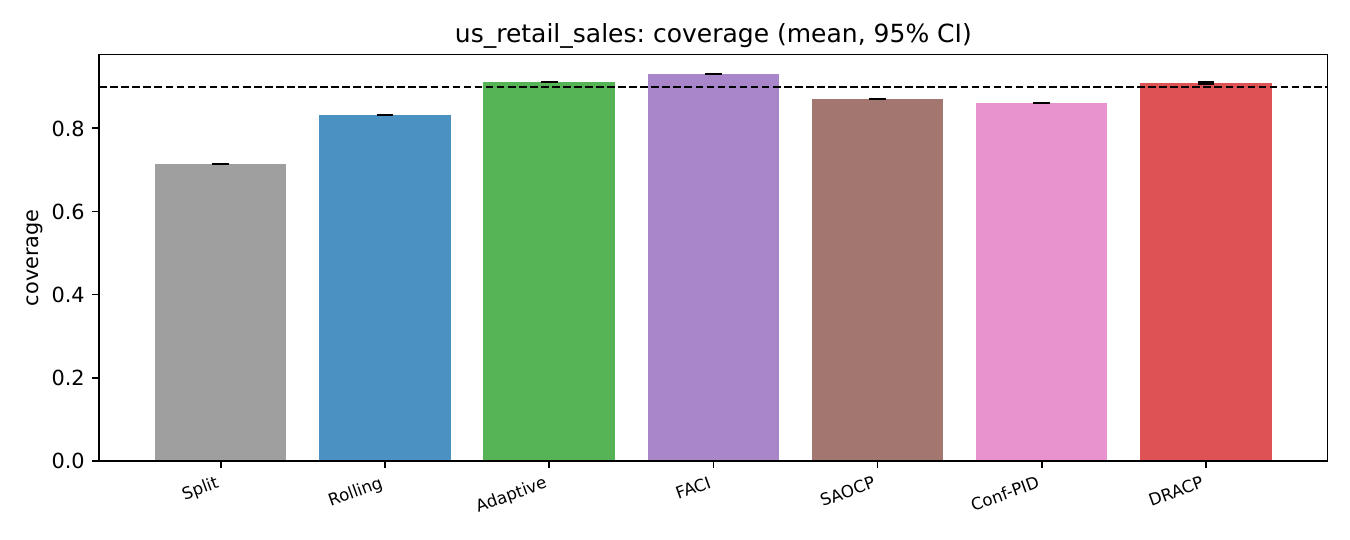}\\[2pt]
	\includegraphics[width=.32\textwidth]{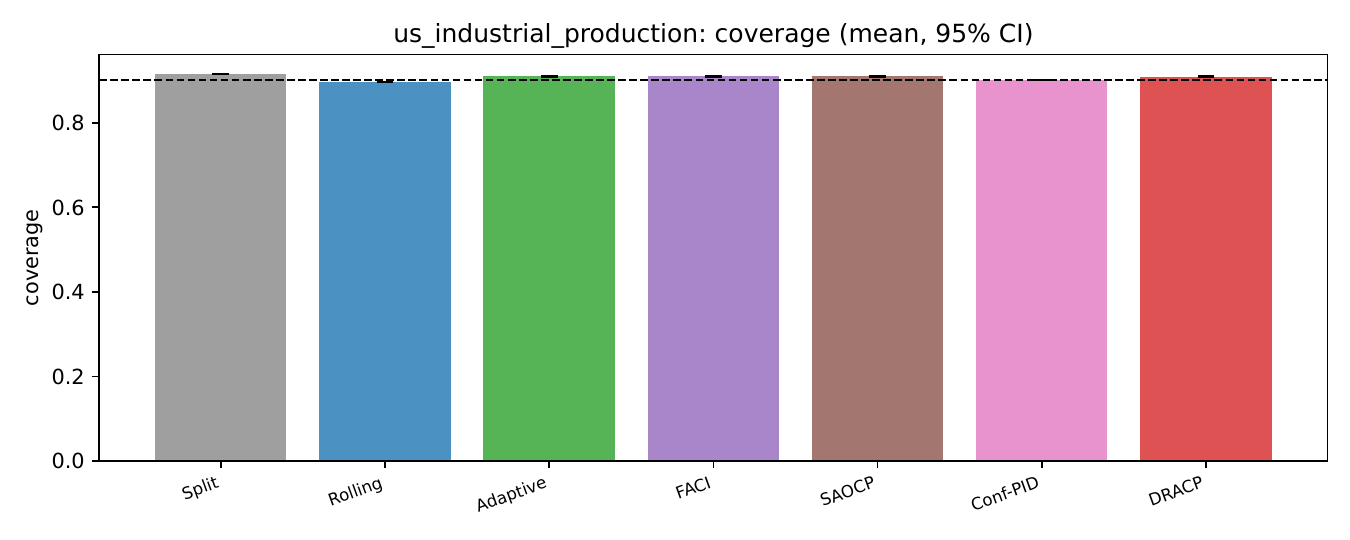}\hfill
	\includegraphics[width=.32\textwidth]{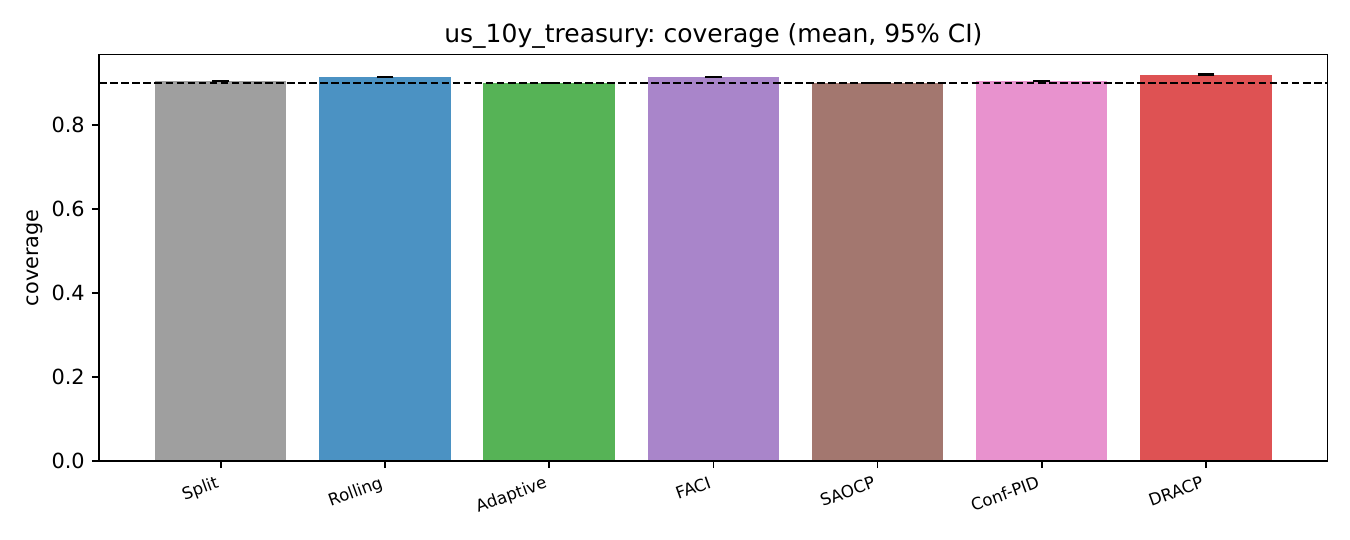}\hfill
	\includegraphics[width=.32\textwidth]{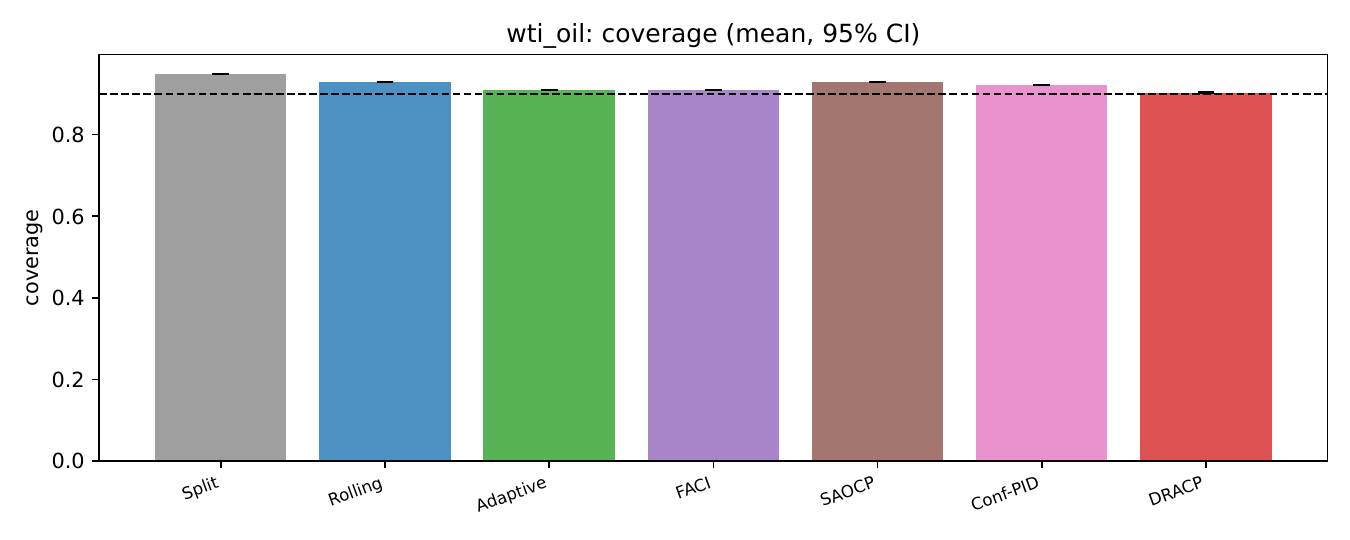}
	\caption{Per-series empirical coverage (20-seed means, 95\% bootstrap whiskers, nominal
		$0.90$ dashed) for representative structured (top) and financial (bottom) series.}
	\label{fig:covgrid}
\end{figure}

\begin{figure}[htbp]\centering
	\includegraphics[width=.42\textwidth]{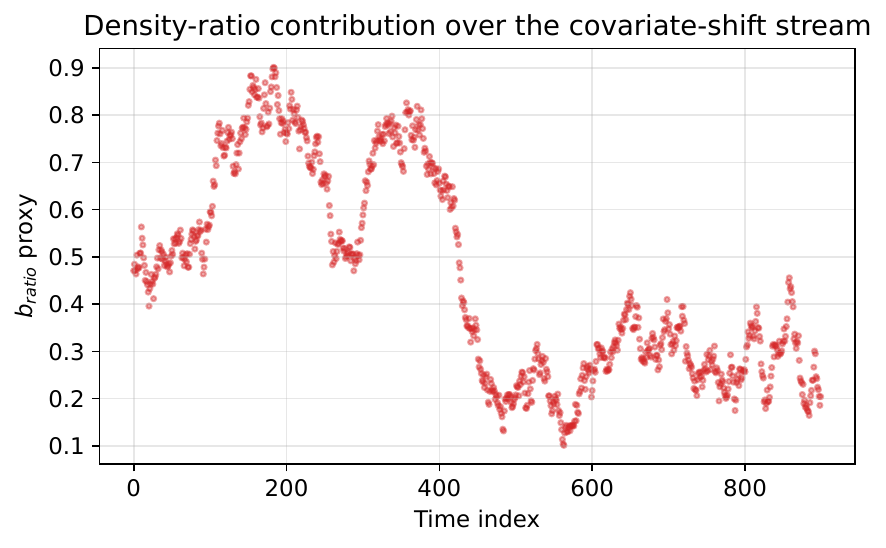}\hfill
	\includegraphics[width=.42\textwidth]{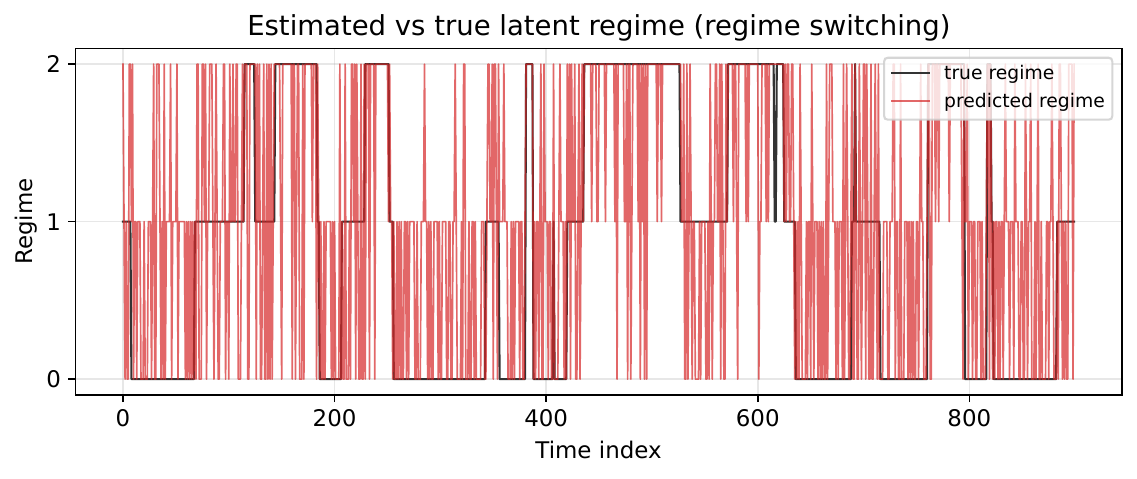}\\[2pt]
	\includegraphics[width=.42\textwidth]{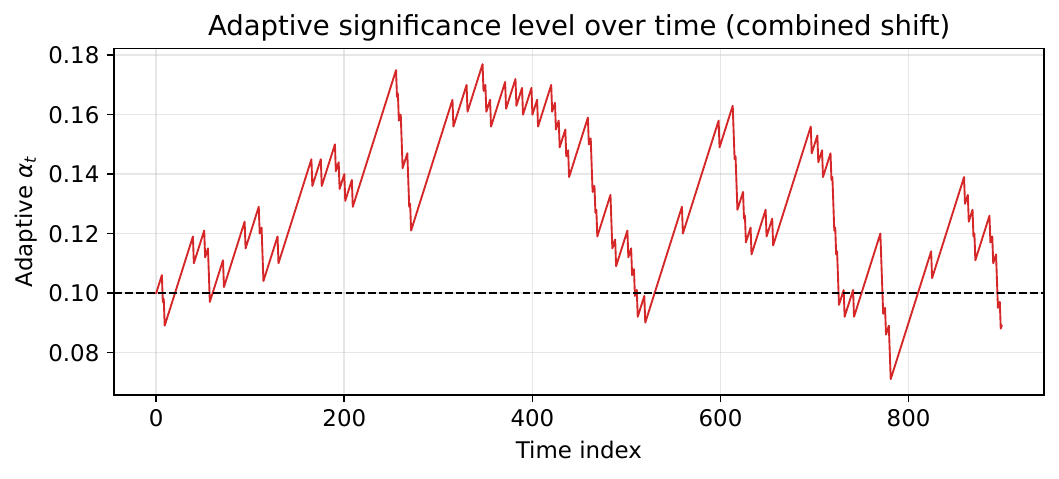}\hfill
	\includegraphics[width=.42\textwidth]{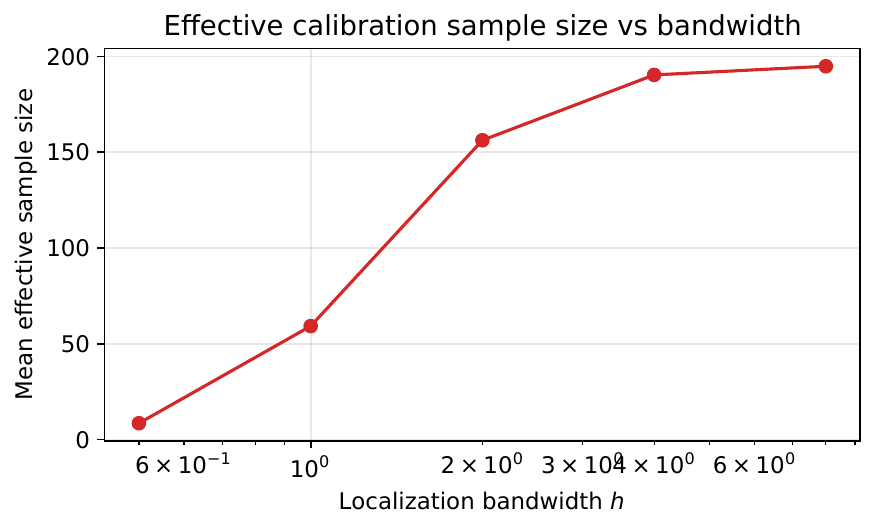}
	\caption{Diagnostics. Top: estimated density-ratio contribution and posterior regime
		probabilities. Bottom: the online significance level $\alpha_t$ and the effective
		calibration sample size versus bandwidth.}
	\label{fig:diag}
\end{figure}

\begin{figure}[htbp]\centering
	\includegraphics[width=.42\textwidth]{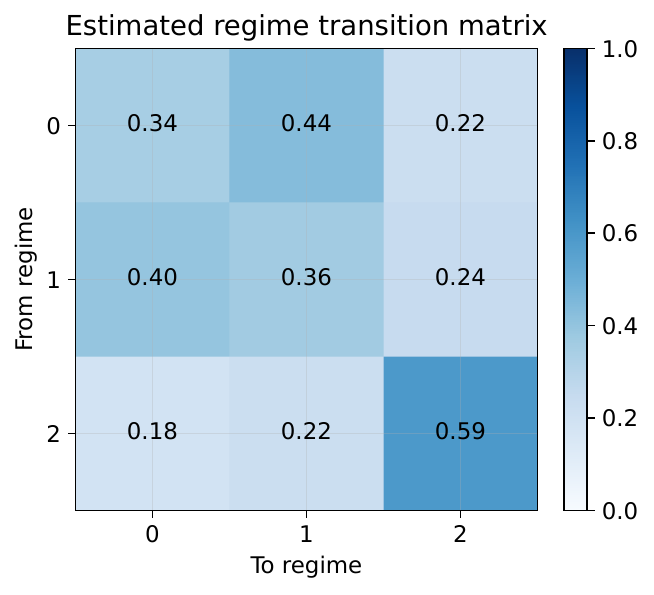}\hfill
	\includegraphics[width=.42\textwidth]{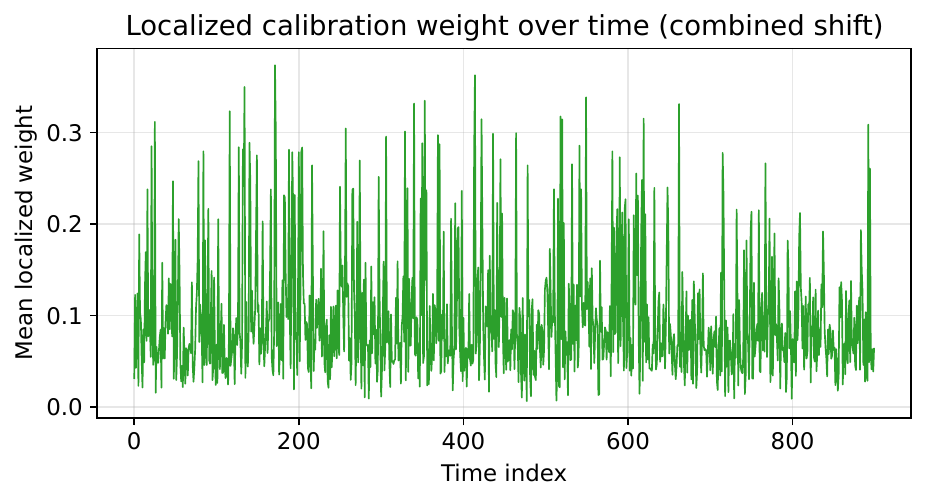}\\[2pt]
	\includegraphics[width=.42\textwidth]{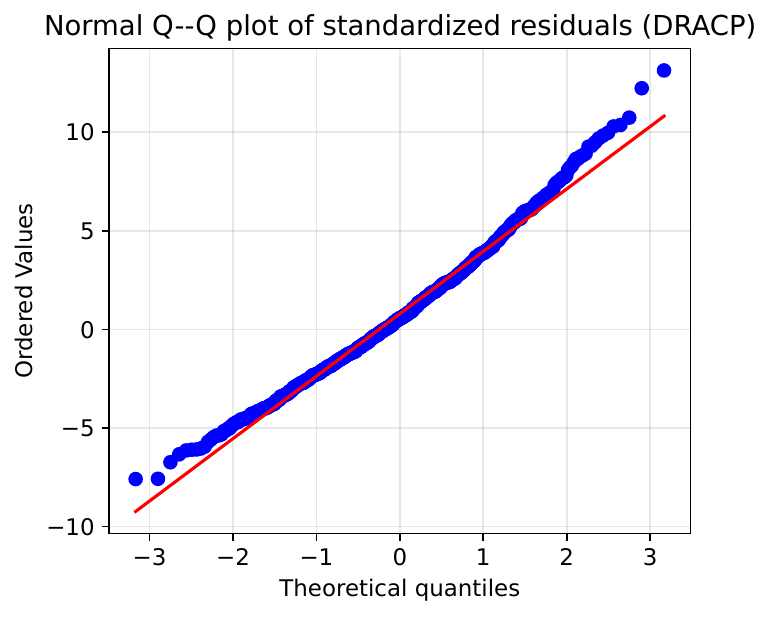}\hfill
	\includegraphics[width=.42\textwidth]{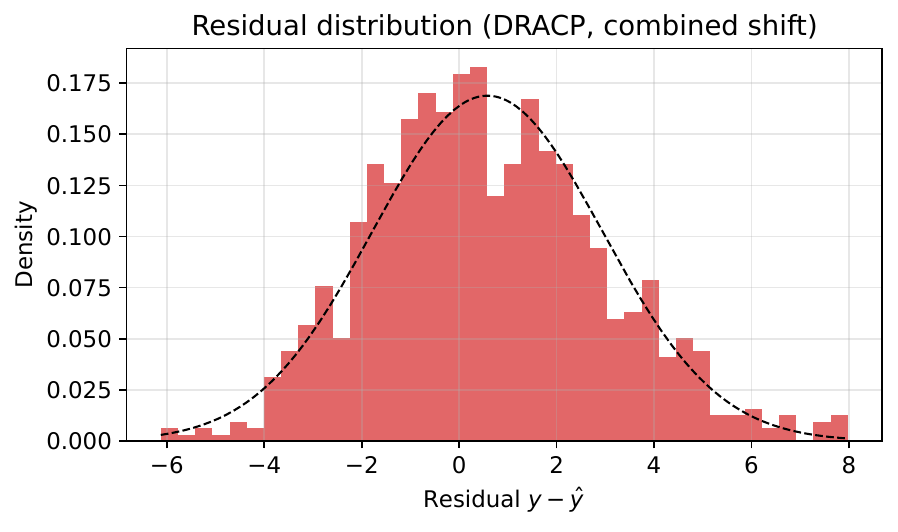}
	\caption{Further diagnostics: estimated regime transition matrix, localization weights vs.\
		distance, and the normal Q--Q plot and histogram of standardized residuals.}
	\label{fig:diag2}
\end{figure}

\begin{figure}[htbp]\centering
	\includegraphics[width=.32\textwidth]{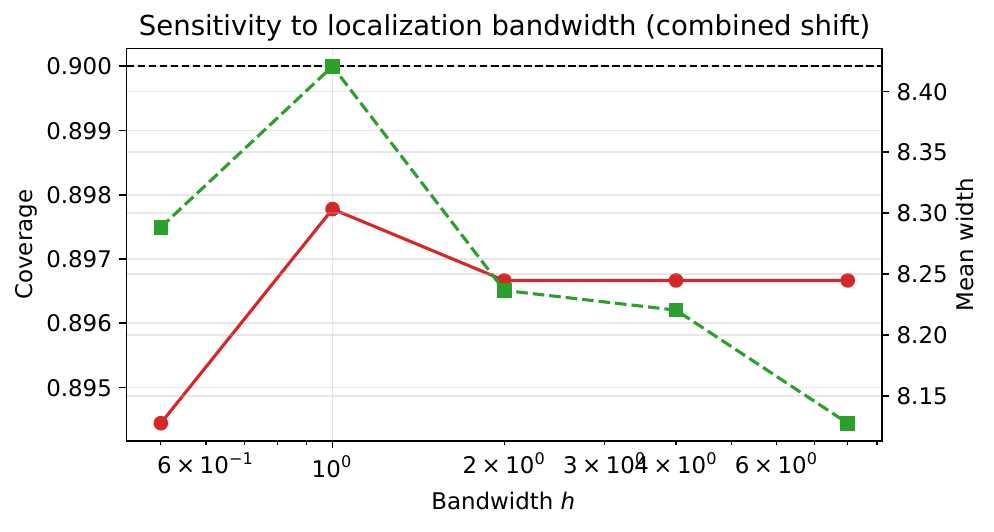}\hfill
	\includegraphics[width=.32\textwidth]{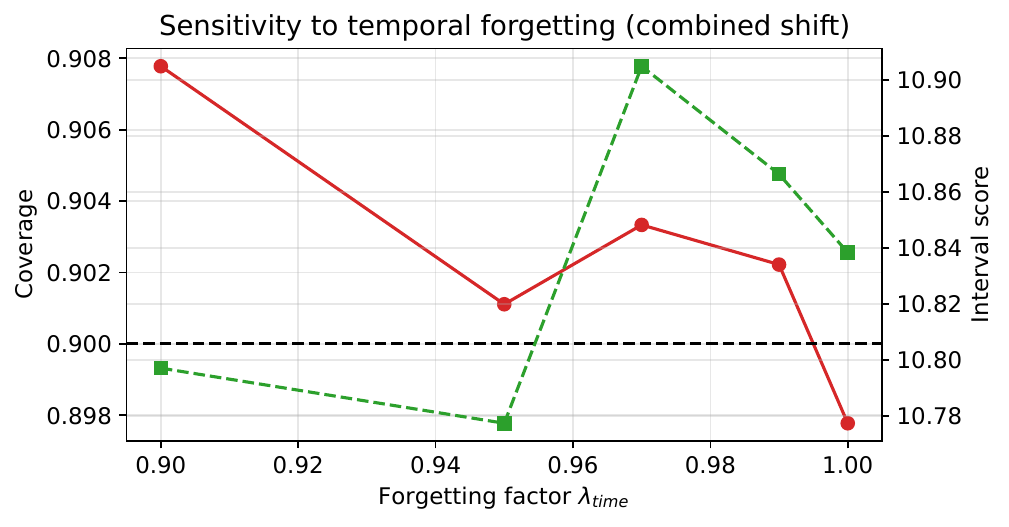}\hfill
	\includegraphics[width=.32\textwidth]{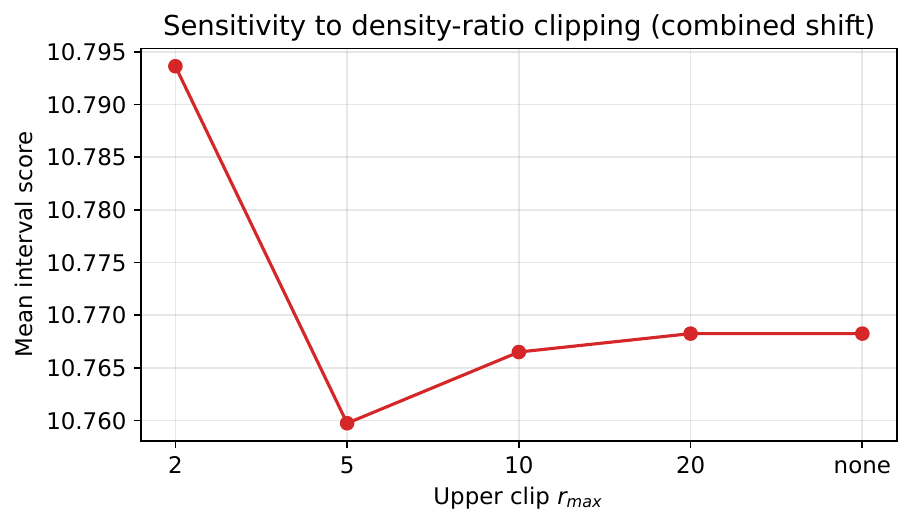}
	\caption{Hyperparameter sensitivity of \DRACP{} (combined-shift scenario): localization
		bandwidth, temporal forgetting factor, and density-ratio clipping.}
	\label{fig:sens}
\end{figure}

\begin{figure}[htbp]\centering
	\includegraphics[width=.42\textwidth]{figures/coverage_vs_drift.pdf}\hfill
	\includegraphics[width=.42\textwidth]{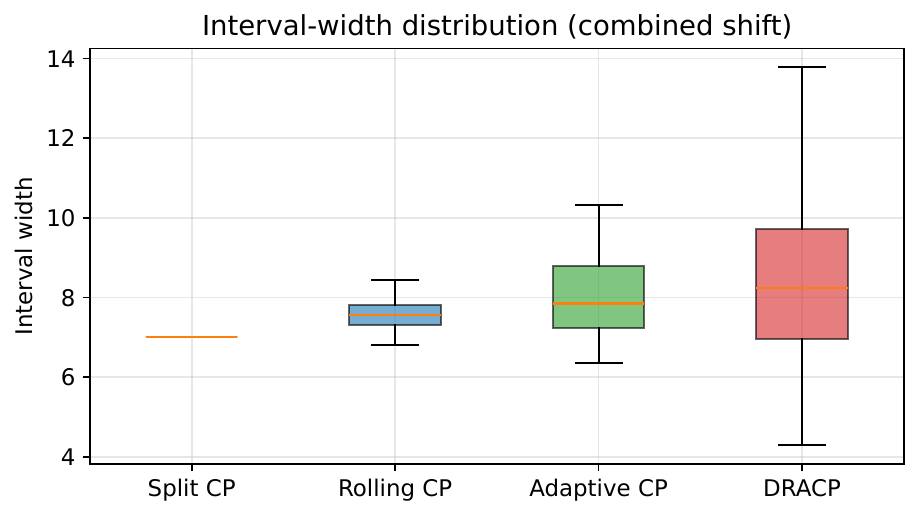}
	\caption{Left: coverage across scenarios of increasing shift complexity. Right: distribution
		of interval widths by method under combined shift.}
	\label{fig:covdrift}
\end{figure}

\begin{figure}[htbp]\centering
	\includegraphics[width=.42\textwidth]{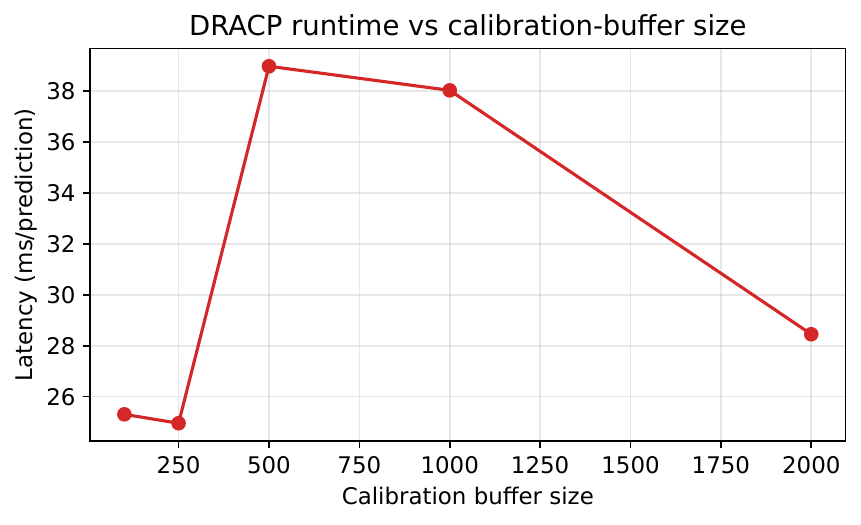}\hfill
	\includegraphics[width=.42\textwidth]{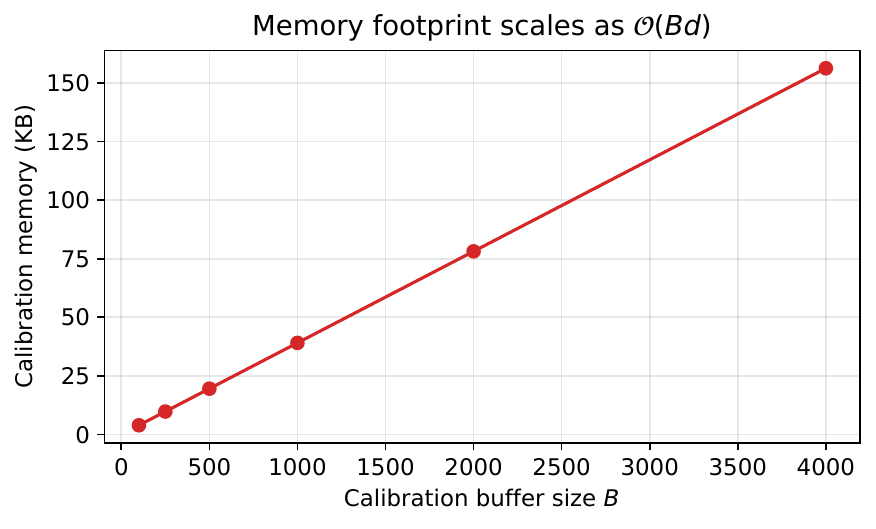}
	\caption{Computational cost of \DRACP{}: runtime per prediction versus calibration-buffer
		size (left) and the $\mathcal{O}(Bd)$ memory footprint (right).}
	\label{fig:cost}
\end{figure}

\begin{figure}[htbp]\centering
	\includegraphics[width=.48\textwidth]{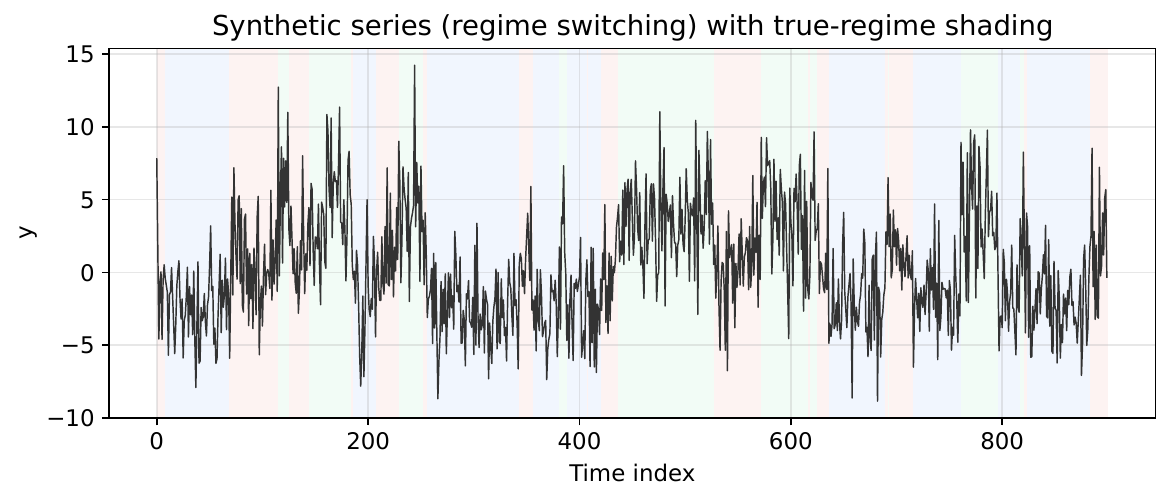}\hfill
	\includegraphics[width=.48\textwidth]{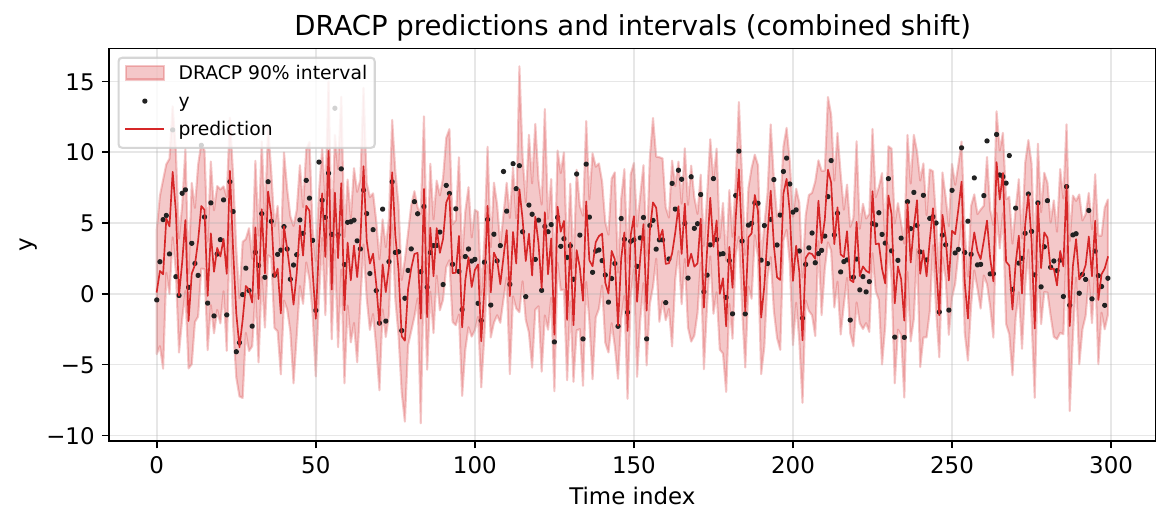}
	\caption{A synthetic regime-switching series with annotated regimes (left) and \DRACP{}
		prediction intervals on the combined-shift scenario (right).}
	\label{fig:synthfig}
\end{figure}

\begin{figure}[htbp]\centering
	\includegraphics[width=.48\textwidth]{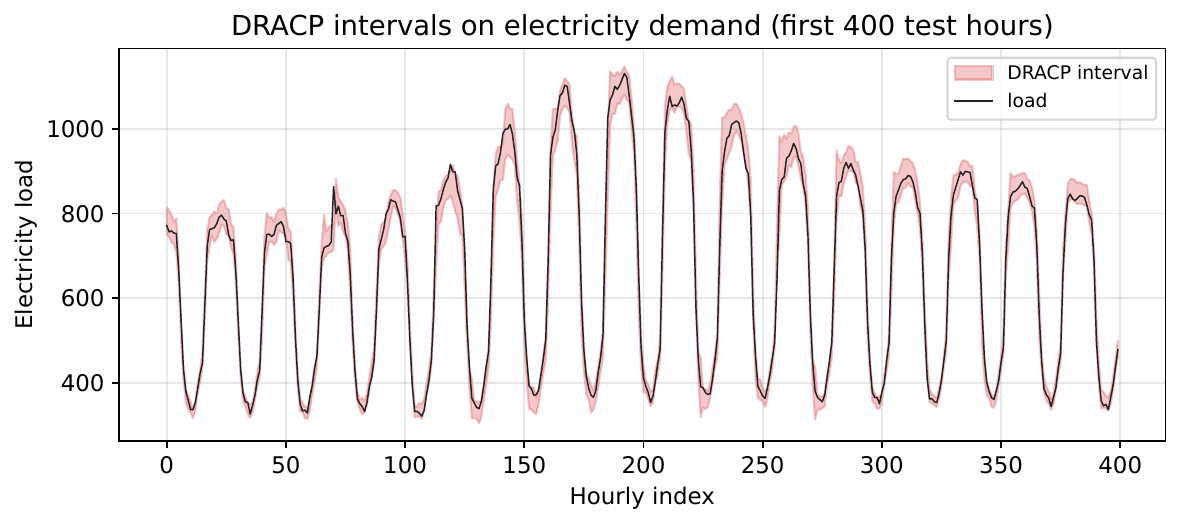}\hfill
	\includegraphics[width=.48\textwidth]{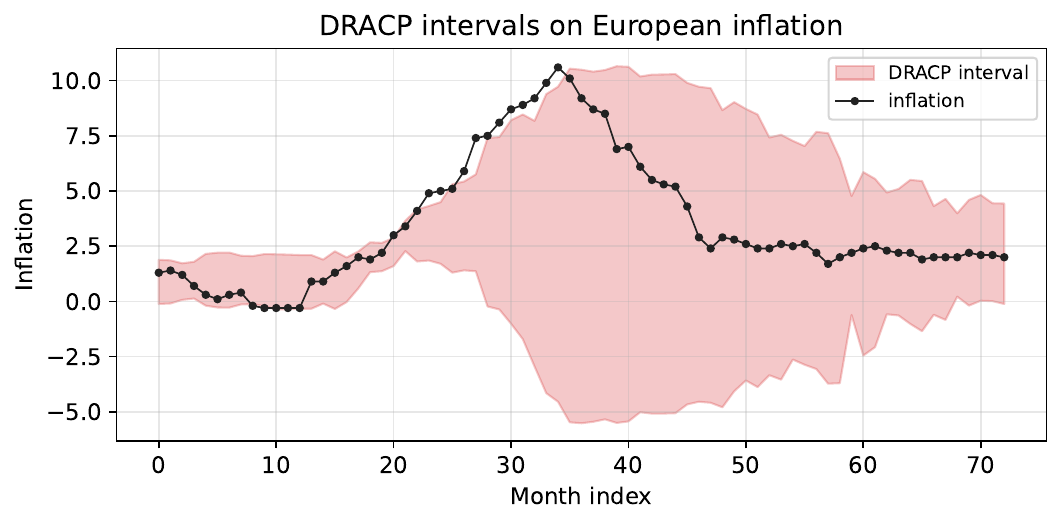}
	\caption{\DRACP{} prediction intervals on electricity demand (left) and euro-area inflation
		(right).}
	\label{fig:pred}
\end{figure}

\clearpage
\bibliographystyle{elsarticle-harv}

\section*{CRediT authorship contribution statement}
\textbf{Bogdan Oancea:} Conceptualization, Methodology, Software, Validation, Formal analysis,
Investigation, Data curation, Writing -- original draft, Writing -- review \& editing,
Visualization, Funding acquisition.

\section*{Declaration of competing interest}
The author declares that he has no known competing financial interests or personal
relationships that could have appeared to influence the work reported in this paper.

\section*{Funding}
This work was supported by UEFISCDI, project PN-IV-P6-6.1-CoEx-2024-0196, contract
no.~17/2.03.2026.

\section*{Data availability}
All data analyzed in this study are publicly available and are downloaded programmatically by
the released software. The harmonized index of consumer prices series are obtained from the
Eurostat dissemination API (\texttt{prc\_hicp\_manr}); the United States macroeconomic,
monetary and financial series from the Federal Reserve Economic Data (FRED) database; and the
electricity demand series from the UCI ElectricityLoadDiagrams archive. The code, configuration
files and processed data required to reproduce every table and figure are released as an
open-source Python packag available at \url{https://github.com/bogdanoancea/dracp}, archived under
\doi{10.5281/zenodo.21675706}.

\section*{Declaration of generative AI and AI-assisted technologies in the manuscript
preparation process}
During the preparation of this work the author used Claude (Anthropic) in order to assist with
revising prose, restructuring sections, checking numerical consistency between the
manuscript and the underlying result files, and preparing \LaTeX{} tables. After
using this tool, the author reviewed and edited the content as needed and takes full
responsibility for the content of the published article.

\bibliography{references_with_dois}

\end{document}